\documentclass{article}
\PassOptionsToPackage{sort&compress}{natbib} 
\usepackage[main,final]{neurips_2026}

\usepackage[utf8]{inputenc} \usepackage[T1]{fontenc}    \usepackage{hyperref}       \usepackage{url}            \usepackage{booktabs}       \usepackage{amsfonts}       \usepackage{nicefrac}       \usepackage{microtype}      \usepackage{xcolor}

\usepackage{graphicx}
\usepackage{subcaption}
\usepackage{listings}
\usepackage{refcount}
\usepackage{tikz}
\usepackage{xspace}
\usepackage{amsmath}
\usepackage{amssymb}
\usepackage{mathtools}
\usepackage{amsthm}
\usepackage{placeins}
\usepackage[capitalize,noabbrev]{cleveref}
\usepackage{thm-restate}
\usepackage{algorithm}
\usepackage{algorithmic}
\usepackage{longtable}
\usepackage{tcolorbox}
\usepackage[textsize=small]{todonotes}
\usepackage{float}
\usepackage{outlines}
\usepackage[normalem]{ulem}
\usepackage{wrapfig}
\usepackage{needspace}
\usepackage{caption}
\crefname{figure}{Fig.}{Figs.}
\Crefname{figure}{Fig.}{Figs.}
\crefname{subfigure}{Fig.}{Figs.}
\Crefname{subfigure}{Fig.}{Figs.}
\crefname{section}{\S}{\S\S}
\Crefname{section}{\S}{\S\S}
\crefname{subsection}{\S}{\S\S}
\Crefname{subsection}{\S}{\S\S}
\crefname{subsubsection}{\S}{\S\S}
\Crefname{subsubsection}{\S}{\S\S}
\crefname{appendix}{\S}{\S\S}
\Crefname{appendix}{\S}{\S\S}
\crefname{subappendix}{\S}{\S\S}
\Crefname{subappendix}{\S}{\S\S}
\crefname{subsubappendix}{\S}{\S\S}
\Crefname{subsubappendix}{\S}{\S\S}
\crefname{subsubsubappendix}{\S}{\S\S}
\Crefname{subsubsubappendix}{\S}{\S\S}
\AddToHook{env/outline/begin}{\let\1\relax}
\tcbset{
  colback=gray!3,
  colframe=gray!60,
  boxrule=0.5pt,
  arc=2pt,
  left=6pt,
  right=6pt,
  top=6pt,
  bottom=6pt
}
\usetikzlibrary{calc, positioning}

\newcommand{\aigen}{AI-generated\xspace}

\newcommand{\detectionmodel}{detection model\xspace}

\newcommand{\pnu}{PNU\xspace}

\newif\ifexcludeesif
\excludeesiftrue
\newcommand{\esif}[1]{\ifexcludeesif\else#1\fi}

\renewcommand{\paragraph}[1]{\textbf{#1}}

\theoremstyle{plain}
\newtheorem{theorem}{Theorem}[section]

\newtheorem{observation}[theorem]{Observation}
\theoremstyle{definition}

\theoremstyle{remark}

\newif\iffinal
\finaltrue    
\iffinal
  \newcommand{\nikhil}[1]{}

  \newcommand{\todoinline}[1]{}
  \newcommand{\kevin}[1]{}
  \newcommand{\mr}[1]{}

\providecommand{\todo}[1]{}
\else

  \newcommand{\nikhil}[1]{    \todo[inline,linecolor=orange,backgroundcolor=orange!20,bordercolor=orange]{\textbf{NG:} #1}  }

  \newcommand{\todoinline}[1]{    {\color{orange}\textbf{TODO:} #1}  }

  \newcommand{\kevin}[1]{    \todo[inline,linecolor=red,backgroundcolor=red!20,bordercolor=red]{\textbf{KR:} #1}  }

  \newcommand{\mr}[1]{    \todo[inline,linecolor=olive,backgroundcolor=olive!20,bordercolor=olive]{\textbf{MR:} #1}  }

\fi

\title{Hitting a Moving Target: Test-Time Adaptation for AI Text Detection under Continual Distribution Shift}

\author{  Kevin Ren \\
  Cornell Tech \\
\texttt{kevinren@cs.cornell.edu} \\
  \And
  Manish Raghavan \\
  MIT \\
  \texttt{mragh@mit.edu}
\And
  Nikhil Garg \\
  Cornell Tech \\
  \texttt{ngarg@cornell.edu}
}

\date{}

\begin{document}

\maketitle

\begin{abstract}
Deployed AI text detection approaches are supervised, relying on training-time access to labeled datasets of both human-written and \aigen text. This approach is vulnerable to three types of distribution shifts that occur continually post-deployment, and for which labeled data is often unavailable: adversarial humanization, new LLMs being released, and unidentifiable temporal drift in human writing after widespread LLM usage. Simultaneously, existing approaches do not leverage a key signal of LLM usage: inference-time homogeneity. We propose a test-time adaptation (TTA) approach that adapts to distribution shifts by leveraging homogeneity among unlabeled samples observed at inference time.
Empirically, we find that state-of-the-art supervised detectors systematically fail when they encounter distribution shifts in \aigen and human writing, both adversarial and natural, while TTA with semi-supervised learning is largely robust; e.g., the commercial model Pangram detects just {27.0}\% of our adversarial \aigen text, compared to {93.3}\% for our test-time approach. We establish that TTA is a promising framework for AI text detection {in the wild}.

\end{abstract}

\section{Introduction}\label{sec:intro}

\aigen text detection is widely used to flag undisclosed, undesirable AI usage on specific texts and to broadly understand societal effects of LLM usage \citep{liang2024icmlreviews, corpus2025introducing, wang2025perceived, russell2025ai,russell2025people, chakrabarty2025ai, chakrabarty2025readers, thorat2025dactyl, leippold2026authoritative, naddaf2025major,qian2026rise, sharma2026llms,elazar2026llm, kusumegi2025scientific, liang2025quantifying,liang2025widespread}. {These methods are widely used by the public}; e.g., the commercial detector Pangram has integrations with Substack and Gmail.
In turn, there is substantial effort to evade these detectors, such as through automated ``humanizers'' and fine-tuned LLMs. Detector developers assert that AI detection is solvable with current approaches: Pangram\footnote{\citet{jabarian2025artificial} compare Pangram, OriginalityAI, and GPTZero, finding that Pangram had the lowest error rates; their methods include a hard negative mining procedure in which additional data near the decision boundary is repeatedly added to the training set. We thus use Pangram as a strong supervised baseline.} claims 99\% accuracy \citep{emi2024technical, russell2025ai} and robustness to popular humanizers \citep{masrour2025damage}, corroborated by similar, independent estimates \citep{jabarian2025artificial}. Yet other evaluations suggest that these claims may not generalize. \citet{chakrabarty2025readers} show that, after fine-tuning models to match a specific human author's stylistic features, the outputs largely evade Pangram, with a 3\% detection rate.

\begin{figure*}[tbh]

\begin{subfigure}{.73\linewidth}
  \centering
  \includegraphics[width=\linewidth]{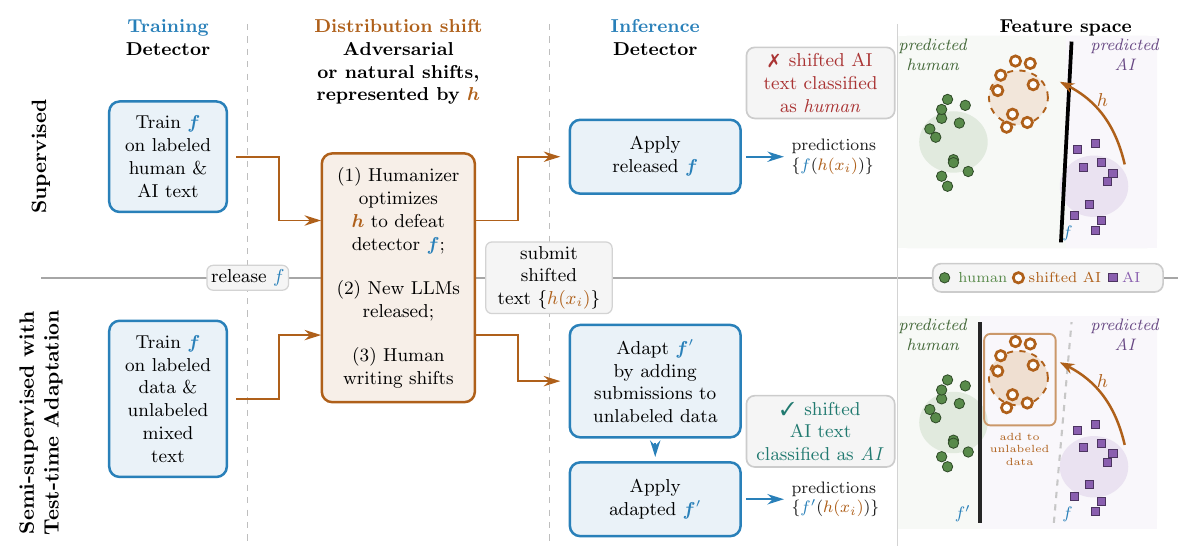}
  \label{fig:pipeline-1}
\end{subfigure}\begin{subfigure}{.27\linewidth}
  \centering
  \includegraphics[width=\linewidth]{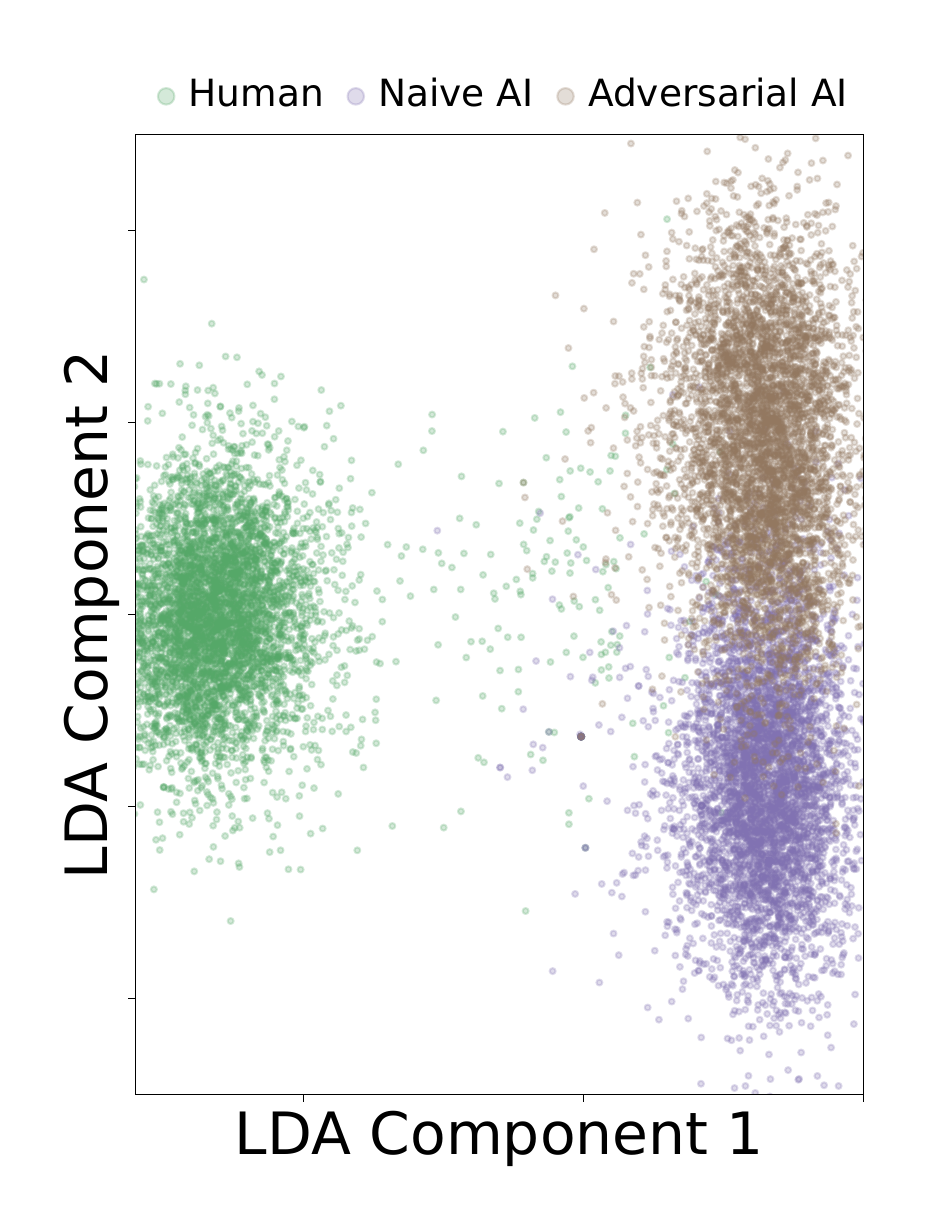}
  \label{fig:pipeline-2}
\end{subfigure}\caption{(left) Comparing AI text detection with supervised learning (top) and test-time adaptation with semi-supervised learning (bottom). With supervised detection: a detector is trained on labeled human and AI text, released, and can fail under distribution shift. With TTA: the detector re-estimates its boundary using submitted texts at inference time and therefore can adapt. (right) Empirical illustration---human, naive AI, and adversarial AI are all (approximately) mutually separable, but the latter may evade a fixed classifier. Details on figure generation are in \Cref{app:lda}.
}
\label{fig:pipeline_and_feature_space}
\end{figure*}

Substantial recent public discourse highlights this ``cat and mouse'' dynamic between detectors, humanizers, and users, seeking to understand detector performance and pitfalls \citep{goldstein2026cheating, rosmine2026deft, keleta2026accent, mit2026aieducation, gherman2026super, naddaf2026ai}. 
Despite this discourse, there is surprisingly little academic research on (1) understanding fundamental challenges in AI detection using existing methods, or (2) developing new frameworks robust to the challenges. For example, while a large body of work found general brittleness in previous generation detection models \citep{dawkins2025detection,dugan-etal-2024-raid,wang-etal-2024-m4,wang-etal-2024-m4gt,li-etal-2024-mage,sadasivan2023can,yu-etal-2025-evobench,liu2025generalization}, modern supervised models generally saturate such benchmarks---yet, as we argue, AI detection is not a solved problem. 

Existing popular detection models are primarily trained with {supervised learning} methods. For {individual-level classification} (identifying specific pieces of \aigen writing), Pangram trains a transformer-based classifier on millions of texts across dozens of domains, with a sophisticated data augmentation algorithm \citep{emi2024technical}. For {prevalence estimation} (estimating the \textit{proportion} of \aigen text in a corpus), \citet{liang2024mapping} propose a lightweight n-grams-based method. 
These methods assume access to training datasets that are in-distribution to the evaluation. To curate such datasets, practitioners (a) find source domain text written prior to the widespread adoption of LLMs (i.e., before 2022) for known human-written text, and (b) generate \aigen text themselves (``synthetic mirrors'') by querying an LLM to rewrite the human text. A perhaps surprising fact illustrated by modern detectors is that, when the evaluation (approximately) matches the training distribution, \aigen text can be nearly perfectly detected (accuracy of over 99\%).

However, as we argue and empirically demonstrate, distribution shift in AI text detection is especially challenging: it is \textit{continual} and often post-deployment \citep{yu-etal-2025-evobench}. Consider the shifts common in detection: adversarial humanization, in which users or humanizer services \citep{grammarly_ai_humanizer,humanizeai_pro,krishna2023paraphrasing}
strategically rewrite AI-generated text to evade a published detector; and
natural shifts over time, due to new LLM releases and temporal drift in human
writing. These shifts share two features: (1) they
\textit{continually} occur, and so the task differs from one-time
transfer learning; and (2) labeled test data may never be available or come with substantial cost. For example, verified human text is expensive to obtain in 2026, when any internet text may be \aigen. These features limit the performance of retraining or domain adaptation approaches, in which a deployed model is regularly retrained (e.g., monthly) to adapt to fixed target shifts. 

At the same time, existing methods (supervised or otherwise) do not leverage \textit{test-time homogeneity} in LLM text, which reflects \textit{algorithmic monoculture} \citep{kleinberg2021algorithmic,peng2024monoculture,wu2025generative, padmakumar2024does, raghavan2024competition, ugander2024art, doshi2024generative, goelgreat2025, kimcorrelated2025, toups2023ecosystem,jo2026subjectivity}. Our insight is that while it may be hard to classify any single shifted text sample, detection becomes feasible given a \textit{cohort} of unlabeled samples {(i.e., a sufficiently large test-time batch)}, if the distribution of shifts (e.g., releases of new LLMs and humanizer tools) is reasonably homogeneous.\footnote{e.g., in the NYTimes \citep{vigdor2025professors}, two professors describe asking students to disclose non-permitted LLM usage. They received many responses, and an unusually large fraction of them had the phrase ``sincerely apologize''---suggesting that many students were still using LLMs, to apologize for using LLMs.}

Thus, we {claim} that {AI text} detection---due to continual distribution shifts and test-time homogeneity---is an ideal application for the emerging
paradigm of \textbf{test-time adaptation} (TTA) 
\citep{sun2020test,wang2021tent,zhang2022memo,liang2025comprehensive,sahoo2020unsupervised}, which
adapts an existing classifier or re-trains a new classifier given unlabeled
test-time data. TTA methods can leverage test-time homogeneity in \aigen text, while methods that are {frozen at train-time}  consider each test-time sample independently and cannot exploit this latent signal. 

In this work, first, we present a conceptual framework to explain the dynamics of AI text detection under (potentially adversarial) continual distribution shifts, using insights from strategic classification, test-time adaptation, and semi-supervised learning. We show standard methods are at a fundamental timing-based disadvantage, and therefore vulnerable to adversarial and natural shifts. In contrast, test-time adaptation-based approaches can perform well.

Second, we empirically demonstrate both the challenges of distribution shift---even against modern supervised detectors---and the potential advantages of our framework. We demonstrate three types of distribution shift, using data on scientific abstracts and the RAID benchmark \citep{dugan-etal-2024-raid}. 
(a) \textbf{Strategically chosen humanization}; e.g., we produce \aigen text that reduces Pangram's detection rate of \aigen text from {96.0}\% (for non-adversarially chosen text) to {27.0}\%. (b) \textbf{Text from LLMs unseen at training time}; e.g., in experiments following the release of GPT 5.4, we find that Pangram predicted only 34\% of its outputs as fully AI or AI assisted while being near-perfect on older models. (c) \textbf{Human distribution shift over time}; e.g., when trained on 2010 data, the supervised approach of \citet{liang2024mapping} estimates that 15.1\% of sentences from abstracts on arXiv in 2020 were already LLM-generated (the true number is ostensibly near-zero). {In all settings, our test-time adaptation methods---that additionally use unlabeled test-time data---are robust, e.g., maintaining near-perfect performance even against strategically chosen humanizers.}

Overall, our work clarifies that \textit{continually} shifting distributions, both adversarial and natural, are the key bottleneck for robust LLM detection. Existing methods can perform extremely well with in-distribution data, even on texts as short as a sentence---i.e., naive AI usage of LLMs is detectable given appropriate training, perhaps surprisingly so. However, performance degrades substantially on out-of-distribution text. Empirically, we show the potential for test-time adaptation to mitigate the impacts of distribution shift in LLM text detection, by leveraging latent homogeneity. However, we warn that it may not be a panacea---while TTA methods may help, it remains an open question of how effective any detection method will be in the future, given the dual challenges of (a) converging human and \aigen distributions; and (b) acquiring new, verified human-written text.

\section{Conceptual framework} \label{sec:conceptual}

Consider the characteristics of distribution shifts in LLM text detection: (1) they occur continually, and (2) labeled data from the shifted distribution may be scarce. These characteristics both underlie the problem's difficulty, and motivate test-time adaptation. 
The dynamics are sharpest with adversarial adaptation, viewed through the lens of \textit{strategic classification} and \textit{performative prediction}
\citep{dalvi2004adversarial,hardt2016strategic,perdomo2020performative}. As illustrated in \Cref{fig:pipeline_and_feature_space}, the detector first publishes a classifier. Then, the subject (the evader) strategically alters their features (submitted writing) to receive a favorable classification. In other words, the classifier is a Stackelberg leader. 

(1) Distribution shifts occur continually, and the detector does not respond in real-time; i.e., the detector ``moves first'' and is thus at a disadvantage. A detector publicly releases its predictions through an API (or as an open source release).
Thus, in the standard supervised learning paradigm, this detector is fixed (between retraining runs), and the evader can respond in real-time to this fixed detector; for example, a user can keep editing generated text until it receives a favorable rating. As we show, a humanizer to defeat a fixed detector can also be built given black box API access to the detector; in contrast, supervised detector retraining is a slow process, and requires acquiring and labeling new data on humanization strategies.
The evader can thus always win (pass AI text off as human-written) \textit{if} they can shift the distribution of \aigen text to the other side of the fixed decision boundary. It is not sufficient for the detector to adapt once and publish a new classifier, as the evader can simply adapt to the new classifier. This disadvantage extends to transfer learning or domain adaptation methods that assume a fixed target domain. Intuitively, under this framework, the detector needs to ``move last,'' i.e., adapt to the evader's latest strategic responses at test-time. 

(2) A publicly released detector like Pangram {does not have true labels for text that is sent to its API}, and thus cannot retrain its model in real-time using a supervised learning framework. At best, if the popular humanization strategies are publicly available, the detector can generate humanized \aigen text using them, and then regularly re-train a supervised model to detect their outputs. 

These characteristics are also present in natural distribution shifts. LLMs are constantly being released and updated, and it may be costly or time-consuming to continually generate labeled data and retrain one's detector against them. While slower, human distribution shift may be even more challenging to respond to: pre-2022 human text is increasingly outdated, and more recent text is not guaranteed to be human-written. Thus, obtaining verified, recent, human-written text may be expensive (e.g., require paying writers and finding ways to ensure that they are not using LLMs).

\paragraph{Theoretical observations.}
We formalize these intuitions in a simple model. Let $\mathcal{H}$ and $\mathcal A$ denote human and initial AI text distributions, respectively. Function $h$ denotes distribution shift, mapping $\mathcal A$ to $h(\mathcal A)$. 
Pre-deployment, we observe labeled samples from $\mathcal{H}$ and $\mathcal A$, but none from $h(\mathcal A)$. The task is to classify samples from an (unknown) test-time mixture \(
\mathcal{U} = \lambda_{{H}} \mathcal{H} + \lambda_A \mathcal A + \lambda_{{h}} h(\mathcal A)\), where $\lambda_A,\lambda_{{h}}>0$. 
Of course, the task is impossible if the AI text distributionally matches human text.\footnote{Typically, strategic classification models impose costs on
manipulation to prevent arbitrary behavior \citep{hardt2016strategic,hu2019disparate,milli2019social,kleinberg2020classifiers,liu2022strategic}.
} 

\begin{observation}
Let $\mathcal{H}$ and $\mathcal G$ be two distributions, with total variation distance $\operatorname{TV}(\mathcal{H},\mathcal G)$. Then, for any binary classifier \(f\), we have balanced accuracy $\operatorname{Acc}(f;\mathcal{H},\mathcal G)
\leq \frac{1+\operatorname{TV}(\mathcal{H},\mathcal G)}{2}$.

\label{obs:tvbound}
\end{observation}

All proofs are in \Cref{app:theory}. Thus, to make progress, we assume that $\mathcal{H}$ is separable from $\mathcal A \cup h(\mathcal A)$ by some classifier in the hypothesis class $\mathcal F$ (i.e., $\exists f \in \mathcal F: {Acc}(f;\mathcal{H},\mathcal A \cup h(\mathcal A)) = 1$). Empirically as illustrated in \Cref{fig:pipeline_and_feature_space} (right), perhaps surprisingly, we observe separability---in our experiments, human and \aigen text can always be almost perfectly classified given in-distribution training data, even for texts as short as one sentence.\footnote{We confirm that this separability is not due to naive differences caused by poor prompting, such as length differences; we also find similar separability on existing  detection benchmarks, using our models.}  We also assume $h$ is non-trivial, where $h(\mathcal A)$ places positive mass outside the measurable support of $\mathcal A$. Given this setup, \textit{supervised learning can fail.}  

\begin{observation}

  For sufficiently rich hypothesis classes, many classifiers may perfectly separate $\mathcal{H}$ from $\mathcal A$, while not all of these separate $\mathcal{H}$ from $h(\mathcal A)$. Thus, a supervised detector $f$ can achieve perfect training-time accuracy on $\mathcal{H}$ and $\mathcal A$ while still misclassifying the shifted AI text.

\label{obs:supfails}
\end{observation}

Indeed, with a humanizer $h$ that is adapting to the published detector $f$, the text $h(\mathcal A)$ may be \textit{adversarially} designed to be misclassified as human by $f$. On the other hand, \textit{semi-supervised learning can adapt with test-time data}---by using unlabeled inference-time data, drawn from $\mathcal{U}$.

\begin{observation}
There exists a {test-time adaptation optimization} whose population minimizers are classifiers that separate $\mathcal{H}$ from
$\mathcal A \cup h(\mathcal A)$, using only labeled $\mathcal{H}$
and unlabeled $\mathcal U$ data.
\label{obs:puworks}
\end{observation}
This observation is consistent with standard positive-unlabeled (PU) learning results under our assumptions.\footnote{A PU classifier uses known labels from one class, and an unlabeled mixture set, cf. \citet{liu2002partially,mansouri2025learning}. One proof is: with human text as the positive class, a PU learner can solve:
\(
\min_{f\in\mathcal F} P_{\mathcal{U}}(f(X)= \text{human})\) s.t. \(
P_{\mathcal{H}}(f(X)=\text{human})=1.
\)
The constraint ensures labeled human text is classified as human; as unlabeled human text is also drawn from $\mathcal{H}$, it will also be classified as human. The objective penalizes classifying the unlabeled text as human. Under realizability, the population optima are perfect separators of $\mathcal{H}$ from $\mathcal A \cup h(\mathcal A)$. This guarantee does not require knowledge of the weights $\lambda$, and holds even when the humanized text $h(\mathcal A)$ is adversarially designed, as long as it is separable from $\mathcal{H}$.
} 
In the remainder of this work, we show this theoretical discussion empirically: supervised learning can be systematically evaded, while semi-supervised TTA methods can efficiently adapt, without labels. Finally, as we empirically show in \Cref{sec:results_temporal}, our framework can also be extended to settings with shifts in human text, by treating (known) LLM examples as the positive class and treating the test-time distribution as a mixture of LLM and human text. In the case of distribution shifts in human texts (holding AI text distribution fixed), the above insights apply, with a change of notation, in which $\mathcal{H}$ instead refers to the fixed \aigen distribution, $\mathcal{A}$ to the training-time human distribution, and $h(\mathcal{A})$ to the post-shift human distribution.

\esif{

}

\section{Experimental setup}\label{sec:exps}
Our pipeline for dataset curation largely follows that used in related literature \citep{liang2024mapping,emi2024technical}. We first collect human text. Then, we use an LLM and prompt to produce \aigen ``mirrors'' of the human text (given existing text, prompt the LLM to rewrite the text with additional style suggestions while preserving the information), and simulate train-time and test-time conditions, with potential distribution shift between the two. We train a supervised classifier on the train-time data with fully known labels, and semi-supervised methods with TTA, as detailed below. Additional details are in \Cref{sec:results} and \Cref{app:additionaldetails};  generating and then using ``mirrors'' is standard practice in training and evaluating supervised learning LLM text detection classifiers. Our code (including for model training, evaluation, and plots) is available at \href{https://github.com/kkr36/llm_detection}{https://github.com/kkr36/llm\_detection}.

\paragraph{Data and Prompts.} We conduct experiments on the arXiv dataset \citep{cornell_arxiv_kaggle}, using abstracts from 2010-2025.
Across experiments, we use {11} total LLMs. For experiments involving strategic behavior, mirrors are produced using an adversarial rewriting prompt, optimized via an autoresearch loop \citep{karpathy2026autoresearch}, to specifically evade a given \detectionmodel; 
otherwise, we use the prompting strategy defined in \citet{liang2024mapping} (we refer to this as the ``naive prompt''). {We also use the RAID benchmark \citep{dugan-etal-2024-raid}, with qualitatively similar findings.}

We validate that our adversarial rewrites generally preserve content and that hallucinations or omissions do not explain detector evasion; the generated text does not differ substantially from human text in length or other observable features. (Surprisingly to us, LLM detection on our data seems to be a task our detectors can effectively perform with suitable training, despite it being a hard task for humans to perform by inspection). See \Cref{sec:mirrorvalidation} for details; example generations are in \Cref{tab:sampled_sentences}.

\paragraph{Supervised detection methods.} We fine-tune DistilBERT \citep{sanh2019distilbert} to train our own supervised models, using cross-entropy as the loss function. We also use Pangram 3.3.2, the prevalence estimator of \citet{liang2024mapping}, and the zero-shot method from \citet{bao2023fast}. 

\paragraph{TTA detection methods.} We use several methods to adapt to unlabeled data at test-time.

Primarily, we use the positive-unlabeled (PU) $(\mathrm{TED})^n$ method of \citet{garg2021mixture}, which works effectively in high dimensions. $(\mathrm{TED})^n$ alternates between classifier training and estimating the positive proportion $\alpha$, using ratios of unlabeled to labeled-positive mass above thresholds and assuming a pure-positive upper tail exists. When training the classifier, it uses all but the highest-scoring $\alpha$-fraction of unlabeled points as pseudo-negatives. {As with our supervised detectors, the detectors trained with $(\mathrm{TED})^n$ also use DistilBERT as the classifier component, with the sole difference being in the loss function.}
We train such PU classifiers on \textit{one} of the train-time data classes (either AI or human) as the known positive class and unlabeled data sampled from the test-time distribution. We also adapt $(\mathrm{TED})^n$ to include labeled negatives from the source domain (i.e., a \pnu learning method; similar in principle to \citet{kiryo2017positive}). We refer to these methods as PU + TTA and PNU + TTA, respectively. Note that these models are largely similar to the supervised models in underlying architecture and primarily differ in their use of unlabeled test-time data; thus, comparisons between these and our supervised models directly identify the effects of interest.\footnote{Note that the PU method does not have access to one of the labeled classes at train-time. In practice, we foresee that a PNU classifier that combines state-of-the-art supervision alongside test-time adaptation would work best; our PU measurements are \textit{lower bounds} on such performance, and our PNU models perform effectively in our setup.}  

For some experiments, we adapt the streaming Test-Time Training (TTT) method of \citet{sun2020test} and the ConDA method of \citet{bhattacharjee2023conda} {(for both of which we also use DistilBERT as the detection model architecture)}. For TTT, we pre-train the detection model using both a supervised and self-supervised loss function; at test-time only the self-supervised loss is used to update the model. {Here, we choose masked language modeling (MLM) as the self-supervised task. TTT places two heads on top of the DistilBERT architecture, one for each loss function: at train-time the \detectionmodel is trained to minimize the classification loss and MLM loss simultaneously, and at test-time the classification head is frozen and only the MLM head is used to update the underlying DistilBERT component. For ConDA, the source domain is a mix of labeled LLM and human sentences, and the test domain is an unlabeled mix of human sentences and LLM sentences from a different LLM from training (the goal is to learn a classifier that is invariant to both the source and target domain). ConDA also requires a text transformation for its contrastive loss: while synonym replacement is used in the original paper, we instead use token dropout with a probability of 0.2.}

Additional model details are in \Cref{app:additionaldetails} and \Cref{sec:conda}. For all models, we threshold continuous predictions at 0.5 to create binary labels. We also evaluate models on a \textit{prevalence estimation} task (what fraction of the set is \aigen), by averaging their predictions of $\Pr(\text{LLM})$ over texts. The detectors in general are cheap to train, requiring about 20 minutes on a single GPU.

\paragraph{Metrics and evaluation.} In the main text, we assess classification performance using binary \textit{accuracy} on a label-balanced test set; when interested in predictions on a given class, we further evaluate the \textit{detection rate} (class-conditional recall) for that class. For prevalence estimation, we primarily use \textit{bias}, the difference between estimated and true prevalence of LLM text. \Cref{app:figs} shows other metrics, such as AUC, cross entropy loss, and recall for each class. We primarily train and evaluate on individual \textit{sentences} from abstracts (except for Pangram, as sentences are too short for their stated reliability claims; for these, we use full abstracts). To prevent data leakage, we cluster full abstracts into the train/test split, sharing splits across models and bootstrapping the split for CIs.

\esif{
}

\section{Results}\label{sec:results}
We empirically evaluate performance given three distribution shifts: \textit{adversarial} shifts in AI text (\Cref{sec:results_adversarial}), \textit{natural} shifts between LLMs (\Cref{sec:results_llm}), and temporal shifts in human writing (\Cref{sec:results_temporal}).
Supervised approaches perform poorly out-of-distribution, but TTA approaches perform well.

\subsection{Strategic behavior between \aigen text detectors and AI users}
\label{sec:results_adversarial}
We simulate a multi-turn detection-evasion game; in each turn, the detector trains and releases a \detectionmodel trained on all observed \aigen text from prior turns, and the evader produces a humanizer prompt that rewrites AI text to evade detection by the most recent \detectionmodel.

\paragraph{Adversarial humanizer development.}
{As the evader, we develop rewriting (humanizer) prompts $h$ that can ``evade'' a given fixed detector $f$: given AI text $x$, the humanizer prompt instructs an LLM to rewrite $x$, producing text $h(x)$ that is predicted as human-written by $f$, while preserving the substance of $x$. We use Claude Code to generate candidate prompts in an autoresearch loop \citep{karpathy2026autoresearch}}. The agent observes the \aigen text from its candidate prompts, along with the \detectionmodel's predictions. It updates the prompts by analyzing what texts evade the \detectionmodel, until it produces a prompt that consistently yields \aigen text that evades the \detectionmodel. Our approach only sees data from scientific abstracts---and so tailors its strategy to bypassing the detectors on such tasks; however, we foresee this approach being generally useful, and it underscores the ease with which humanizers can be developed. \Cref{app:claude} contains details on our approach, which can be seen as a purely agent-driven, in-context analogue of gradient-based or logprob-based prompt optimization \citep{zhu2024advprefix, khattab2023dspy}; \Cref{tab:sampled_sentences_xz} contains sample generations.

\paragraph{Experiment overview.}
{In each iteration, the detector releases a model trained on AI text generated by the prompt in the previous iteration (in the first iteration, we use the naive prompt from \citet{liang2024mapping} to train our detectors); the evader develops a humanizer prompt optimized against this detector.
For supervised learning, we evaluate the released detector against new adversarial generations, and retrain before the next iteration (i.e., the detector eventually receives labeled training data). PU + TTA can adapt to the new generations---we retrain it using unlabeled set drawn from the test distribution (containing new adversarial generations).}
We continue for three total iterations. 

More precisely, we seed iteration 1 by evaluating on adversarial prompts attacking Pangram 3.2;
subsequent iterations involve retraining detectors on AI text from all prior iterations and attacking the new detectors with Claude Code as described above. We also show results for a prompt attacking Pangram 3.3.2. The evader optimizes against the previous round's detector.
As we cannot retrain Pangram, we exclude it here; we also exclude PNU + TTA given the performance of PU + TTA. Throughout, we generate AI text using gpt-oss-120b. \Cref{fig:adversarial-1} (left) describes the first iteration, against our DistilBERT model (supervised), Pangram,\footnote{For Pangram, we use their discretized predictions of text as either ``AI-written'', ``AI-assisted'', or ``human-written''; any text predicted as ``AI-assisted'' receives an accuracy of 0.5.} PU + TTA, and PNU + TTA. \Cref{fig:adversarial-1} (right) illustrates subsequent iterations.

{We further consider an online learning scheme in which a detector trained on naive AI writing and human writing receives a stream of in-distribution human writing and adversarial AI generations at test-time; we use the Test-Time Training (TTT) method from \citet{sun2020test} to adapt to the test-time distribution, as detailed in \Cref{sec:conda}. (Recall that this method does not need to be fully retrained to adapt at test time, and so can work effectively with streaming batches).

\begin{figure}[tb]
\begin{subfigure}{.67\linewidth}
  \centering
  \includegraphics[width=\linewidth]{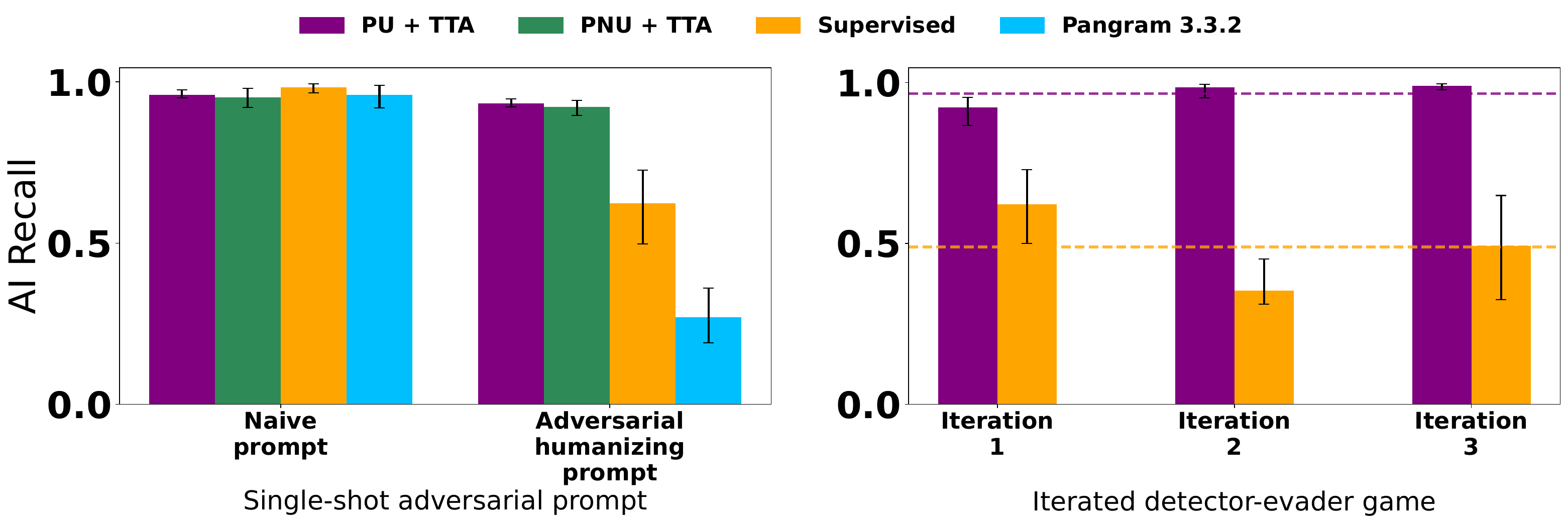}
  \caption{Multi-turn adversarial ``arms race''}
  \label{fig:adversarial-1}
\end{subfigure}\begin{subfigure}{.33\linewidth}
  \centering
  \includegraphics[width=\linewidth]{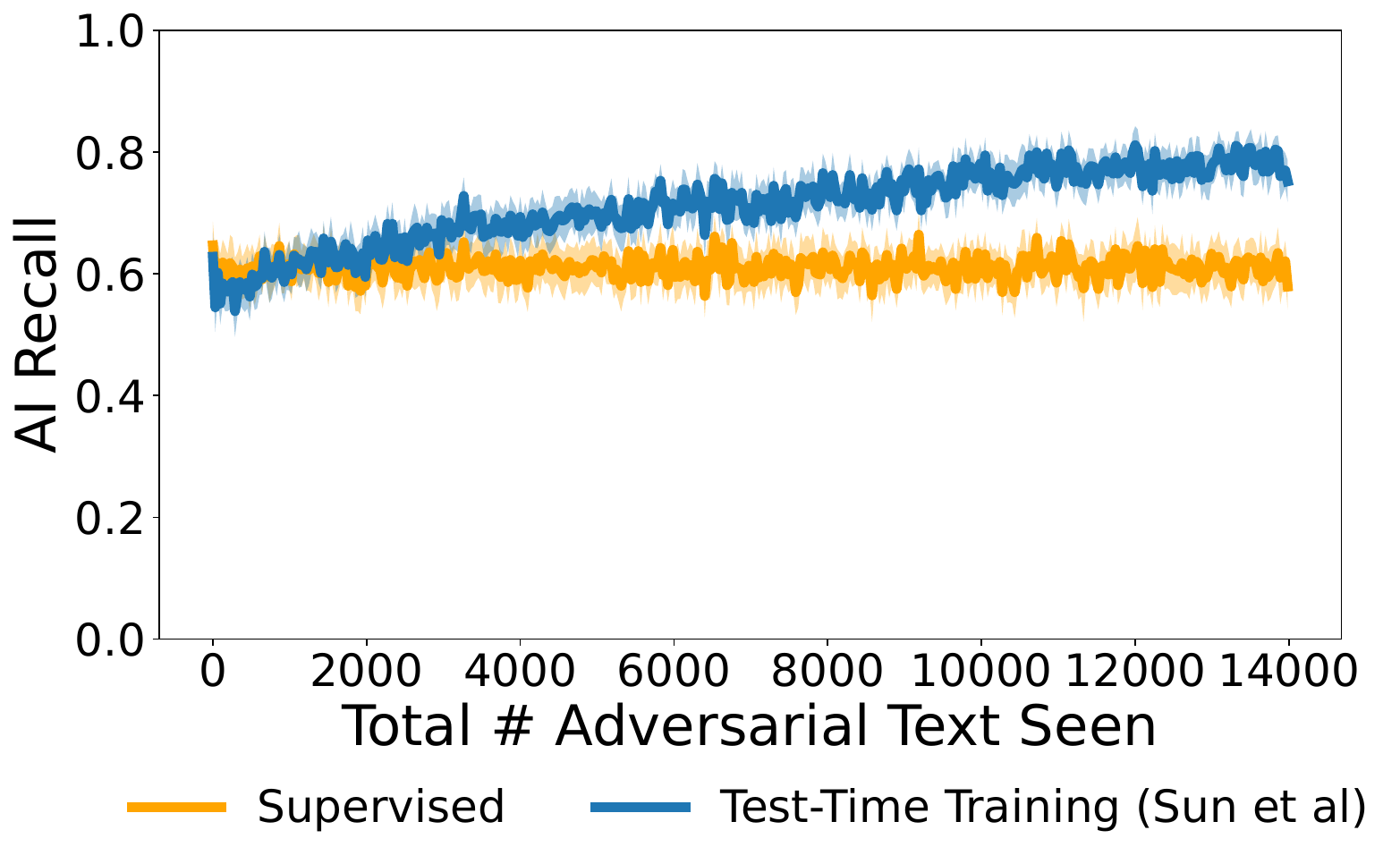}
  \caption{Test-Time Training (TTT)}
  \label{fig:adversarial-2}
\end{subfigure}\caption{(a) We simulate the back-and-forth between detectors and evaders. PU + TTA and PNU + TTA are robust, while supervised methods are not. (b) {The self-supervised method in \citet{sun2020test} effectively adapts classifiers trained on naive AI writing to adversarial AI text at test-time.} 
}
\label{fig:adversarial}
\end{figure}

\paragraph{Results: Supervised learning can be defeated, while TTA is robust.} Our autoresearch loop develops rewriting prompts that evade the supervised models. While Pangram is near-perfect on naively generated AI text (balanced accuracy of {98}\%; AI recall of {96}\%), an adversarially generated rewriting prompt can produce AI text that evades it (balanced accuracy of {63.5}\%; AI recall of {27}\%). In contrast, TTA remains robust {(for both PU + TTA and PNU + TTA, on the first adversarial prompt)}. This remains true across multiple iterations; e.g., at the final iteration, PU + TTA has a balanced accuracy of {94.4}\% and AI recall of {99}\%; supervised learning has a balanced accuracy of {69.3}\% and AI recall of {49.2}\%. Compare our results with those of \citet{masrour2025damage}, whose evaluations show that training on one humanizer increases robustness against unobserved humanizers. We find that while it may be \textit{harder} to optimize against such a model (it took Claude Code more attempts to find a prompt that beat each classifier in later iterations), we can still do so successfully. Further, TTA is essential: a PU classifier without access to adversarial test-time unlabeled data can be defeated.\footnote{During the autoresearch loop, we held the unlabeled batch constant until it was defeated. Then, we evaluated on an updated unlabeled batch that included adversarial generations. This likely makes our simulation conservative: in practice, optimizing against a constantly adapting TTA detector may be harder.} The RAID benchmark analysis (\Cref{sec:raid}) has qualitatively similar results, with PU + TTA performing well even on attacks that defeat the supervised model.

\paragraph{Results: A small amount of unlabeled test-time data is enough.} \Cref{app:adversarial} contains a sensitivity analysis of the performance of PU + TTA with respect to the amounts of adversarial and naive \aigen text; we find that relatively small amounts of in-distribution \aigen text in the unlabeled set are sufficient for good detection performance. Furthermore, \Cref{fig:adversarial-2} shows that the approach of \citet{sun2020test} (which does not require full-model updates to adapt to test-time data) performs well on test-time streams of human writing and test-time adversarial AI writing. These suggest that TTA is a promising approach for real-world detection even in adversarial settings.

\subsection{Natural distribution shift across LLMs as new models are released}
\label{sec:results_llm}

\begin{wrapfigure}{L}{0.5\textwidth}
  \centering
  \captionsetup{skip=4pt}
  \vspace{-6pt}
  \includegraphics[width=0.9\linewidth]{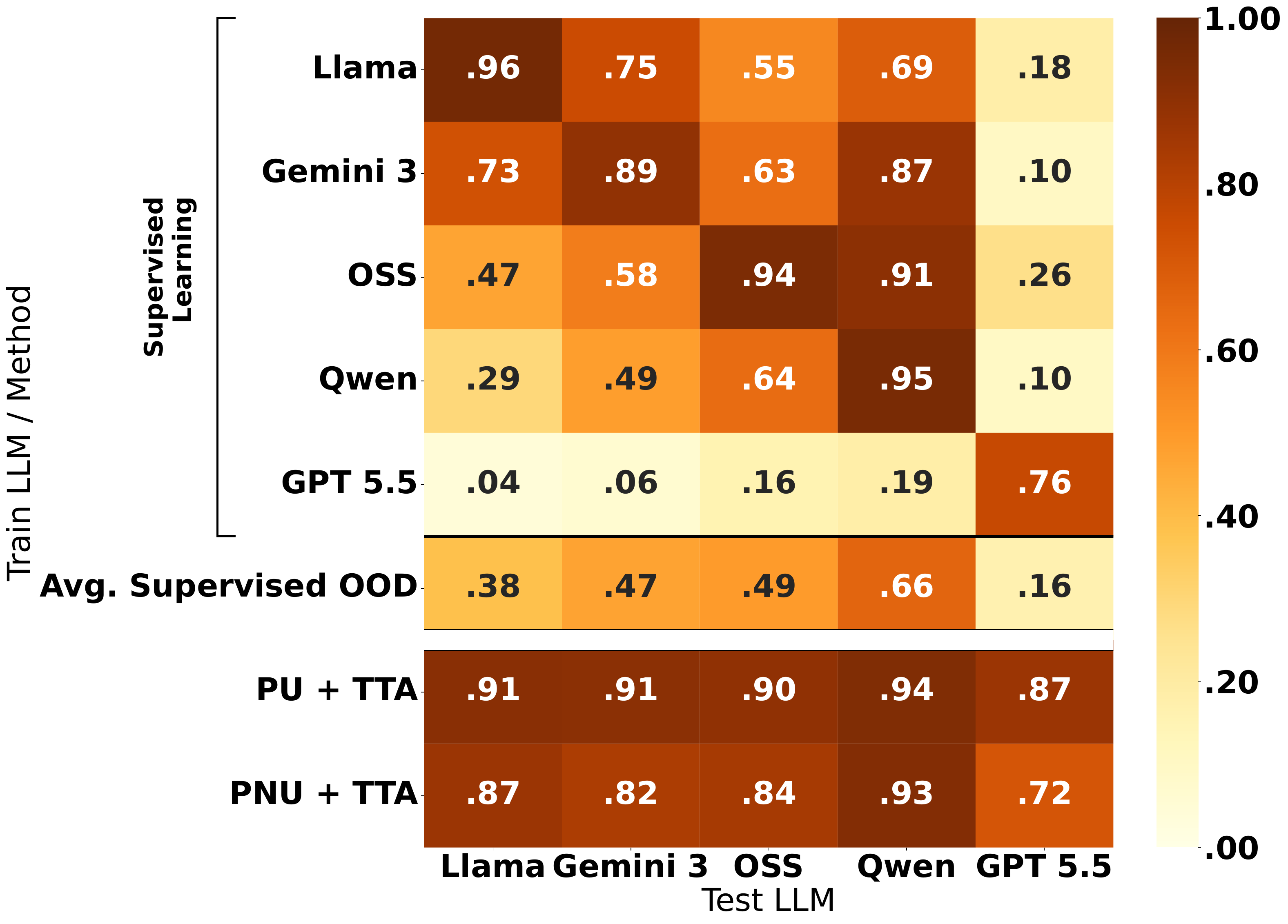}
  \caption{Detection rate with LLM shift}
  \label{fig:llm-shift}
  \vspace{-12pt}
\end{wrapfigure}

We train supervised models on labeled data from one LLM and evaluate on text from another. We also evaluate PU + TTA and PNU + TTA, which see unlabeled text from the test-time distribution. We plot AI detection rates in \Cref{fig:llm-shift}. ``Avg. Supervised OOD'' averages the off-diagonal supervised results for each test distribution, while PNU + TTA averages over models trained on labeled \aigen text from LLMs other than the test-time LLM. For PU + TTA, no train distribution is used, so there is only one model trained with the test distribution as unlabeled data. Here and in \Cref{sec:results_temporal}, we use the naive prompt to generate text.

\paragraph{Results: Supervised learning performs poorly out-of-distribution, while TTA methods are robust.} Supervised detectors perform well on \textit{in-distribution} data, as shown by the main diagonal in \Cref{fig:llm-shift}; both balanced accuracy and AI recall almost always exceed 89\%. However, the off-diagonals reveal they perform poorly on the outputs of LLMs not trained on; e.g., a detector trained on Llama 3.3 70B Instruct induces low AI recall on abstracts mirrored with gpt-oss-120b (55\%). In contrast, PU + TTA and PNU + TTA---whose unlabeled set includes the evaluation LLM---are robust; detection rates are close to or surpassing that of a supervised detector trained in-distribution.

These results are robust. We find similar results on the RAID benchmark when analyzing performance on in- and out-of-distribution LLMs \Cref{sec:raid}.
Furthermore, the domain adaptation method in \citet{bhattacharjee2023conda} generalizes \textit{when used as a TTA method}, and the zero-shot method in \citet{bao2023fast} performs poorly out of distribution (\Cref{sec:conda})---our conceptual findings generalize beyond specific detector instantiations. \Cref{app:lda} shows that cross-LLM generalizability relates to LLM similarity; e.g., detectors perform better when applied within the Gemini family.
TTA further performs relatively well simultaneously on multiple distributions of AI writing (see ``all'' column in \Cref{fig:gemini}).

The challenge of out-of-distribution LLMs extends to Pangram. When GPT 5.4 was released on March 5, 2026, we found that Pangram predicted its abstracts  as human-written a majority of the time (34.0\% predicted fully AI or assisted AI); in contrast, Pangram is near-perfect for older models (see \Cref{fig:pangram}). {Notably, Pangram's performance increased significantly following the release of Pangram 3.3.2 in May 2026. This delay, but eventual high performance, reflects the philosophy of improving performance by finding labeled, test-distribution data to retrain the \detectionmodel.}\footnote{Pangram also reported near-perfect (99.5\%) AI recall \citep{masrour2026pangram} on GPT 5.4 at this time. It is unclear why the results differ; possible explanations include text length, and that our prompt and data are OOD.}

\subsection{Temporal distribution drift in human writing}\label{subsec:temporal}
\label{sec:results_temporal}

\begin{figure}[t]
\begin{subfigure}{.60\linewidth}
  \centering
  \includegraphics[width=\linewidth]{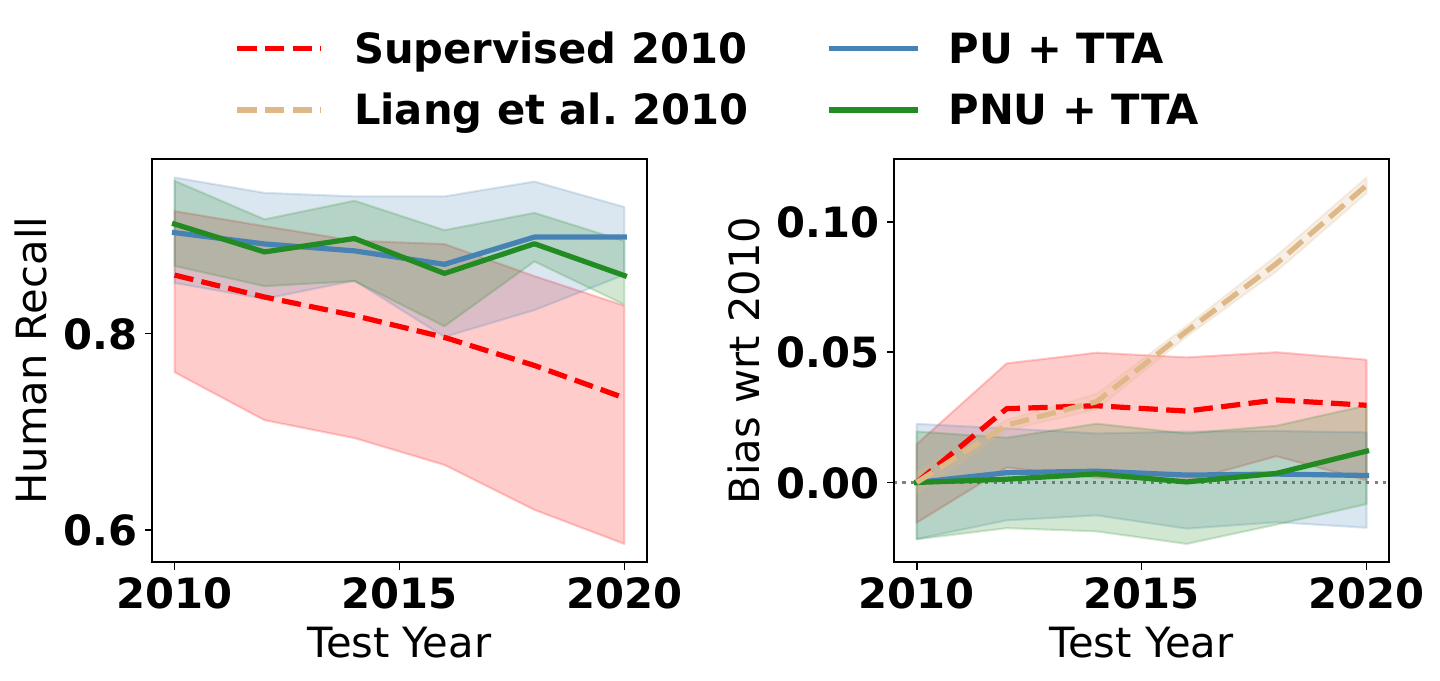}
  \caption{Simulated effect of human distribution shift 2010-2020}
  \label{fig:droptemporal-1}
\end{subfigure}\begin{subfigure}{.39\linewidth}
  \centering
  \includegraphics[width=\linewidth]{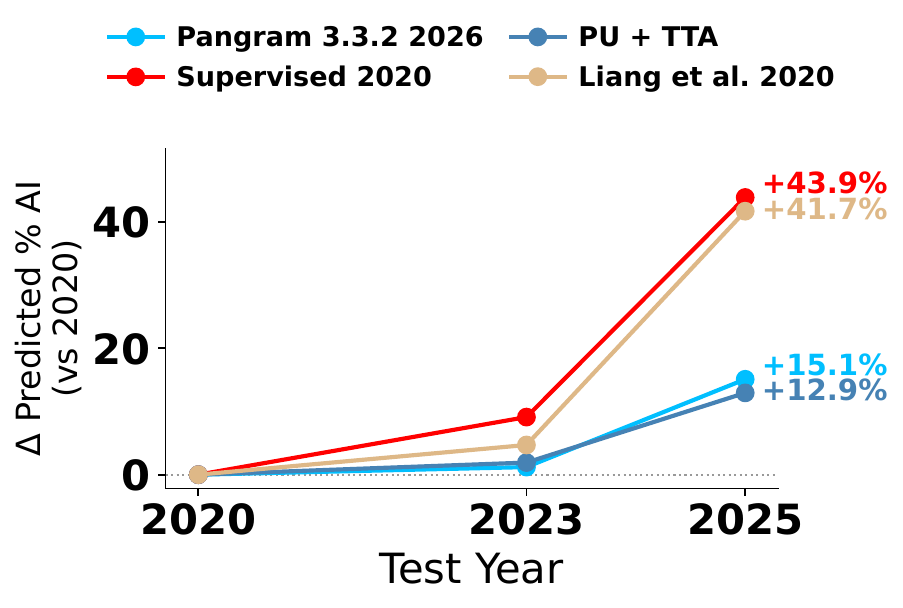}

\caption{Estimated AI usage from 2020-2025}
\label{fig:droptemporal-3}
\end{subfigure}

\caption{(a) Due to temporal distribution shift in human writing, detectors trained in 2010 degrade and overestimate LLM prevalence (label human text as AI). We plot human recall and relative bias in estimated proportion of \aigen text from the 2010 estimate on held-out data through 2020. Supervised methods are less robust than TTA. (b) Analysis of AI usage in years post-2020---PU + TTA predicts significantly lower AI usage than supervised methods trained in 2020 data.
}

\label{fig:droptemporal}
\end{figure}

Finally, we quantify the effect of distribution shift in human writing. Given widespread LLM use, ground-truth labels for post-2022 text are unavailable for training---and as pre-2022 human-written text is increasingly out-of-distribution with current-year writing, supervised models relying on such human data may perform poorly. Perhaps surprisingly---while discussion of adversarial humanization is prevalent---this pending challenge to the availability of contemporary human data to train detectors has been ignored. To our knowledge, including in discussion with prominent companies in the space, no solution is known except to pay writers to generate verifiable human text. 

However, to measure this effect, we run into the same challenge: we do not have easy access to post-2022 verified human-written text to evaluate detectors. Instead, we conduct the following simulation: suppose that LLMs were widely adopted after \textit{2010}, so supervised detectors must train on 2010 or earlier human text. Then, we can measure detector performance on a test set of known human text and AI mirrors from 2010-2020. Since we are interested in human distribution shift, we report the recall rate on human writing; for prevalence estimation, we plot the difference in the estimated proportion of \aigen writing at each year and the estimate at 2010 (i.e., measuring the drift in estimated prevalence of \aigen text). {We train PU + TTA {and \pnu + TTA} on known \textit{LLM}-generated text (mirrors of human abstracts from the evaluation year) as the positive set, and a mixture of \aigen and human text from the evaluation year as the unlabeled set}.

\paragraph{Results: Supervised methods are susceptible to drift in human language, but TTA  is more robust.} \Cref{fig:droptemporal-1} (left) illustrates that as the test data becomes increasingly out-of-distribution, the individual-level prediction quality of a fixed supervised detector gradually decreases (human writing recall decreases from 85.9\% in 2010 to 73.4\% in 2020). TTA does not degrade as much {(for PU + TTA and \pnu + TTA, recall on human writing is more stable over time, from 90.2\% and 91.1\% in 2010 to 89.8\% and 85.9\% in 2020, respectively)}. \Cref{fig:droptemporal-1} (right) shows the same for prevalence estimation: while PU + TTA has a stable prevalence bias from 2010 to 2020, the bias of  supervised methods substantially increases over time, as they increasingly predict recent, out-of-distribution human text as LLM-written. For the commonly used method of \citet{liang2024mapping}, bias increases by  11.4\%. 
When evaluated on only human text, the same method estimates that 15.1\% of sentences from 2020 abstracts are \aigen (but only 2.3\% in 2010). Notably, these errors are of the kind that is especially potentially harmful: inaccurately labeling human-written text as \aigen.

\paragraph{Results: Drift affects contemporary analyses.} In \Cref{fig:droptemporal-3}, we train models using 2020 data, and apply them to 2023 and 2025 data. Perhaps surprisingly, our supervised models estimate an increase in LLM usage over three times as large as that estimated by the PU + TTA model trained on the labeled data distribution. While ground truth data for the 2025 data is unavailable, our results above suggest that the true prevalence is closer to the latter's estimate. Pangram 3.3.2's estimate is also close to our PU + TTA model's estimate, suggesting that a large enough, professionally trained model can counter these effects at least up to the present day. We note that the method of \citet{liang2024mapping} is used to estimate prevalence in numerous studies, as cited in the introduction. 

Overall,  our results demonstrate that human distribution shift---combined with the difficulty of obtaining contemporary, verified human text---will substantially affect \aigen detection performance. We further note that falsely labeling human writing as AI-generated, as supervised models do here, may be especially harmful. While TTA-based methods help in our analyses, it is unclear how effective they can be long term; while PU + TTA (with LLMs as the positive class) is more robust to human distribution shift, it does not perform as well on detecting LLM text (\Cref{app:cdf}). It may be that detectors need to see labeled human writing to perform well overall. This challenge will be further exacerbated as humans' own writing is influenced by reading \aigen writing \citep{klebanov2025chatgpt, hutson2026bland}, reinforcing potential distribution convergence.

\section{Related work}
\label{sec:related}

\aigen text detection is a rapidly growing field. 
State-of-the-art methods used in practice are primarily supervised, encompassing traditional machine learning (e.g., neural network-based), contrastive learning, linguistically inspired feature engineering, and intrinsic dimensionality methods \citep{emi2024technical,li2025s,tulchinskii2023intrinsic,liu2023coco, la2024contrasting,guo2024detective,kim-etal-2024-threads,soto2024few,su2023hc3,verma2024ghostbuster}. A related body of work focuses on detecting AI-written code with supervised learning methods \citep{shi2024between, orel2025codet, xu2025distinguishing}. As we are concerned with test-time adaptation, we focus on the works closest to ours, and refer to surveys on LLM detection for more details \citep{wu-etal-2025-survey,xiang2026ai,kumarage2024survey,jawahar2020automatic}. Additional discussion is in \Cref{app:related}. 

Adversarial humanization and domain adaptation are persistent concerns highlighted in the above \citep{dugan-etal-2024-raid, li-etal-2024-mage, liu2025generalization, leibowicz2021deepfake, wang-etal-2024-m4gt}. A supervised-learning response to continual drifts is to incorporate \textit{labeled} test-time data into the training set of the detector \citep{masrour2025damage,hu2023radar,koike2024outfox}. 
Crucially, these approaches are all susceptible to the challenges we articulate: they rely on \textit{labeled} data from an adversarial humanizer, which may be difficult to obtain in practice, and do not address continual shifts, e.g., future humanizers that are specifically designed to evade the current detector. }

One response to distribution shift has been to focus on zero-shot, unsupervised, and semi-supervised methods \citep{mitchell2023detectgpt, bao2023fast, su2023detectllm, gehrmann2019gltr,galle2021unsupervised,tian2023multiscale,11401963,he2025detree,park2025enhancing,hans2024spotting}. Crucially, these methods are largely designed to leverage \textit{priors} (the researcher's or learned) about the distribution of human- and LLM-written text, e.g., that AI repeats higher-order n-grams more often than humans \citep{galle2021unsupervised}---these approaches are robust if the assumptions hold. However, they are not designed to use the test distribution itself, and so they may be susceptible to shifts that break their priors, either adversarially or naturally. Our work is complementary to the above approaches: a detector in practice should both learn from well-collected training data and adapt during inference.

Another line of work has focused on domain transfer methods to LLM detection, e.g., \citep{bhattacharjee2023conda,bhattacharjee2024eagle}. 
Our insight is that these approaches work \textit{when they are used to adapt at test-time}, and perhaps not otherwise. In our experiments, e.g., we adapt \citet{bhattacharjee2023conda} (which assumes a \textit{fixed} target data generator available during development) to adapt to a shifting deployment-time test-time mixture, in which case it is effective.

To our knowledge, test-time adaptation methods have not been proposed or evaluated for \aigen text detection; we believe this adaptation is necessary (cf. \Cref{sec:conceptual}). As such, our work contributes to a growing line of work on the emerging  TTA paradigm  \citep{sun2020test,zhang2022memo,wang2021tent,liang2025comprehensive,wang2022continual,gong2022note,sahoo2020unsupervised}.

\section{Conclusion and limitations}\label{sec:conclusion} 

We theoretically frame and empirically examine continual distribution shifts in LLM text detection, toward (1) understanding fundamental challenges in the task and (2) developing a test-time adaptation approach to mitigate some of these challenges. We find that while ``naive'' AI usage can be reliably detected by existing methods, widely used supervised learning approaches are indeed susceptible to the distribution shifts we examine here; TTA-based methods are more robust. However, test-time adaptation approaches are not a panacea to the challenges in \aigen text detection---there are several limitations to both our framework and fundamental to the task itself. 

{\paragraph{Long-term validity of \aigen text detection.} There is an open question regarding the long-term validity of the task itself---the convergence of the distributions of human and \aigen writing would invalidate any learning method, cf. \Cref{obs:tvbound}. It could be that future LLMs will produce more natural-sounding language---through, e.g., fine-tuning and appropriate scaffolding. Various groups are fine-tuning models to specifically evade detectors such as Pangram \citep{rosmine2026deft, keleta2026accent, gherman2026super}.
In the short term, this ``cat and mouse'' game will continue; what remains unclear is whether convergence to the human distribution (itself changing, and task- and time-dependent) is possible, and thus the long-term theoretical viability of the task.

More generally, our findings carry implications for the long-term limitations of training AI detectors, given the challenges of data curation. In particular, while a given out-of-distribution LLM can typically be (eventually) detected by re-training (via mirror generation), the challenge of shifting human text distribution is more substantial: verified human text from after 2022 will be expensive to obtain, and it is unclear whether even the TTA approach we develop will be robust to many years of human distribution shift. While commercial entities may be able to pay to acquire such data, this challenge will be especially acute for researchers using light-weight approaches. (We note however that paid writing will also differ in distribution to all human writing).

\paragraph{The ethics of text detection.} While we analyze and advance \aigen detection methods, we do not speak to the appropriateness of any particular application of such methods. In many settings, the costs of incorrect detection, especially falsely labeling human writing as LLM-written, may be high. For this reason, in this work, we do not apply our methods to labeling any particular (unknown) piece of text as \aigen, and only report prevalence estimates for such text in \Cref{fig:droptemporal-3}. The use of detection methods in practice should account for concerns such as the cost of mistakes; while our work seeks to reduce the prevalence of labeling errors, it cannot do so perfectly. Our work speaks to the potential of systematic error via distribution shift. Temporal human writing distribution shift in particular produced false positives (real human writing labeled as AI) in our experiments, while the other types of shift primarily affected AI detection rates. Practical applications of text detection require carefully and empirically measuring error rates in realistic distribution shift settings.

\paragraph{Engineering considerations of TTA and other real-world challenges.} While our experiments seek to resemble some real-world dynamics of distribution shifts and approaches to tackle it, applying the ideas to practice would require several additional steps. First, one must reengineer pipelines to adapt models with batches of received data from the API. We note that we do not anticipate significant costs: such detectors are relatively cheap to train, and, as we analyze, our methods are reasonably sample-efficient; using the self-supervised method of \citet{sun2020test} would not even require full-model retraining, and so could work with streaming data.}

\section*{Acknowledgements}

We thank Sarah Dean, Thorsten Joachims, Divya Shanmugam, Vibhhu Sharma, and members of the Garg lab and Cornell AIPP for useful feedback and discussion. This project used API credits provided by Pangram to use their detectors. NG is supported by NSF CAREER IIS-2339427, NASA, the William T. Grant Foundation, and research awards from the Cornell Tech Urban Tech Hub, Google, Mastercard, and Amazon. All opinions are those of the authors and not of any funding agencies, supporters, or other institutions.

\clearpage

\bibliographystyle{plainnat}
\bibliography{example_paper}

\newpage

\appendix
\onecolumn

\section{Additional discussion}\label{app:related}

\subsection{Positive-unlabeled learning}
Beyond our chosen methods, significant prior literature in semi-supervised learning has demonstrated the usefulness of positive-unlabeled (PU) learning. The two tasks of mixture proportion estimation (MPE; equivalent to prevalence estimation) and individual-level classification are intertwined in the PU learning literature, as machine learning classifiers trained with PU methods frequently estimate a class prior $\alpha$ of the proportion of positive data in the unlabeled set. \citet{ivanov2020dedpul}, \citet{bekker2018estimating}, and \citet{ramaswamy2016mixture} propose MPE methods given PU data to estimate $\alpha$. \citet{elkan2008learning, du2014analysis, kiryo2017positive, garg2021mixture} additionally propose methods that, given an estimate or pseudo-estimate of $\alpha$, train a machine learning classifier on the positive and unlabeled data, using the estimate of $\alpha$ to alter the training loss on the unlabeled data. We utilize the methods proposed in \citet{garg2021mixture}; namely, their MPE method, named Best-Bin Estimation (BBE), which assumes that there exists within the distribution of model predictions on positive data a ``top bin,'' above which no predictions on negative data are present, and their model-training algorithm Transform-Estimate-Discard ($\text{(TED)}^{n}$).

\subsection{The validity of \aigen text detection}\label{app:validity}

A growing body of work questions both humans' and algorithms' ability to reliably distinguish \aigen text from human writing. \citet{jakesch2023human} show that people rely on inaccurate heuristics when identifying \aigen text, which can be exploited; i.e., LLMs can be prompted to generate text that is judged as more human-like than actual human writing. \citet{kadoma2025generative} demonstrate that such heuristics systematically disadvantage certain demographic groups, whose human-written text is more likely to be misclassified as \aigen. In contrast, \citet{russell2025people} find that people who frequently use LLMs are better at recognizing \aigen text, and report strong performance by Pangram on their dataset. Our work does not directly weigh in on the applications where AI text detection should be used---some applications may be appropriate, while in others the harms of incorrect detection, especially if disproportionately affecting some groups, may be too high.

 Some works argue that the binary framing of \aigen versus human-written text is increasingly inadequate, noting a wide range of LLM usage categories such as AI-assisted writing and partial editing that may weaken the validity of simple binary prediction models' predictions \citep{thai2025editlens, saha2025almost}. {Our analysis does not consider such mixed text, although we do not anticipate substantial issues adapting our evaluated TTA methods to a continual domain or ternary classification task. However, the detection of AI-assisted writing carries certain validity challenges outside the scope of this work. For instance, it seems unclear at what point it becomes valid for one of the human or LLM to claim primary authorship over the other (i.e., at what point something becomes human- or machine-generated). Related to our discussion on the usefulness of text detection, it is a domain-dependent decision, on which we take no stand, on the extent of AI assistance that it is useful to detect.}

Finally, we note that data distribution shift in the real-world may be even more extreme than what we analyze. For example, we do not analyze performance of our models with multiple kinds of distribution shift simultaneously, such as with both LLMs and humans changing over time; TTA methods may also be susceptible to data poisoning attacks. Handling such shifts may require new kinds of PNU + TTA methods that encode appropriate levels of inductive biases for how much to rely on labeled, out-of-distribution data (labeled positives and negatives) versus the unlabeled, in-distribution data.

\section{Experimental details}

\subsection{Additional experimental details}
\label{app:additionaldetails}

\paragraph{Model training and evaluation.} All PU and supervised models trained by fine-tuning DistilBERT use a learning rate of 0.00001 and 3 epochs. For PU + TTA models, we always train on data drawn from the test distribution that is separate from the evaluation data itself. The procedure of $\text{(TED)}^{n}$ has a calibration step (also using positive and unlabeled data); we ensure that the training set and calibration set for PU + TTA are no larger than the training and validation sets for supervised learning, respectively.
For uncertainty quantification during evaluation, we obtain metrics via bootstrapping---we independently sample multiple train/test splits per-model; furthermore, given a model and test set, we get predictions on each sentence in the test set with the model, and re-sample the predictions 2,500 times to obtain a distribution of point estimates, from which we can obtain confidence bounds. All confidence intervals and error bars in our experiments are at a 95\% level.

\paragraph{Curating human and AI writing.} To curate known \aigen and human writing, we sample 10,000 paper abstracts per year from alternating, even-numbered years 
between 2010 and 2020 (inclusive), yielding 6 years of data and an upper bound  of 60,000 total human-written abstracts combined between all models' training and evaluation sets. For 2020 only, we additionally sample 2,500 human abstracts and mirror them with GPT 5.5 at extra-high effort via Codex. We also sample 10,000 abstracts each from years 2023 and 2025.

We generate \textit{non-adversarial} or naive synthetic mirrors of each abstract using the method described in \citet{liang2024mapping}---the LLM is prompted to summarize the abstract, generate a new abstract from the summary, and check for grammar. In the strategic classification experiments, the LLM (see \Cref{tab:llm-mappings} for an exhaustive list of all LLMs used across experiments) may also be prompted with a strategy generated by an autoresearch loop using Claude Code.
We treat each sentence of each abstract as an individual data point in our training and test data, and perform basic data cleaning by normalizing quote characters, replacing tabs/newlines with spaces, and filtering to printable ASCII characters.
See \Cref{tab:sampled_sentences} and \Cref{tab:sampled_sentences_xz} for examples of \aigen sentences---we remark that while the writer is often difficult to ascertain with the human eye, our detection models can separate the distributions of human and AI writing surprisingly accurately. All models are evaluated on held-out data composed of equal proportions of human- and \aigen sentences, although the construction of the \aigen sentences may differ by evaluation (e.g., be from an LLM not seen during training, using an adversarially generated prompt to fool the model, or from a paper written after the year of training). For evaluation, we collect all metrics described in \Cref{tab:metric-mappings}.

\paragraph{Strategic classification experiments.} For experiments in which we examine adversarially generated \aigen text: we fix the distribution of all human writing to abstracts written in 2020, and vary the distribution of \aigen writing to potentially be from adversarially generated prompts by an autoresearch loop. To train all detection models, 4,000 human abstracts form the labeled positive set; for supervised models, these abstracts and their synthetic mirrors form the labeled training set. For PU + TTA, we use 4,000 additional abstracts and their synthetic mirrors (either adversarially or non-adversarially generated) to construct an unlabeled training set with 75\% of sentences being \aigen; we have 3,000 used for training and 1,000 held out for calibration. If outputs from multiple adversarially generated prompts are available, the unlabeled set consists of evenly-distributed proportions of sentences from each of the prompts. The final 2,000 human abstracts and their mirrors (sometimes in-distribution; sometimes generated using a strategically chosen prompt) are reserved for evaluation. To ensure robustness, we bootstrap 5 training and evaluation splits per trained model; given the same random seed, the supervised and PU models share identical human abstracts in training and an identical evaluation set. All synthetic mirrors are generated with gpt-oss-120b; adversarial prompts are generated using the process in \Cref{app:claude}. For insight on raw data, we include randomly sampled sentences from naive and adversarially generated prompts in \Cref{tab:sampled_sentences_xz}. Finally, the detection models in iterations 1, 2, and 3 in \Cref{fig:adversarial-1} (right) contain training (and evaluation, for iteration 1) data on AI mirrors generated from an optimization loop that began with defeating Pangram 3.2, for computational cost reasons. Specifically, in iteration 1 we train PU + TTA on human writing plus naive AI writing and adversarial AI writing to defeat Pangram 3.2, and supervised detectors on human writing plus naive AI writing (since we assume that supervised learning has delayed access to labeled adversarial AI writing); we evaluate both on the adversarial AI writing targeting Pangram 3.2. In subsequent iterations, we train all detectors on human writing and all naive/adversarial AI writing from previous iterations, use Claude Code to defeat the detector from the current iteration, and evaluate on adversarial AI text attacking the \detectionmodel from the prior iteration (for PU + TTA) or from the current iteration (for supervised). Finally, in \Cref{fig:advgpt} we show that our findings extend to AI text generated using a frontier LLM (GPT 5.5 via Codex).

\paragraph{Natural LLM shift experiments.} For experiments in which we examine performance of \aigen text detectors on out-of-distribution LLM outputs: we fix all analysis to abstracts written in 2020. We evaluate on two sets of LLMs: {gpt-oss-120b \citep{agarwal2025gpt}, Llama 3.3 70B Instruct \citep{grattafiori2024llama}, Gemini 3 Pro Preview \citep{team2023gemini}, Qwen3-Next 80B A3B \citep{yang2025qwen3}, and GPT 5.5}, and a set of five Gemini models of similar architecture but different training dates (Gemini 3 Pro Preview, Gemini 2.5 Flash, Gemini 2.5 Pro, Gemini 2.0 Flash, Gemini 2.0 Flash-Lite), with 2,500 mirrors generated per model. In the experiment in which we assess Pangram's performance on a newly released LLM, we additionally share results for GPT 5.4. For PU + TTA (human-as-positive), we train on the first 50\% of rows (where each row contains a human abstract and its LLM mirror), use another 25\% as a calibration set for the $\hat{\alpha}$ hyperparameter in $\text{(TED)}^{n}$, and hold out the final 25\% as the evaluation set; for supervised models, we train on the first 75\% of rows and hold out the same final 25\% for evaluation. For PU + TTA, we use 25\% of the human abstracts in the training set for the labeled positive set. We use the remaining 75\% of the human-written abstracts in the training set and their synthetic mirrors for the unlabeled set, subsampling by sentence so that 50\% of sentences are negative (\aigen). For supervised models, we mirror this construction, but with the labels swapped and with fully-labeled data (instead of an unlabeled set, we simply use only the human-written abstracts in the PU model's unlabeled set for training). The evaluation set comprises 157 positive abstracts and 468 unlabeled abstracts (from which we sample equal numbers of human and \aigen sentences). For BBE, both PU + TTA and supervised models are evaluated against the same held-out human positives with identical unlabeled data, although the set of \aigen text in the evaluation dataset may be generated using an LLM not encountered at train-time. We bootstrap 5 training and evaluation splits per trained model. For LDA of outputs of different LLMs: all text embeddings are made with Gemini Embedding 2. We sample 5,000 sentences per class (human, each LLM), fit the LDA on a training set comprising 80\% of the available sentences, and plot the remaining 20\% of sentences using the LDA dimensions as a test set.

\paragraph{Natural temporal shift experiments.} For experiments in which we examine temporal shift in human writing over time: for supervised models, we train on 2,500 human abstracts and 2,500 AI abstracts (evenly split between all LLMs used in this experiment). For PU + TTA, we train on a positive set of 2,500 AI abstracts, and an unlabeled set of 1,250 human abstracts and 1,250 AI abstracts. Further, for PU + TTA, we calibrate the $\hat{\alpha}$ in $(\text{TED})^{\text{n}}$ using an additional sample $1/3$ the size of the train set (a positive set of 834 labeled AI abstracts; and an unlabeled set of 417 human and 417 AI abstracts). After making the split by abstract, for the unlabeled set for PU + TTA, we additionally split each abstract into sentences and subsample sentences independently so that they are evenly split between human and AI sentences. 2,000 additional held-out abstracts (evenly split between human-written and synthetic mirrors) are used for each year's test set. Each year maps to its own evaluation set with equal amounts of held-out human-written and \aigen text. We bootstrap 10 training/evaluation splits for each year and training method (supervised, $(\text{TED})^{\text{n}}$). We use four LLMs to generate mirrors for each abstract: gpt-oss-120b, Llama 3.3 70B Instruct, Gemini 3 Pro Preview, and Qwen3-Next 80B A3B. One exception is that we additionally have another set of AI mirrors of human abstracts for 2020, 2023, and 2025 for \Cref{fig:droptemporal-3}. For these, we mirror 10,000 human abstracts per-year with LLMs \{Llama 3.3 70B Instruct, Gemini 3 Pro Preview, gpt-oss-20b, and Gemini 2.5 Flash\}.

\paragraph{Computational resources.} All models are trained using a single A6000 GPU; in total, model training required 53 GPU hours for the models used in all strategic prompting experiments, 30 GPU hours for the models used in all out-of-distribution LLM experiments, and 40 GPU hours for the models used in the temporal shift experiments.

\paragraph{AI usage.} Outside of the use of AI agents as an integrated part of our methodology, we used AI tools throughout this project to write and critique plotting code or make edits throughout a repository (e.g., adjusting the text on a plot label across several plotting scripts, and reading/implementing the methods in related work for additional baselines, such as the work of \citet{bhattacharjee2023conda}, \citet{bao2023fast}, and \citet{sun2020test}). AI tools were additionally used for minor tasks such as searching for relevant papers in the domain, summarizing experimental setup details that were then manually written into the manuscript, translating writing and figures from paper or Google Docs to LaTeX form, and to further iterate on the design of \Cref{fig:pipeline_and_feature_space}. We note that the outputs of all AI tools were individually vetted before integrating the contributions of the tool in the work itself.

\subsection{LLM and metric name mappings}
\label{app:llms}

\begin{table}[h]
\caption{LLM shorthand to full name mapping}
\label{tab:llm-mappings}
\renewcommand{\arraystretch}{1.3}
\begin{tabular}{p{3.5cm} p{4.5cm} p{4.5cm}}
\toprule
\textbf{Shorthand(s)} & \textbf{Full Name} & \textbf{Link} \\
\midrule
OSS & gpt-oss-120b & \url{https://arxiv.org/abs/2508.10925} \\
GPT 5.5 & GPT 5.5; extra-high effort via Codex & \url{https://openai.com/index/introducing-gpt-5-5/} \\
--- & gpt-oss-20b & \url{https://arxiv.org/abs/2508.10925} \\
--- & GPT 5.4 & \url{https://openai.com/chatgpt} \\
Llama, Lla & Llama 3.3 70B Instruct & \url{https://arxiv.org/abs/2407.21783} \\
Gemini 3 Pro, Gemini 3, Gem, 3 Pro & Gemini 3 Pro Preview & \url{https://deepmind.google/technologies/gemini} \\
2.5 Flash & Gemini 2.5 Flash & \url{https://deepmind.google/technologies/gemini} \\
2.5 Pro & Gemini 2.5 Pro & \url{https://deepmind.google/technologies/gemini} \\
2.0 Flash & Gemini 2.0 Flash & \url{https://deepmind.google/technologies/gemini} \\
2.0 Flash-Lite & Gemini 2.0 Flash-Lite & \url{https://deepmind.google/technologies/gemini} \\
Qwen, Qwe & Qwen3-Next 80B A3B & \url{https://arxiv.org/abs/2505.09388} \\
\bottomrule
\end{tabular}
\end{table}

\begin{table}[h]
\centering
\caption{Metric shorthand to full name mapping}
\label{tab:metric-mappings}
\renewcommand{\arraystretch}{1.3}
\begin{tabular}{p{5cm} p{7cm}}
\toprule
\textbf{Shorthand(s)} & \textbf{Full Name} \\
\midrule
Bal. Accuracy & Balanced accuracy on label-balanced test set \\
Human Recall & Accuracy conditioned on human-written text \\
AI Recall & Accuracy conditioned on AI-generated text \\
AUC & Area Under the ROC Curve \\
Avg P(human $\mid$ human) & Mean predicted probability of ``human'' class given human-written input; closely tracks Human recall \\
Avg P(human $\mid$ AI) & Mean predicted probability of ``human'' class given AI-generated input; inversely tracks AI recall (lower is better) \\
Bal. Cross-Entropy & Balanced cross-entropy loss, averaged equally over both classes (lower is better) \\
Bias & Difference between estimated and true prevalence of AI-generated text (relative to 2010 baseline in temporal experiments; closer to 0 is better) \\
Bias Avg P(AI) & Bias in the mean predicted probability of the ``AI'' class, over entire dataset (closer to 0 is better) \\
\bottomrule
\end{tabular}
\end{table}

\clearpage
\newpage

\subsection{\pnu + TTA, main experiments}

{\paragraph{General setup.}
The PNU models combine labeled positive and labeled negative data drawn from a \emph{source} distribution with unlabeled data drawn from a \emph{target} distribution. The labeled source data reuses the labeled training data of the corresponding supervised model for that source domain.
The unlabeled target data reuses the unlabeled training set of the corresponding PU + TTA model for that target domain, and the unlabeled portion of the calibration set reuses the PU + TTA calibration data for the target domain. As in PU + TTA, the calibration set is used to estimate the fraction of positives in the unlabeled training data. In training, the unlabeled target data is partitioned into pseudo-labeled positives and negatives; we place individual, tunable weights on the losses of the labeled positives, labeled negatives, pseudo-labeled positives, and pseudo-labeled negatives. Held-out evaluation data is identical to that used for the supervised and PU + TTA models in each setting, and we bootstrap the same number of train/evaluation splits per setting as the corresponding PU + TTA experiment. We analyze the word counts and propensity for copying from the original abstract in \Cref{sec:copy} and \Cref{sec:eda}, and do not find any evidence that our mirrors have any signs of degenerate behavior (e.g., copying from the original abstract, being trivially short).

\paragraph{Adversarial experiments.}
The source distribution is human writing paired with non-adversarial (naive) mirrors, and the target distribution is the \emph{same} human-writing distribution paired with adversarial mirrors; human writing is the positive class. 
The labeled source data contains 1{,}000 human abstracts as labeled positives and 500 non-adversarial mirrors as labeled negatives; the unlabeled target data comprises 2{,}000 human abstracts and 2{,}000 adversarial mirrors. For calibration, we hold out 500 additional human abstracts as known positives, and draw 1{,}000 human abstracts and 1{,}000 adversarial mirrors for unlabeled target data. All human abstracts across the labeled, unlabeled, calibration, and evaluation sets are mutually disjoint.

\paragraph{Natural LLM shift experiments.}
The source distribution is human writing plus mirrors from a source LLM $X$, and the target distribution is human writing plus mirrors from a different target LLM $Y$; human writing is the positive class. For the labeled training set: from the first 75\% of LLM-$X$ rows (pairs of human abstracts and LLM mirrors, given that X was used to mirror the human abstract), we use the LLM-$X$ mirror sentences as labeled negatives and the human sentences as labeled positives in the same proportion as the supervised model, holding out 10\% of the human abstracts for calibration. The unlabeled target data reuses the PU + TTA unlabeled training set for LLM $Y$ (the first 50\% of LLM-$Y$ rows, subsampled by sentence so half of the sentences are \aigen). The calibration set pairs the held-out LLM-$X$ human positives with the PU + TTA LLM-$Y$ calibration unlabeled data.

\paragraph{Natural temporal shift experiments.}
We simulate a practitioner who holds labeled data anchored to a reference year (2010) but must detect \aigen text in a future year whose abstracts are an unlabeled mixture of human and LLM writing; LLM writing is the positive class. Assuming the practitioner can always mirror contemporary text with a current LLM, the labeled negatives (human writing) are drawn from 2010, while the labeled positives (LLM mirrors) are generated from the test year. The target data is a test-year mixture of human abstracts and LLM mirrors. Concretely, the unlabeled target data reuses the PU + TTA target-year training pool, but with 4{,}500 abstracts rather than 2{,}500: at $\alpha = 0.5$, this yields 2{,}250 human abstracts and 2{,}250 mirrors. The PNU + TTA training split additionally uses 1{,}500 test-year LLM mirrors as labeled positives and 6{,}000 2010 human abstracts as labeled negatives. The calibration set pairs 500 held-out test-year mirror positives with held-out target-year unlabeled data drawn from 1{,}500 abstracts, again mixed at $\alpha = 0.5$ between human abstracts and LLM mirrors. Because the labeled positives and the \aigen portion of the unlabeled set are both produced by mirroring test-year text, which may itself contain \aigen writing, a fraction of the \aigen text consists of mirrors of mirrors (``double mirrors''); as in the PU + TTA setup, we treat this as admissible.}

\subsection{Sampled sentences from raw data (arXiv abstracts)}

To illustrate our raw data, in \Cref{tab:sampled_sentences} and \Cref{tab:sampled_sentences_xz}, we provide randomly sampled sentences from the arXiv data used to train detection models below (for naive writing from each LLM and adversarially generated AI text). For untrained human observers (like the paper's authors), these LLM-generated sentences may be difficult to distinguish from human-written sentences. However, perhaps surprisingly, they can be almost perfectly distinguished---even at the sentence level---by training an in-distribution supervised \detectionmodel, or with PU + TTA or PNU + TTA.

\begin{longtable}{|p{0.75\textwidth}|p{0.18\textwidth}|}
\caption{Sampled sentences by source model; naive prompt from \citet{liang2024mapping}}
\label{tab:sampled_sentences} \\

\hline
\textbf{Sentence} & \textbf{Source} \\
\hline
\endfirsthead

\hline
\textbf{Sentence} & \textbf{Source} \\
\hline
\endhead

\hline
\endfoot

By offering a novel solution to the longstanding challenge of mitigating disturbances in piezo-stepper actuators, this research contributes meaningfully to the advancement of control technologies for high-precision positioning applications. & Llama 3.3 70B Instruct \\
\hline
Furthermore, our results show that time-domain data augmentation enables CPC to achieve state-of-the-art performance on the Zero Speech Benchmark 2017 using an out-of-domain dataset. & Llama 3.3 70B Instruct \\
\hline
Building on the new bounds, the authors establish non-uniform sample complexity results. & Llama 3.3 70B Instruct \\
\hline
The primary objective is to develop \$C\textasciicircum{}1\$ integro cubic splines that not only maintain the shape properties of the input data but also provide a smooth and continuous interpolation. & Llama 3.3 70B Instruct \\
\hline
The investigation is driven by the desire to understand the properties and behaviors of completely regular semigroups in relation to their word problems. & Llama 3.3 70B Instruct \\
\hline
To ensure scalability with large datasets, we utilize sparse Gaussian process approximations. & Gemini 3 Pro Preview \\
\hline
Additionally, the technique proves sufficiently sensitive to distinguish subtle variations in emission characteristics caused by different LED packaging types, establishing it as a robust tool for deep-UV device characterization. & Gemini 3 Pro Preview \\
\hline
Mathematically, this problem translates to finding the shortest path in a graph multiple times, where each iteration utilizes a different set of edge weights corresponding to the specific criterion being optimized. & Gemini 3 Pro Preview \\
\hline
Finally, to capture the full scope of our method, we introduce a general class of "friendly" graph problems. & Gemini 3 Pro Preview \\
\hline
Our experimental results demonstrate that the proposed redistribution method consistently outperforms both the original algorithms and the complete restart approach, offering a robust solution for enhancing global search capabilities and preventing premature convergence. & Gemini 3 Pro Preview \\
\hline
Existing active-learning pipelines for neural networks, however, suffer from two fundamental shortcomings. & gpt-oss-120b \\
\hline
In both scenarios, expert ecologists reported that the visual-interactive tools facilitated the discovery of semantically meaningful behaviours and dramatically reduced the time spent on trial-and-error parameter selection. & gpt-oss-120b \\
\hline
At its core, \textit{freud} provides a robust set of tools for locating particle neighbors in periodic boundary conditions, a capability that underpins virtually every structural and thermodynamic analysis of condensed-matter systems. & gpt-oss-120b \\
\hline
In the low-frequency (infrared) regime, Argentene exhibits an electron-phonon scattering rate and corresponding resistivity that are \textbf{approximately three times smaller} than those of bulk three-dimensional silver, reflecting the suppression of phonon-mediated momentum relaxation in the atomically thin lattice. & gpt-oss-120b \\
\hline
While the openness of the blockchain guarantees transparency and immutability, it also exposes users to a subtle but pervasive problem: a non-trivial portion of tokens become permanently inaccessible, or "locked," after being transferred to accounts that can no longer be controlled. & gpt-oss-120b \\
\hline
Standard data augmentation techniques proved insufficient. & Qwen3-Next 80B A3B \\
\hline
By decoupling bone modeling from physical scarcity, our approach establishes a foundation for personalized, data-driven bone science---moving from observational studies to generative, hypothesis-driven experimentation. & Qwen3-Next 80B A3B \\
\hline
This work demonstrates that joint spatial-frequency diversity, enabled by lightweight feedback, can dramatically enhance the reliability and range of far-field WPT, paving the way for robust, scalable, and energy-autonomous wireless sensor networks without relying on high-end RF infrastructure or sophisticated signal processing. & Qwen3-Next 80B A3B \\
\hline
By enabling high-fidelity, automated delineation of the ONH, fovea, and atrophy borders, this work establishes a new foundation for differential diagnosis. & Qwen3-Next 80B A3B \\
\hline
Building upon this analysis, we curate \textit{BackdoorBench}, a diverse, reproducible database encompassing 18 distinct backdoor attack variants across 7 image classification datasets, spanning trigger types, poisoning rates, and target classes. & Qwen3-Next 80B A3B \\
\hline
The optical Stark effect uses ultrafast light below material absorption to shift resonances at speeds set by pulse duration or coherence. & GPT 5.5 \\
\hline
Several options are proposed, analyzed, and compared to balance mathematical properties with implementation flexibility, with numerical experiments supporting the conclusions. & GPT 5.5 \\
\hline
Continuous-time Bayesian networks model stochastic processes with complex dependence structures, with applications in biology, social science, and medicine. & GPT 5.5 \\
\hline
We present a method for removing short cycles from regular graphs while preserving spectral bounds on nontrivial adjacency eigenvalues, provided the graphs satisfy certain combinatorial conditions. & GPT 5.5 \\
\hline
A geometry transfer network incorporates heterogeneous annotations to improve landmark localization, and a doubly conditional fusion scheme selects relevant feature maps and spatial regions based on the current pose and landmark estimates. & GPT 5.5 \\
\hline

\end{longtable}

\begin{table}[h]
\centering
\caption{Sampled sentences by prompt, written with gpt-oss-120b (naive prompt from \citet{liang2024mapping}), and adversarially generated prompt from an autoresearch loop to bypass Pangram.}
\label{tab:sampled_sentences_xz}
\vspace{1em}
\renewcommand{\arraystretch}{1.4}
\begin{tabular}{|p{0.75\textwidth}|p{0.18\textwidth}|}
\hline
\textbf{Sentence} & \textbf{Source} \\
\hline
For both policies we prove that termination remains guaranteed even when edges are removed over time, i.e., the topology may degrade but never expand. & Naive \\
\hline
The obstacle-avoidance system enables reliable navigation in complex unknown settings while respecting the resource limits of small inexpensive platforms. & Naive \\
\hline
Hard-ODT runs on a Xilinx UltraScale+ FPGA. & Naive \\
\hline
We derive a unified convergence theory stronger than prior analyses. & Naive \\
\hline
Compared with several segmentation baselines, Crossbar-Net yields a mean Dice of 0.89 ± 0.03, true-positive fraction 0.92 ± 0.02, reduces centroid distance by 1.8 mm and Hausdorff distance by 3.4 mm, all statistically significant. & Naive \\
\hline
The algorithm uses triangulation, local conformal tweaks, and global stitching. & Adversarial \\
\hline
Our architecture explicitly separates low-frequency from high-frequency garment deformation modeling; an MLP regresses low-frequency drape from pose, shape, and style, and high-frequency wrinkles are produced by a mixture-of-experts: pose-specific expert networks whose outputs are blended using kernel-derived weights that select experts matching the current pose-shape-style configuration. & Adversarial \\
\hline
Rigid chains increase distinct bonding conformations with flexible partners. & Adversarial \\
\hline
We built an ensemble pipeline to classify tweets; the pipeline creates paragraph-level embeddings using multiple cutting-edge language models like BERT and RoBERTa, and logistic-regression classifiers per embedding are combined by weighted voting. & Adversarial \\
\hline
The resulting compression ratios imply peak temperatures around 3000 K, sufficient to dissociate water vapour and generate radicals for sonochemical activity. & Adversarial \\
\hline
\end{tabular}

\end{table}

\newpage

\subsection{Prompt optimization autoresearch details}
\label{app:claude}

Below, we describe an autoresearch loop in which we utilize a coding agent to generate prompts that adversarially target each \detectionmodel (runs of an ``outer loop'', where at each iteration we use the coding agent to target a specific \detectionmodel, and train a new \detectionmodel using all prior naive and adversarially generated AI text), producing \Cref{fig:adversarial} and \Cref{app:adversarial}. For iterations in which the \detectionmodel is trained with PU + TTA, the outer loop uses the version of the \detectionmodel trained in the previous iteration of the outer loop, and not test-time adaptation applied within the current iteration.

Concretely, for the last released \detectionmodel, we initially provide Claude Code with a repository composed of:

\begin{itemize}
    \item a CSV file with 15 fixed synthetic mirrors of human abstracts generated using gpt-oss-120b (a validation set) and the prompting strategy from \citet{liang2024mapping}, along with predictions from the \detectionmodel (either Pangram by giving the agent access to Pangram's API, or a semi-supervised or supervised \detectionmodel that we train as the detector); abbreviated as \texttt{results.csv}
    \item a Python script with code that takes in a prompt, and generates synthetic mirrors of the 15 original \aigen texts with that prompt, using gpt-oss-120b
    \item a Python file containing the skeleton for a prompt; initialized to a mock function (returns the input text directly)
    \item a .txt file summarizing the repository and the agent's goal
    \item a utility file containing simple functions to replace non-ASCII characters in a string and split a string into sentences
    \item a test set of 50 held-out synthetic mirrors generated using gpt-oss-120b and the prompting strategy from \citet{liang2024mapping}

\end{itemize}

Given this initial repository, we then proceed in multiple iterations (calls to an ``inner loop'', where at each iteration Claude uses its findings from previous iterations of the inner loop to adversarially target the same \detectionmodel) of coding agent prompting. At each iteration of the inner loop, we ask Claude Code to:

\begin{enumerate}
    \item read the .txt file of context and instructions
    \item analyze all \texttt{results.csv} files generated at previous iterations of the inner loop; report natural language patterns that yielded high probability of being human writing by the \detectionmodel
    \item write one or multiple new candidate prompt functions
    \item run the Python script to generate new synthetic mirrors of the 15 validation abstracts
    \item pass each set of 15 new synthetic mirrors through the \detectionmodel; save each mirror and its predictions from the model to a new \texttt{results.csv} file
    \item write reflections to a new file
\end{enumerate}

The inner loop is iteratively run until Claude Code successfully generates a prompt whose synthetic mirrors are usually predicted as human-written (i.e., $\mathbb{E}\left[\Pr(\text{human} \mid \text{AI})\right] \geq 0.5$), at which point we confirm good generalization on the test set. Finally, we mirror all 10,000 available \aigen abstracts written with the prompting strategy from \citet{liang2024mapping}, using the new prompt, and execute the next iteration of the outer loop.

Anecdotally, we find that each successive \detectionmodel, whether trained using supervised or semi-supervised methods, was increasingly ``difficult'' for the agent to adversarially attack, perhaps suggesting that we were training a successively more robust \detectionmodel at each iteration. For instance, we found that certain strategies that yielded prompts that successfully bypassed the \detectionmodel on earlier iterations of the outer loop (e.g., prompting the agent to generate one prompt at a time; prompting the agent to generate a prompt immediately after reading the \texttt{results.csv} files, without an intermediate step to brainstorm high-level patterns to aim for) no longer worked on later iterations, forcing us to adjust our prompting strategies over iterations. Additionally, the number of candidate prompts that had to be generated by the coding agent in order to find a single prompt that bypassed the \detectionmodel grew as the number of adversarial prompts the \detectionmodel had been trained on increased (e.g., for a supervised classifier, from 10 prompts in the first iteration to bypass Pangram, to 15 prompts in the second iteration, to 35 prompts in the final iteration).

\clearpage
\newpage

\section{Proofs for observations in \Cref{sec:conceptual}}\label{app:theory}
\begin{proof}[Proof of \Cref{obs:tvbound}]
For any detector \(f:\mathcal X\to\{0,1\}\), define
\[
S_f=\{x\in\mathcal X:f(x)=1\},
\]
the set of texts classified as human. Then
\[
\Pr_{X\sim \mathcal{H}}[f(X)=1]=\mathcal{H}(S_f),
\qquad
\Pr_{X\sim \mathcal G}[f(X)=0]=\mathcal G(S_f^c)=1-\mathcal G(S_f).
\]
Thus the balanced accuracy of \(f\) is
\[
\begin{aligned}
\operatorname{Acc}(f;\mathcal{H},\mathcal G)
&=
\frac{1}{2}\Pr_{X\sim \mathcal{H}}[f(X)=1]
+
\frac{1}{2}\Pr_{X\sim \mathcal G}[f(X)=0]\\
&=
\frac{1}{2}\mathcal{H}(S_f)
+
\frac{1}{2}\bigl(1-\mathcal G(S_f)\bigr)\\
&=
\frac{1}{2}
+
\frac{1}{2}\bigl(\mathcal{H}(S_f)-\mathcal G(S_f)\bigr).
\end{aligned}
\]
By definition of total variation distance,
\[
\operatorname{TV}(\mathcal{H},\mathcal G)
=
\sup_{S\subseteq\mathcal X}
\left|\mathcal{H}(S)-\mathcal G(S)\right|.
\]
Therefore,
\[
\mathcal{H}(S_f)-\mathcal G(S_f)
\le
\operatorname{TV}(\mathcal{H},\mathcal G),
\]
and hence
\[
\operatorname{Acc}(f;\mathcal{H},\mathcal G)
\le
\frac{1+\operatorname{TV}(\mathcal{H},\mathcal G)}{2}.
\]

\end{proof}

We note that this observation also bounds the possibility of prevalence estimation.

\begin{proof}[Proof of \Cref{obs:supfails}]
Write $\mathcal B:=h(\mathcal A)$. Let $f^\star\in\mathcal F$ be a separator of
$\mathcal{H}$ from $\mathcal A\cup\mathcal B$, and write
$S^\star=S_{f^\star}$. Thus
\[
\mathcal{H}(S^\star)=1,\qquad
\mathcal A(S^\star)=0,\qquad
\mathcal B(S^\star)=0.
\]

Formalizing the non-triviality of $h$, suppose there is a measurable set $D$ such that
\[
\mathcal A(D)=0,\qquad \mathcal B(D)>0.
\]
Let
\[
C := D\cap (S^\star)^c.
\]
Since $\mathcal B(S^\star)=0$, we still have $\mathcal B(C)=\mathcal B(D)>0$.
Also, $\mathcal A(C)=0$ because $C\subseteq D$, and $\mathcal{H}(C)=0$ because
$C\subseteq (S^\star)^c$ and $\mathcal{H}(S^\star)=1$.

By ``sufficiently rich,'' we mean that $\mathcal F$ also contains the classifier
$\tilde f$ whose human-prediction set is
\[
S_{\tilde f}=S^\star\cup C.
\]
Then
\[
\mathcal{H}(S_{\tilde f})=1
\qquad\text{and}\qquad
\mathcal A(S_{\tilde f})=0,
\]
so $\tilde f$ perfectly separates $\mathcal{H}$ from $\mathcal A$. However,
\[
\mathcal B(S_{\tilde f})
=
\mathcal B(S^\star\cup C)
\ge
\mathcal B(C)
>
0,
\]
so $\tilde f$ classifies a positive-measure subset of shifted AI text as human. Thus
$\tilde f$ achieves perfect training-time accuracy on $\mathcal{H}$ versus $\mathcal A$
but does not separate $\mathcal{H}$ from $h(\mathcal A)$.
\end{proof}

\begin{proof}[Proof of \Cref{obs:puworks}]
For any detector $f:\mathcal X\to\{0,1\}$, define
\[
S_f=\{x\in\mathcal X:f(x)=1\},
\]
where $f(x)=1$ denotes prediction as human-written.
Write $\mathcal B=h(\mathcal{A})$ for brevity. By ``$f$ separates $\mathcal{H}$ from $\mathcal A\cup h(\mathcal A)$,'' we mean, as in the proof of Observation 2.2,
\[
\mathcal{H}(S_f)=1,\qquad \mathcal A(S_f)=0,\qquad \mathcal B(S_f)=0.
\]

Consider the following population positive-unlabeled objective:
\[
\min_{f\in\mathcal F} \mathcal U(S_f)
\quad\text{subject to}\quad
\mathcal{H}(S_f)=1,
\tag{1}
\]
where
\[
\mathcal U=\lambda_{{H}} \mathcal H+\lambda_A \mathcal A+\lambda_{{h}} \mathcal B,
\qquad \lambda_A,\lambda_{{h}}>0.
\]
This population objective can be instantiated using only labeled samples from $\mathcal{H}$ to enforce the constraint and unlabeled
samples from $\mathcal{U}$ to evaluate the objective.

By the separability assumption in Section 2, there exists $f^\star\in\mathcal F$ such that
\[
\mathcal H(S_{f^\star})=1,\qquad \mathcal  A(S_{f^\star})=0,\qquad \mathcal B(S_{f^\star})=0.
\]
For this classifier,
\[
\mathcal U(S_{f^\star})
=
\lambda_{{H}} \mathcal H(S_{f^\star})
+
\lambda_A \mathcal A(S_{f^\star})
+
\lambda_{{h}} \mathcal B(S_{f^\star})
=
\lambda_{{H}}.
\]
Thus the optimum value of (1) is at most $\lambda_{{H}}$.

Now let $f$ be any feasible classifier, so $\mathcal{H}(S_f)=1$. Then
\[
\mathcal U(S_f)
=
\lambda_{{H}} \mathcal{H}(S_f)
+
\lambda_A \mathcal A(S_f)
+
\lambda_{{h}} \mathcal B(S_f)
=
\lambda_{{H}}+\lambda_A \mathcal A(S_f)+\lambda_{{h}} \mathcal B(S_f)
\ge \lambda_{{H}}.
\]
Therefore the optimum value of (1) is exactly $\lambda_{{H}}$.

Finally, if $f$ is a population minimizer of (1), then feasibility gives $\mathcal{H}(S_f)=1$ and
optimality gives
\[
\lambda_{{H}}+\lambda_A \mathcal A(S_f)+\lambda_{{h}} \mathcal B(S_f)=\lambda_{{H}}.
\]
Hence
\[
\lambda_A \mathcal A(S_f)+\lambda_{{h}} \mathcal B(S_f)=0.
\]
Since $\lambda_A,\lambda_{{h}}>0$ and both probabilities are nonnegative,
\[
\mathcal A(S_f)=0
\qquad\text{and}\qquad
\mathcal B(S_f)=0.
\]
Thus every population minimizer separates $\mathcal{H}$ from $\mathcal A\cup h(\mathcal A)$, up to null sets. Conversely,
any classifier satisfying
\[
\mathcal{H}(S_f)=1,\qquad \mathcal A(S_f)=0,\qquad \mathcal B(S_f)=0
\]
achieves objective value $\lambda_{{H}}$ and is therefore a population minimizer. 
\end{proof}

\esif{\subsection{SemEval Dataset}
\begin{itemize}
    \item fine tuning / epoch / lr / wd info
    \item dataset sizes
    \item how you curated pn vs pu sets
\end{itemize}}

\newpage

\section{Additional empirical results}\label{app:figs}

\subsection{Adversarial behavior}
\label{app:adversarial}

\begin{figure}[ht]

  \centering
  \includegraphics[width=0.8\linewidth]{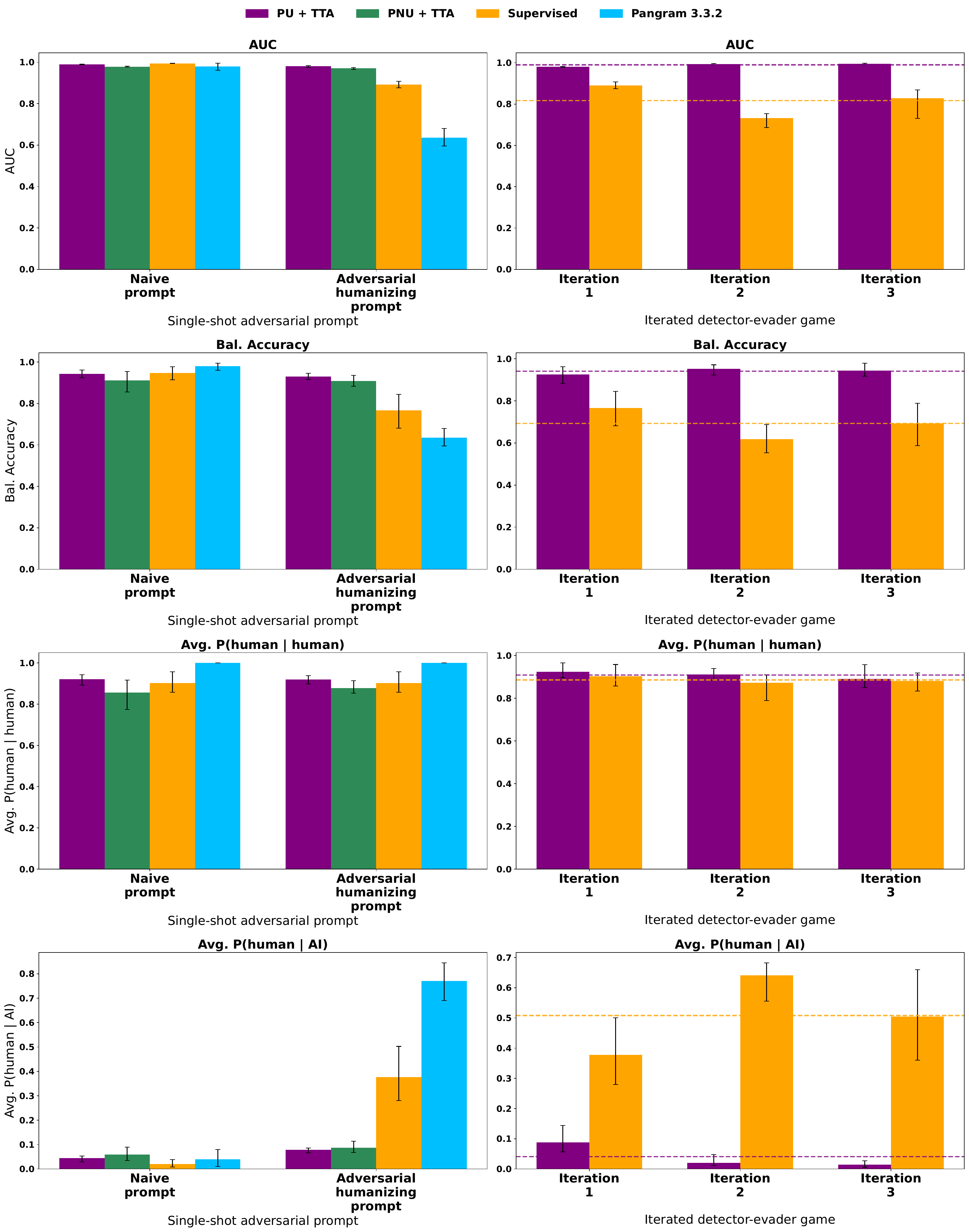}

\caption{Plotting the same experiment as in \Cref{fig:adversarial-1} but reporting AUC, balanced accuracy, average predicted probability of being human on human writing, and average predicted probability of being human on AI writing. Note that AUC, balanced accuracy, and $\Pr(\text{human} \mid \text{AI})$ degrade significantly for supervised learning, although PU + TTA  and \pnu + TTA are comparatively stable. Meanwhile, supervised learning and TTA perform comparably on human text. (1/2)}

\label{fig:extraline1}
\end{figure}

\begin{figure}[ht]

  \centering
  \includegraphics[width=0.8\linewidth]{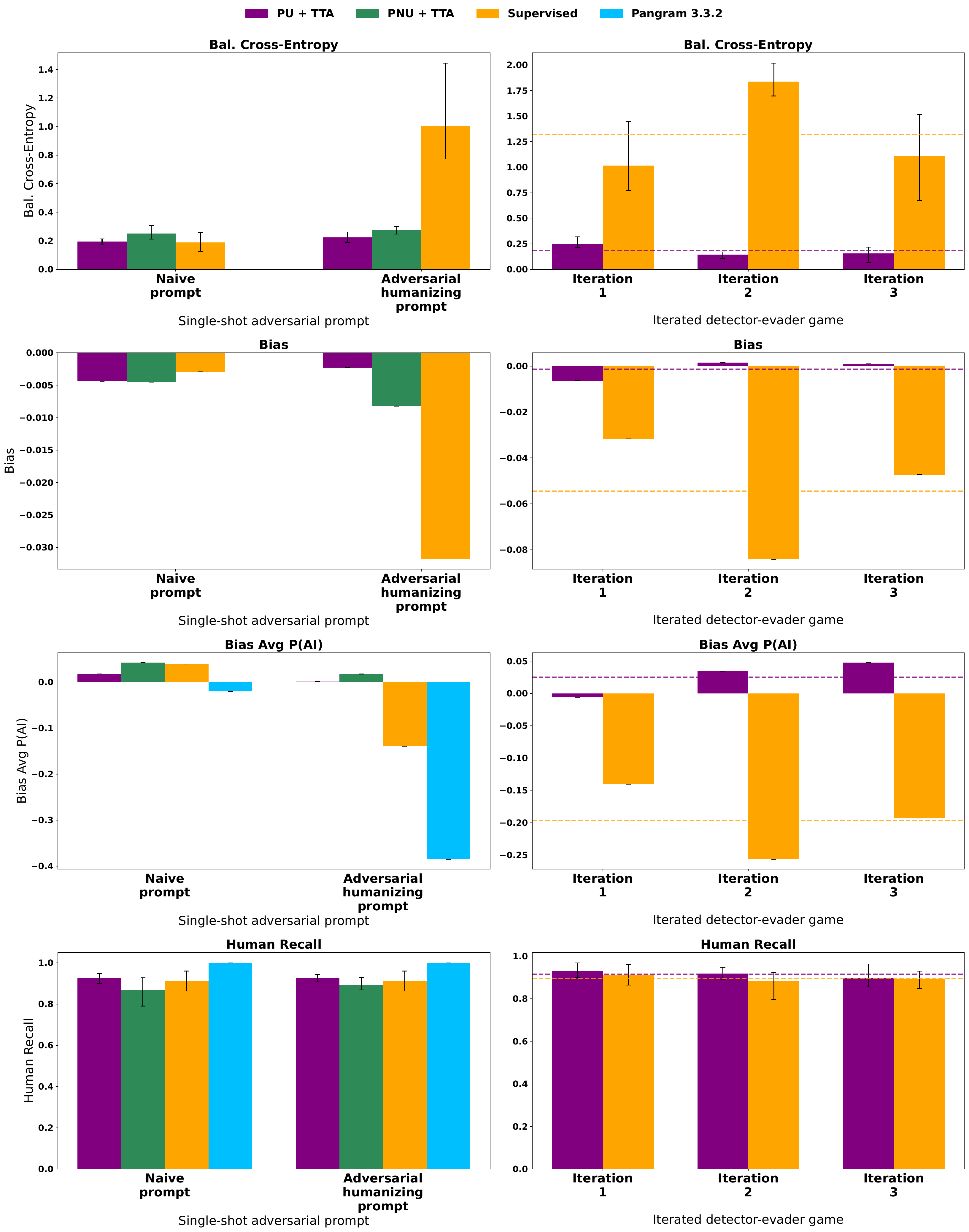}

\caption{Plotting the same experiment as in \Cref{fig:adversarial-1} but reporting balanced cross-entropy, bias on a label-balanced proportion estimation task as measured with BBE, bias on the same proportion estimation task as measured with average $\Pr(\text{AI})$ over the entire test set, and recall on human text. Note that balanced cross-entropy, and bias on prevalence estimation tasks degrade significantly for supervised learning, although TTA methods are comparatively stable. Meanwhile, supervised learning and TTA perform comparably on human recall. (2/2)}

\label{fig:extraline2}
\end{figure}

\begin{figure}[ht]

  \centering
  \includegraphics[width=\linewidth]{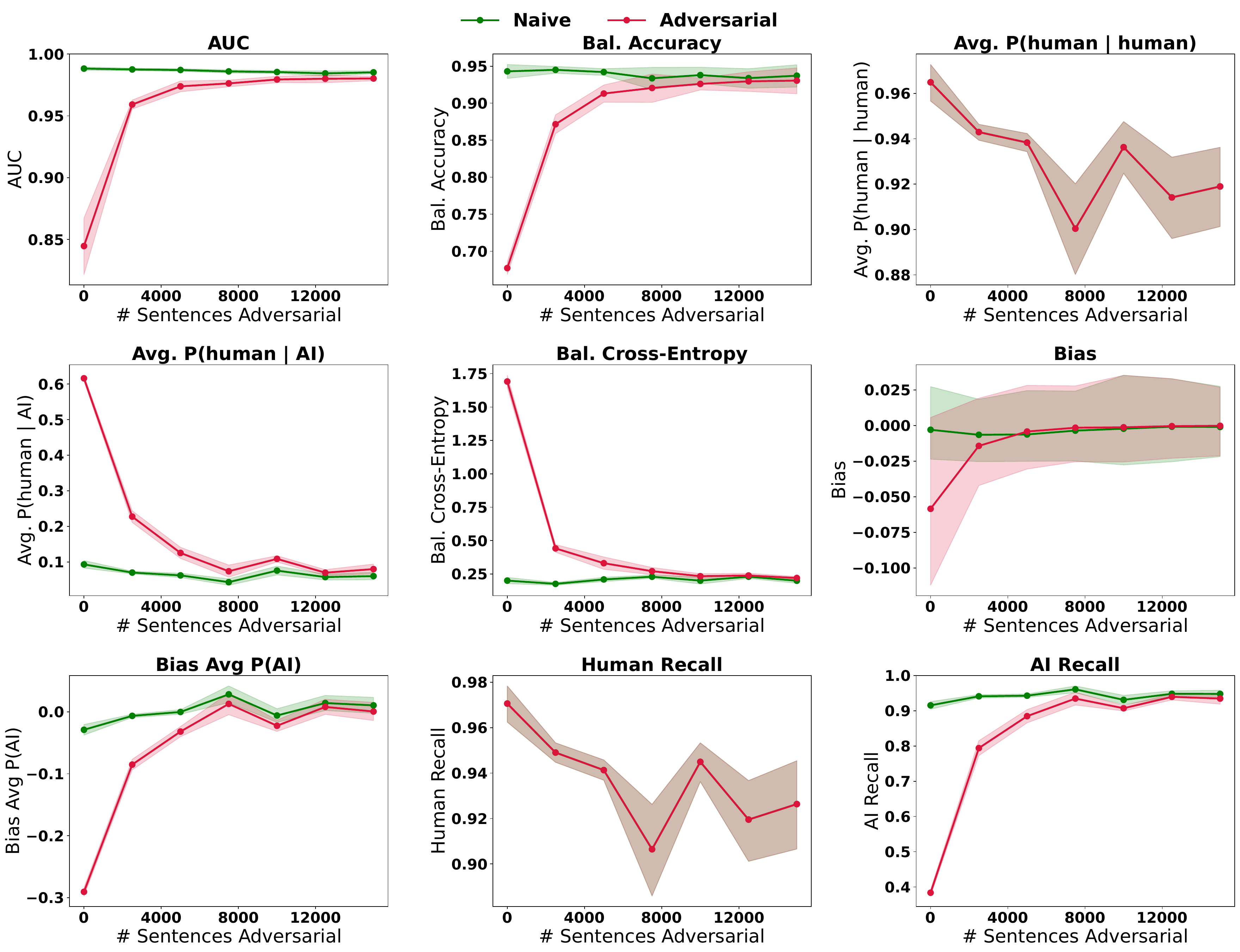}

\caption{Plotting metrics as the number of adversarial sentences increases for PU + TTA (the number of unlabeled naive AI sentences is fixed at 15,000, and the number of human sentences is fixed at 5,000). As the number of sentences from the adversarial prompt increases in the unlabeled training data for PU + TTA, the model rapidly approaches the strong performance achieved on naive \aigen text, but for the adversarial text.}

\label{fig:slidinggrid}
\end{figure}

\begin{figure}[tbh]

\begin{subfigure}{.5\linewidth}
  \centering
  \includegraphics[width=\linewidth]{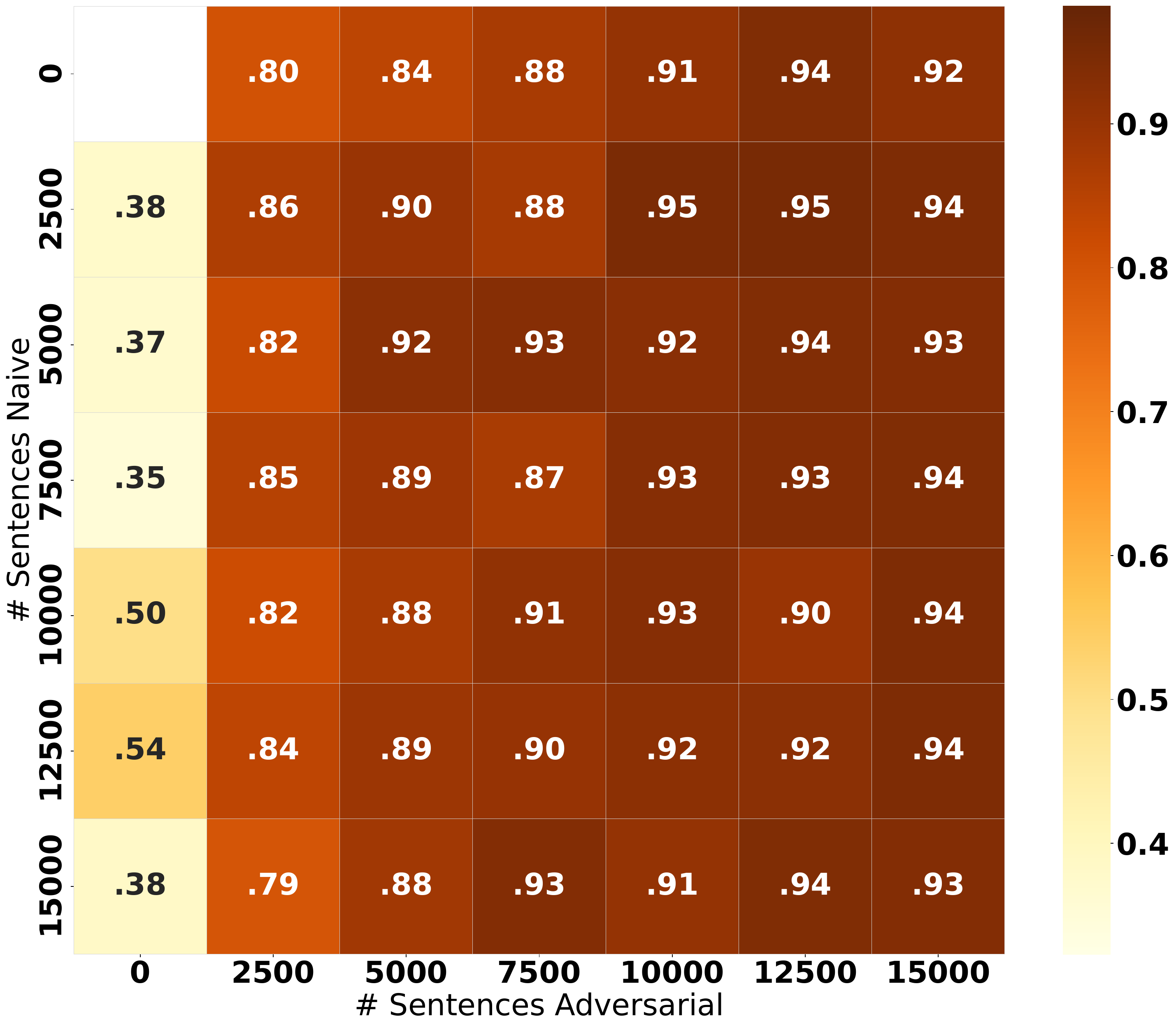}
  \caption{}
  \label{fig:mini-1}
\end{subfigure}\begin{subfigure}{.5\linewidth}
  \centering
  \includegraphics[width=\linewidth]{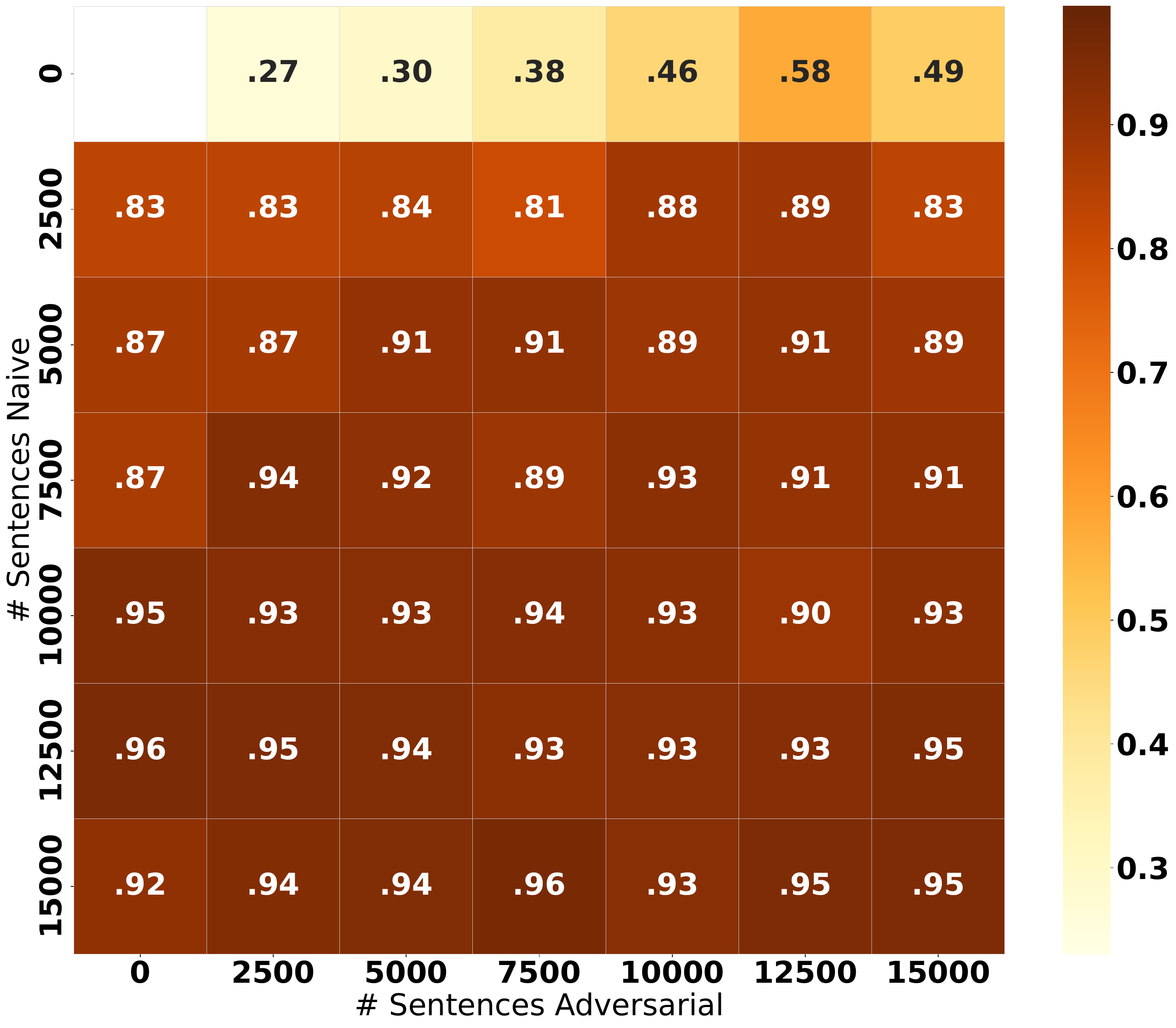}
  \caption{}
  \label{fig:mini-2}
\end{subfigure}

\begin{subfigure}{.5\linewidth}
  \centering
  \includegraphics[width=\linewidth]{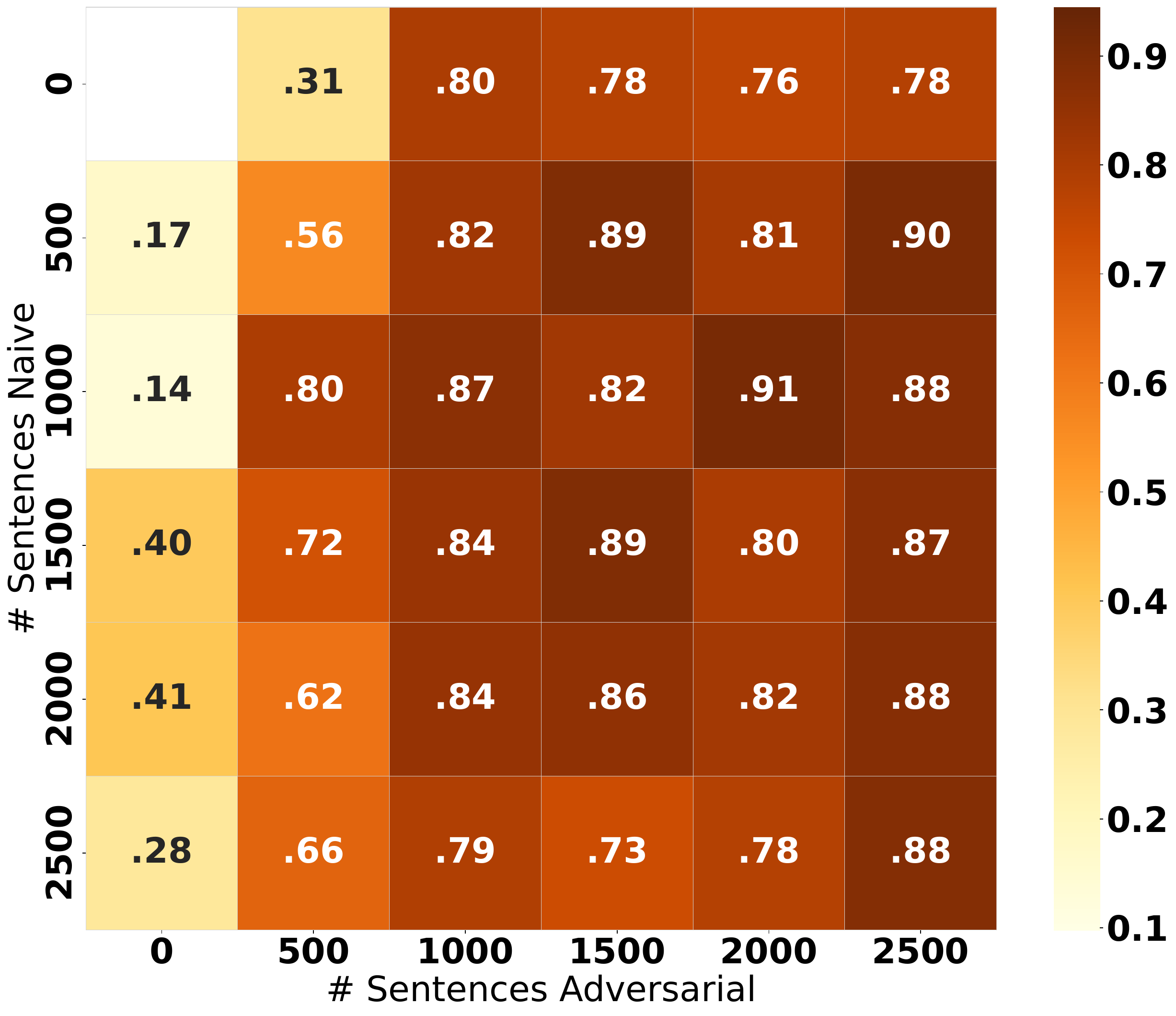}
  \caption{}
  \label{fig:mini-3}
\end{subfigure}\begin{subfigure}{.5\linewidth}
  \centering
  \includegraphics[width=\linewidth]{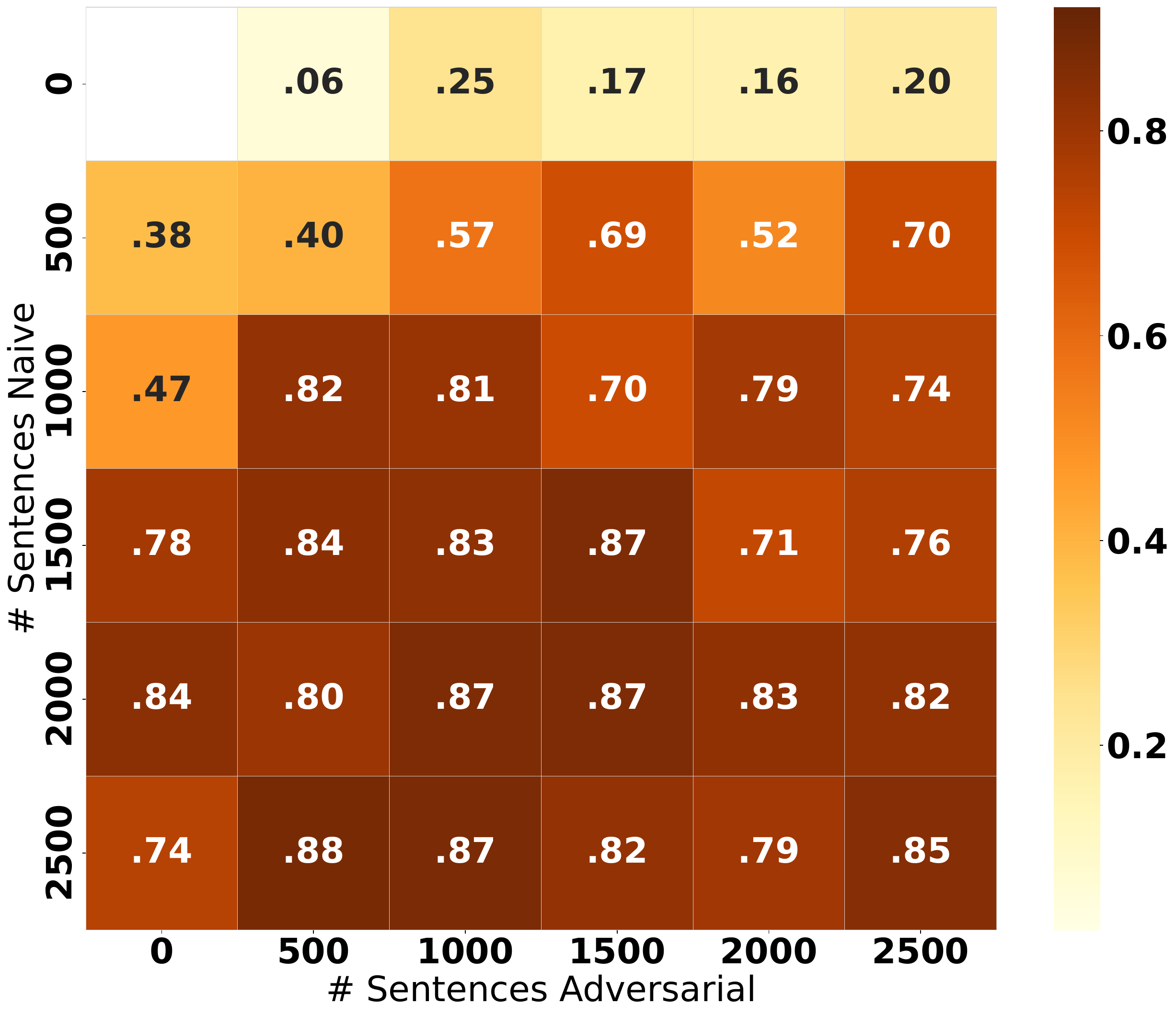}
  \caption{}
  \label{fig:mini-4}
\end{subfigure}

\caption{{Sensitivity analysis for introducing varied amounts of sentences from naive \aigen text and adversarial \aigen text into the unlabeled data during training, for PU + TTA. The total number of human-written sentences in the unlabeled set is fixed at 5000, and we show AI recall. \ref{fig:mini-1} and \ref{fig:mini-3} report performance on AI text written with the adversarial prompt while \ref{fig:mini-2} and \ref{fig:mini-4} report it on text written with the naive prompt; \ref{fig:mini-1} and \ref{fig:mini-2} span training counts from 0 to 15{,}000 sentences, whereas \ref{fig:mini-3} and \ref{fig:mini-4} zoom into the 0-2{,}500 range. Small amounts of each kind of \aigen (unlabeled) text are sufficient for performance that substantially outperforms supervised learning.}}

\label{fig:mini}
\end{figure}

\begin{figure}[tbh]
\begin{subfigure}{.5\linewidth}
  \centering
  \includegraphics[width=\linewidth]{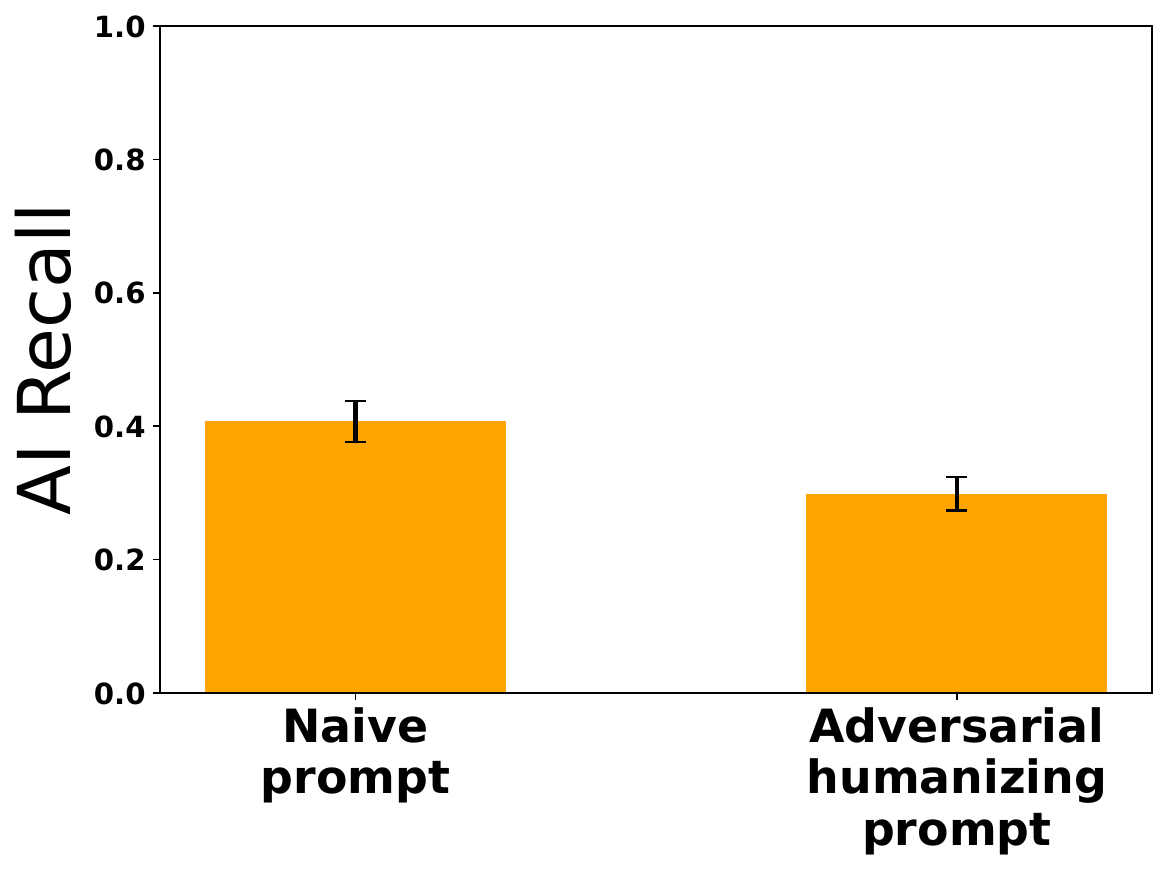}
  \caption{}
  \label{fig:advgpt-1}
\end{subfigure}\begin{subfigure}{.5\linewidth}
  \centering
  \includegraphics[width=\linewidth]{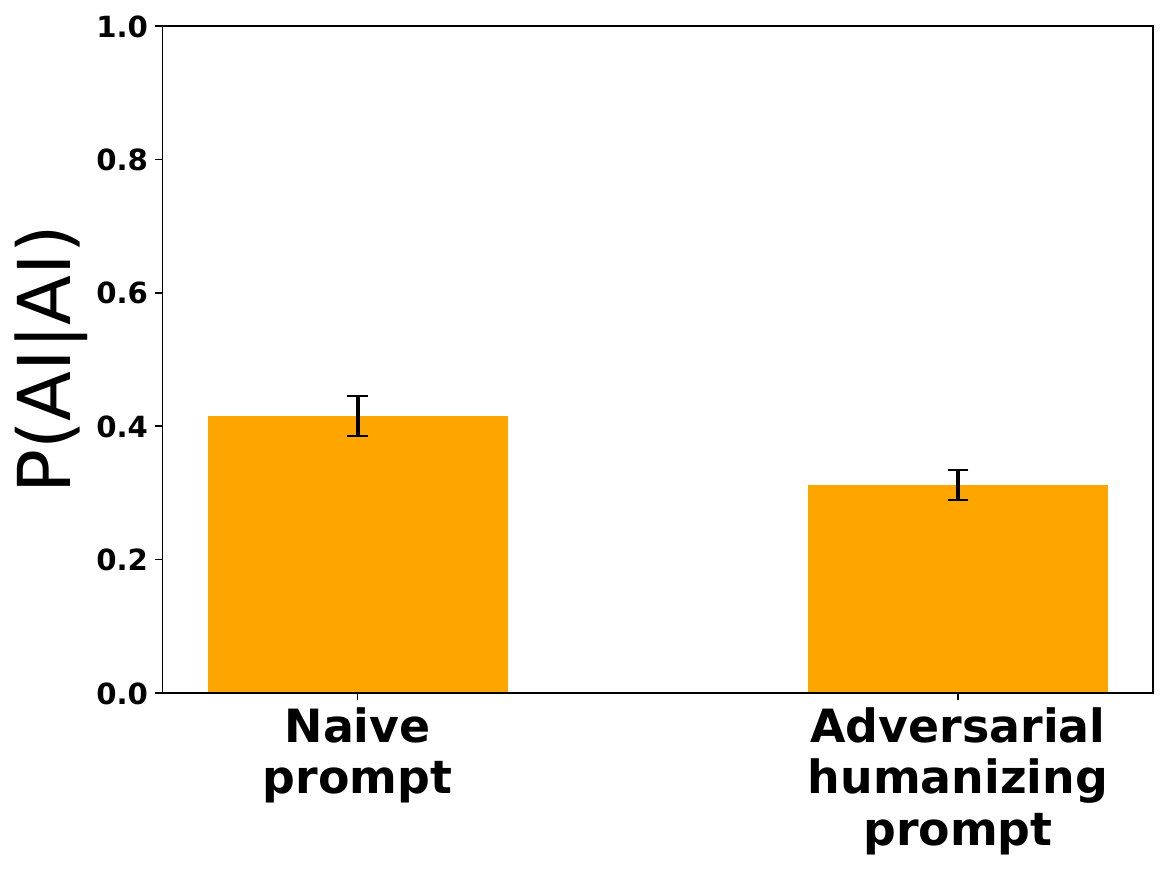}
  \caption{}
  \label{fig:advgpt-2}
\end{subfigure}

\caption{Detector performance for our supervised models (trained on human writing and AI sentences from the prompt in \citet{liang2024mapping}), for AI mirrors written with the prompts from \citet{liang2024mapping} and the adversarial prompt used to defeat Pangram 3.3.2 (through GPT 5.5 via Codex). For robustness to multiple metrics we plot results for AI Recall and $\Pr(\text{AI}\mid \text{AI})$. We find that our findings in \Cref{fig:adversarial-1} extend to frontier LLMs. Furthermore, the effects of evaluating on out-of-distribution LLMs and out-of-distribution prompts stack: the scores are uniformly low because the LLM used to generate test-time mirrors differs from the train-time mirrors, and evaluating on the adversarial prompt further lowers detection rate.}

\label{fig:advgpt}
\end{figure}

\clearpage
\newpage

\subsection{Validating adversarial prompts as humanizers}
\label{sec:mirrorvalidation}

\paragraph{Motivation.} Here, we evaluate our adversarial prompts as humanizers. In particular, do they preserve the substance and not hallucinate content? While our goal is not to create a highly-performant humanizer (i.e., our humanizer prompts are developed strictly to illustrate the point that detection models can be evaded, rather than to produce AI text that real-world users would want to use), it would be undesirable for any errors in preserving the original substance to \textit{correlate} with detectability---i.e., the \detectionmodel learns to classify text by whether it contains AI hallucinations, rather than by inherent stylistic characteristics of AI writing; or if the way that the prompt defeats a classifier is by outputting a single short text that largely omits information from the original text.

\paragraph{Methods.} We utilize an LLM-as-a-judge framework \citep{zheng2023judging, gu2024survey} to assess the degree to which all prompts used in our experiments yield synthetic mirrors that hallucinate information not contained in the original writing, or omit information that was in the original writing (for the naive prompt from \citet{liang2024mapping}, the original text is a human-written abstract; for all adversarially generated prompts, the original text is the synthetic mirror from the naive prompt). Concretely, we sample 100 human abstracts written in 2020, along with their rewrites using all prompts used in the strategic classification experiments. Then, we prompt an LLM (GPT 5.5 on extra-high effort, via Codex) to score each abstract on a continuous scale by (1) how free the rewritten text is from hallucinated or added information not present in the original (1 represents no hallucination and 0 represents a fully hallucinated abstract); (2) how completely the rewritten text preserves the crucial information from the original (1 represents no omission of substance from the original and 0 represents a rewrite that has omitted all key information). Then, we analyze whether detectability correlates with high hallucinations or low content fidelity.

\paragraph{Findings.} In \Cref{fig:judge}, we plot the marginal distribution of scores per-prompt, and also regress the LLM judge scores per-abstract against the predicted probability of being human-written (each prompt is evaluated using the same \detectionmodel as in \Cref{fig:adversarial}). Our findings validate that (1) our prompts generally preserve the information contained in the original abstract well, especially when compared to the naive prompt (used in other work) whose outputs it was rewriting; and (2) the degree to which a synthetic mirror hallucinates or omits information correlates poorly with predictions from the \detectionmodel. We note one prompt where hallucination and omission were more common (Iteration 3; PU + TTA), although we find no evidence that either issue correlated with predictions from the \detectionmodel (i.e., the LLM's deviation from the original abstract's content does not explain the success of the prompt evading the \detectionmodel). Thus, the humanizing prompts are sufficiently high quality for our analyses. 

\begin{figure}[ht]
\begin{subfigure}{.5\linewidth}
  \centering
  \includegraphics[width=\linewidth]{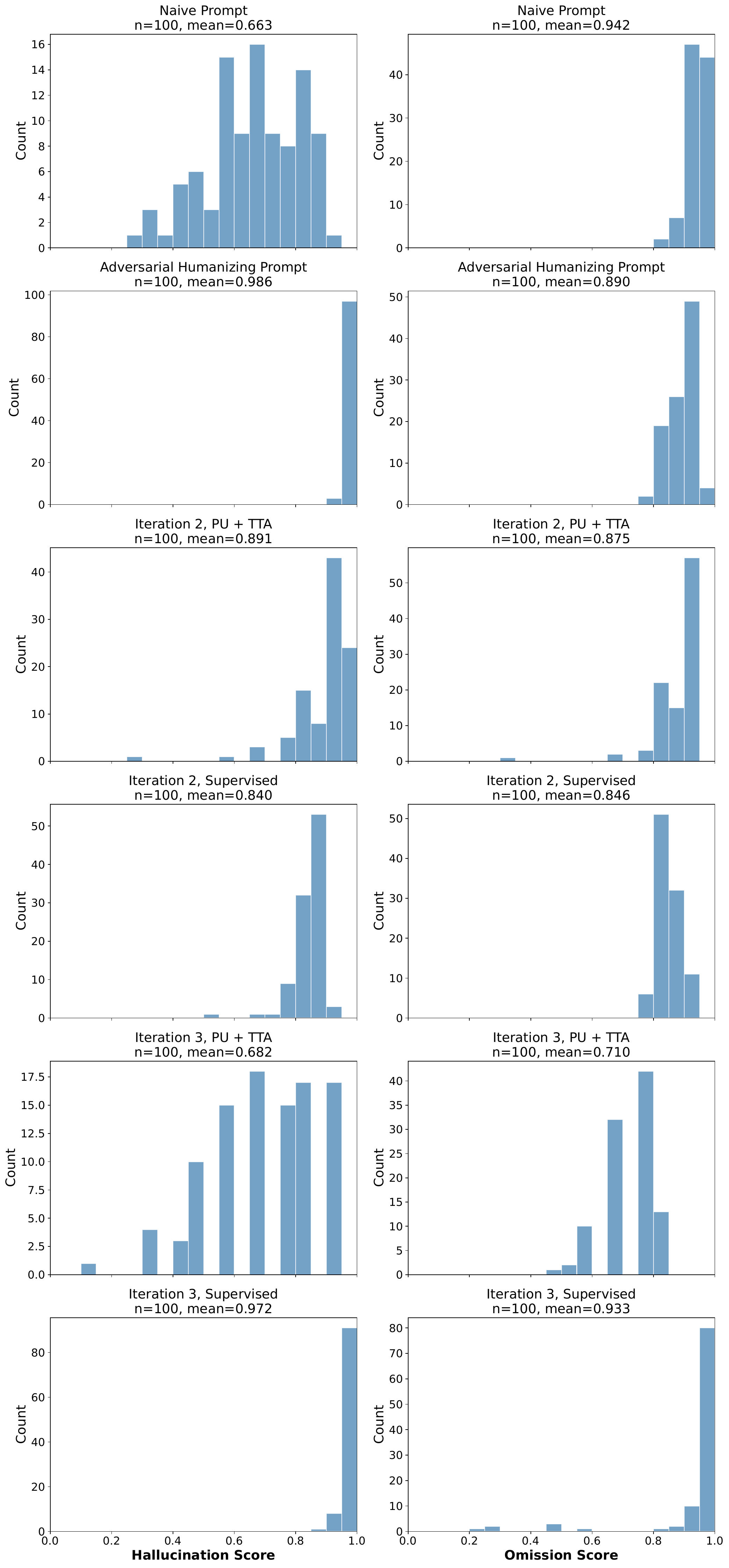}
  \caption{}
  \label{fig:judge-1}
\end{subfigure}\begin{subfigure}{.5\linewidth}
  \centering
  \includegraphics[width=\linewidth]{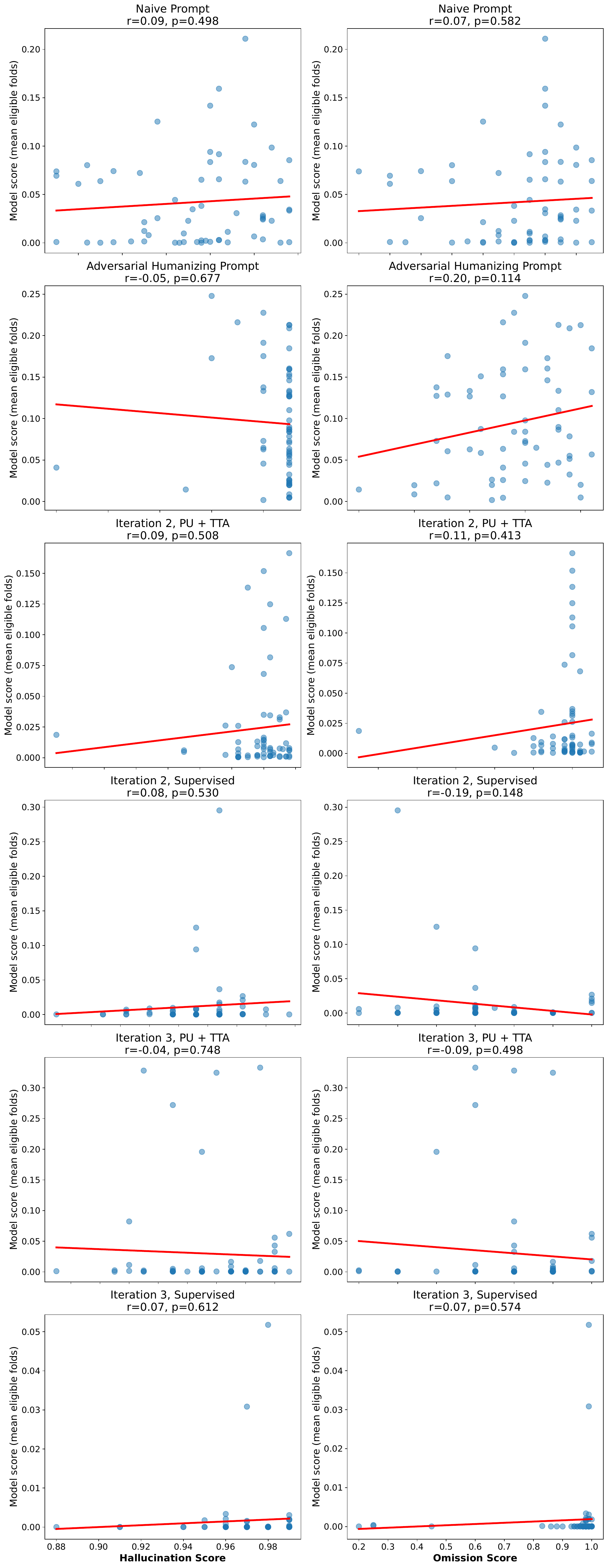}
  \caption{}
  \label{fig:judge-2}
\end{subfigure}

\caption{(\ref{fig:judge-1}) Distribution of hallucination scores and omission scores per-prompt. We generally preserve the content that was originally in the human abstract. (\ref{fig:judge-2}) The scores do not correlate with the predictions of our detection models (same \detectionmodel-prompt mapping as in \Cref{fig:adversarial}), indicating that reasoning breakdowns in the LLM itself do not induce spurious correlations in our experiments.}

\label{fig:judge}
\end{figure}

\clearpage
\newpage

\subsection{Heatmaps with non-Gemini models, out-of-distribution LLMs}
\label{app:nongemini}

\begin{figure}[ht]

  \centering
  \includegraphics[width=\linewidth]{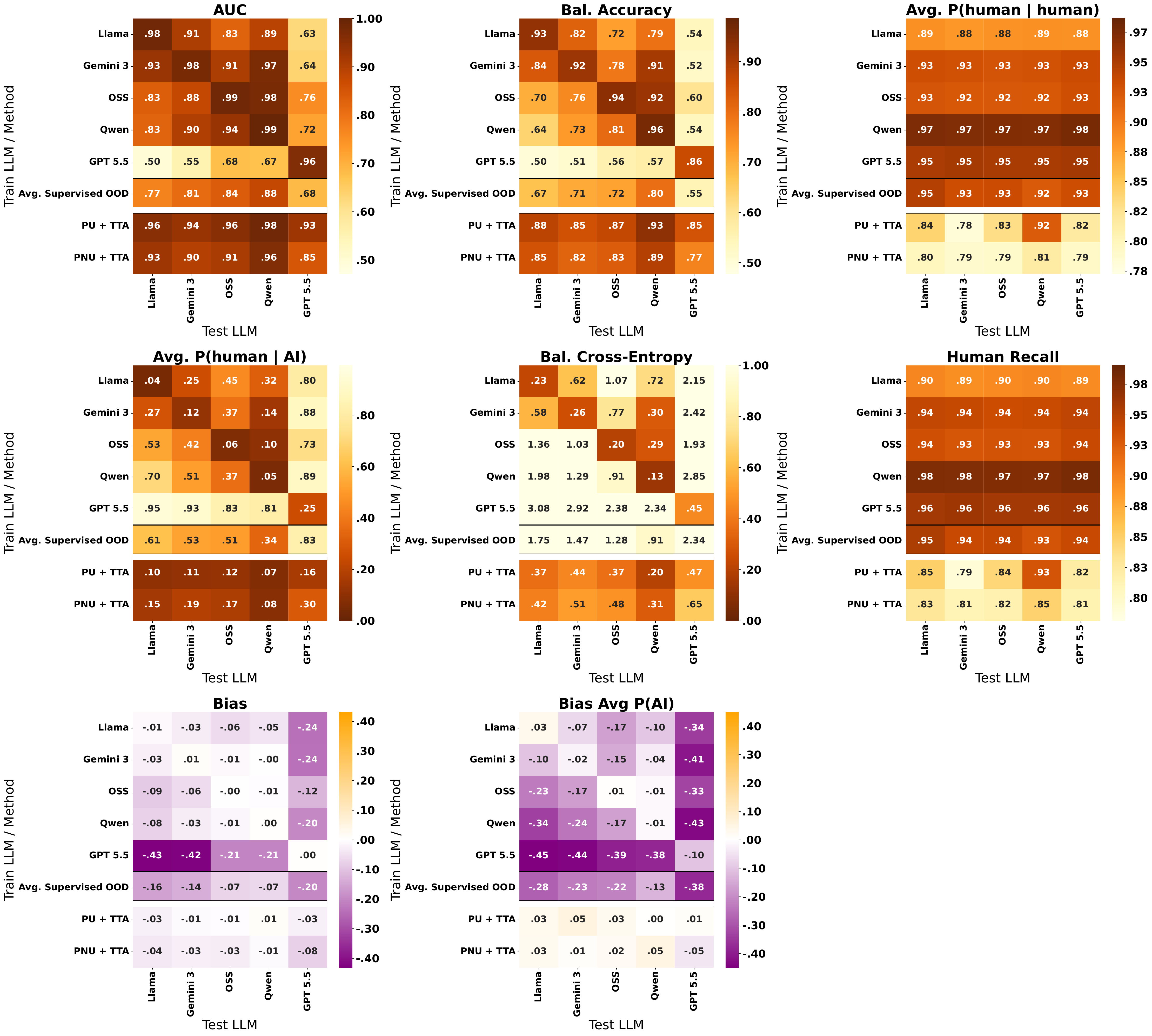}

\caption{Plotting the same experiment as in \Cref{fig:llm-shift}, but reporting all metrics other than in the main text (AI recall). Across metrics, our findings corroborate the claim that detection models are brittle to out-of-distribution outputs---and that TTA methods improve performance relative to worst-case out-of-distribution outcomes for supervised learning.}

\label{fig:nongemini}
\end{figure}

\begin{figure}[ht]

  \centering
  \includegraphics[width=\linewidth]{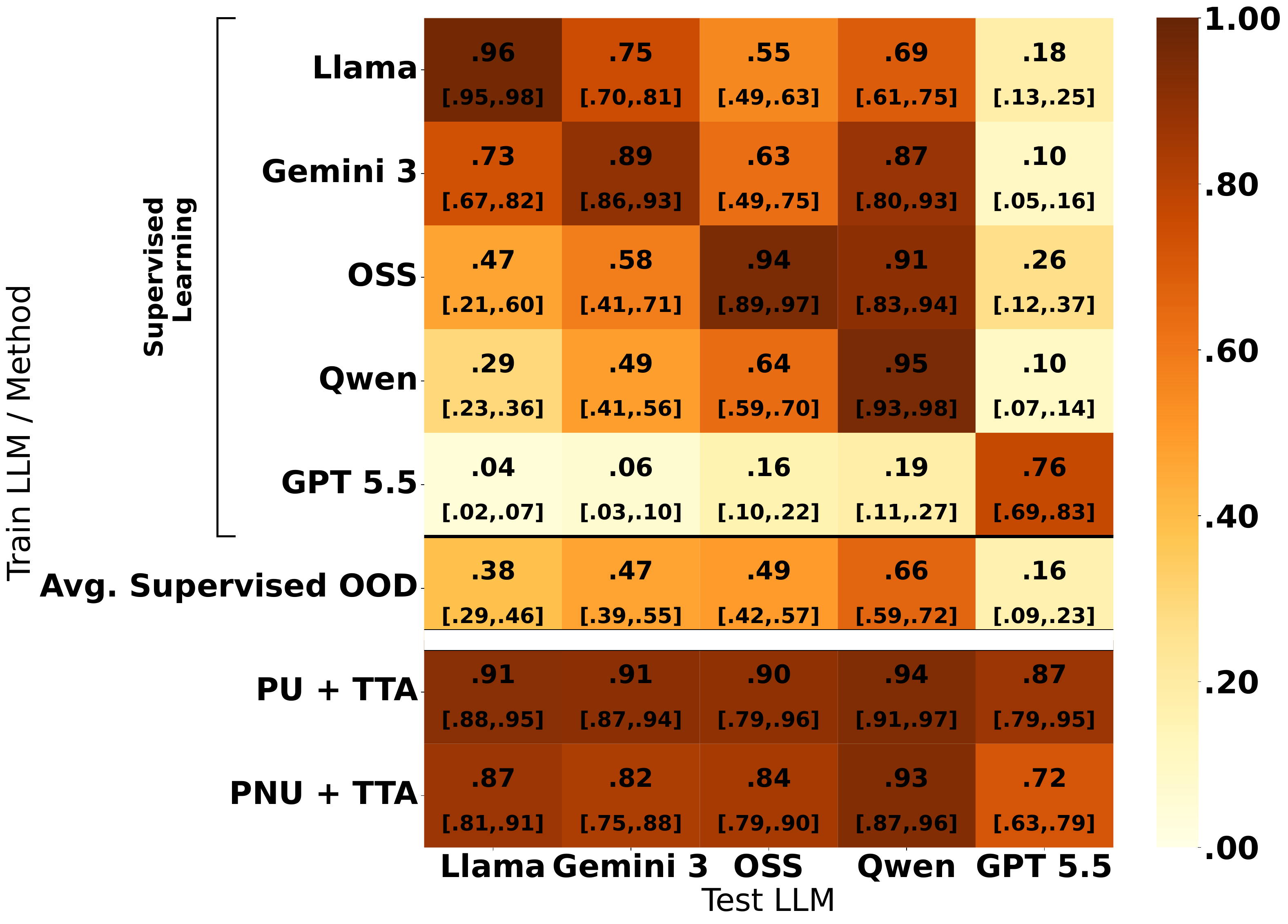}

\caption{The same figure as in \Cref{fig:llm-shift}, but with 95\% confidence intervals added. Our findings are reasonably robust to noise.}
\label{fig:nongeminici}
\end{figure}

\clearpage
\newpage

\subsection{Gemini heatmaps, out-of-distribution LLMs}\label{app:gemini}

\begin{figure}[ht]

  \centering
  \includegraphics[width=\linewidth]{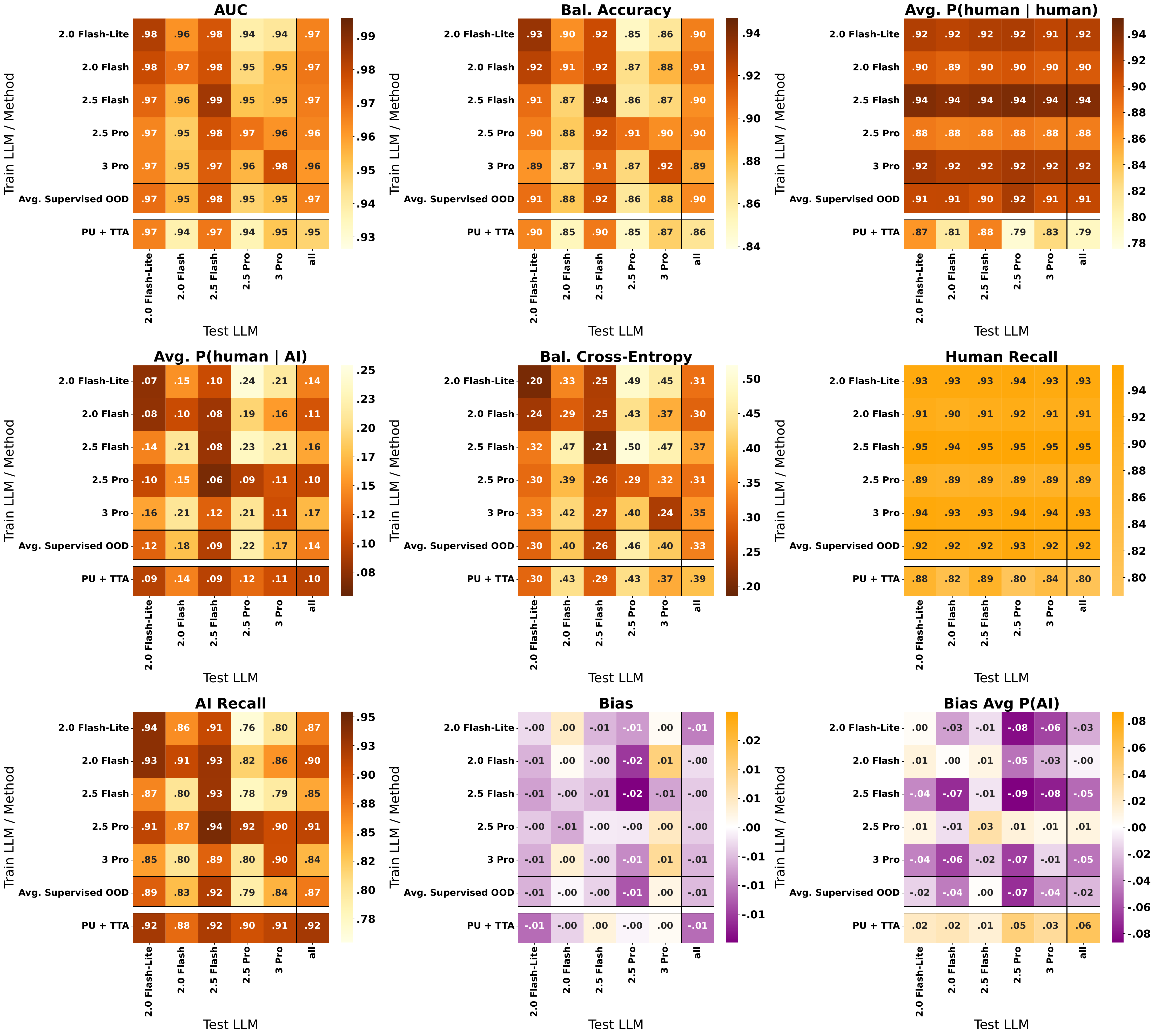}

\caption{We plot heatmaps of all collected metrics for models trained on five different Gemini models and evaluated on test sets composed of outputs of the same LLMs. We corroborate the finding that the effect of out-of-distribution \aigen text is weaker (although still present), when the training and evaluation LLMs are both from the Gemini model class.}

\label{fig:gemini}
\end{figure}

\begin{figure}[ht]

  \centering
  \includegraphics[width=\linewidth]{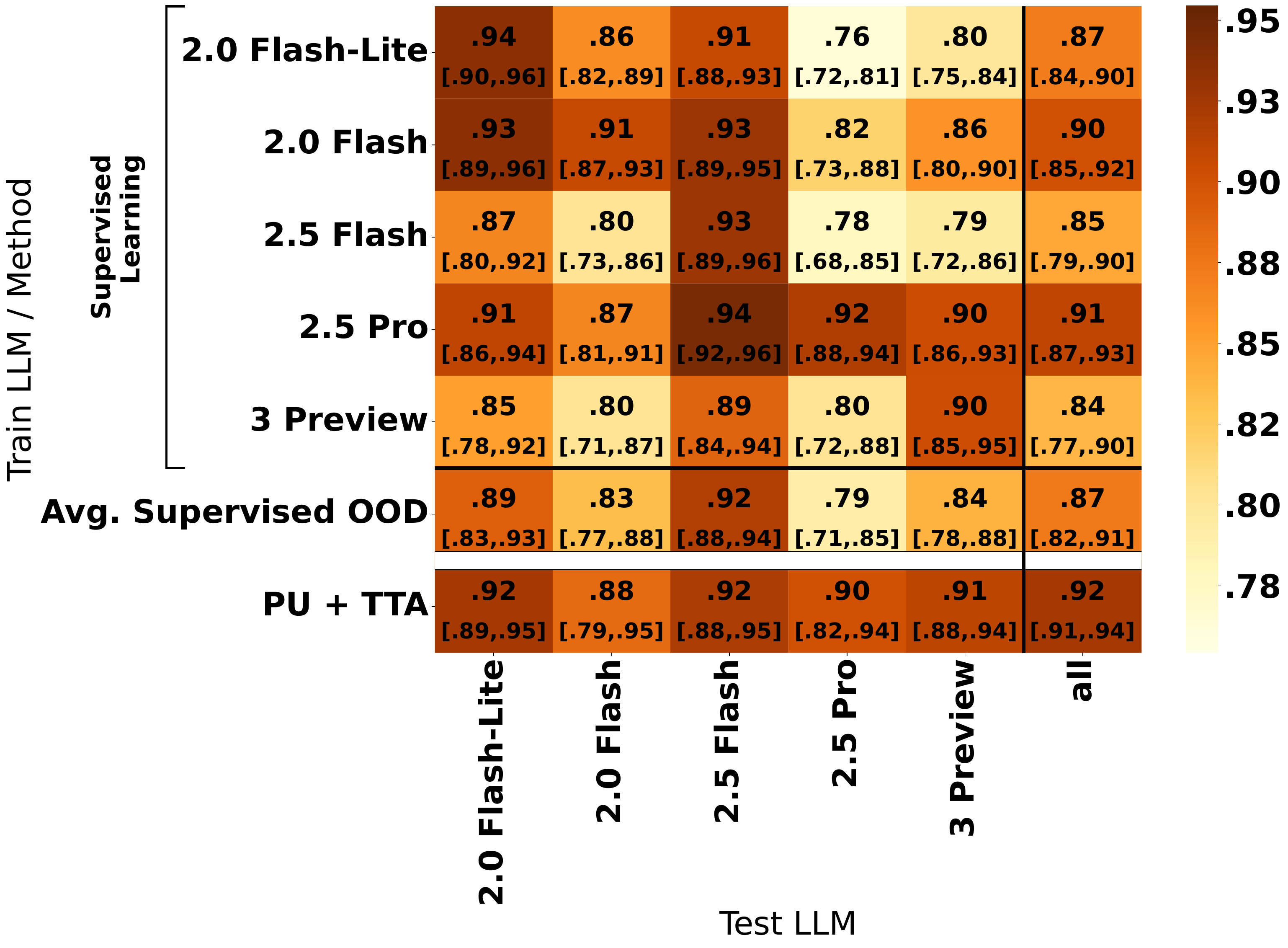}

\caption{The same figure as the AI recall subplot in \Cref{fig:gemini}, but with 95\% confidence intervals added. Our findings
are reasonably robust to noise.}

\label{fig:geminici}
\end{figure}

\clearpage
\newpage

\subsection{Supporting evidence for more correlated LLMs producing better OOD performance}
\label{app:lda}

\paragraph{Analysis motivation.} In \Cref{fig:llm-shift}, we find empirical evidence that certain pairwise distributions of LLMs (e.g., \{Gemini 3 Pro, Llama\}), when used to train a \detectionmodel, perform well when evaluated on each other's outputs; these LLMs also have relatively closely positioned clusters in the LDA projection in \Cref{fig:ldanongemini}. We now ask, does training on a \textit{similar} but not identical LLM correlate with increased out-of-distribution performance?

To explore the positive relationship between the similarity of the LLM outputs used at train-time and test-time (measured within the same sentence embedding space as was used to perform LDA) and test-time performance of a \detectionmodel trained on the train-time distribution,
we run additional analyses by fitting linear boundaries between pairwise combinations of clusters and recording the error rates (where each cluster contains the sentence embeddings for either one LLM or for humans). We are interested in the error rates themselves (intuitively, pairs of sources that have higher error rates are less easy to distinguish, i.e., are distributionally closer in the embedding space) as well as the relationship between error rates and the performance of the detection model when evaluated on an out-of-distribution LLM (a supervised \detectionmodel trained on one of the LLMs' outputs should perform well on the other LLM's outputs, the closer the distributions of the two LLMs' outputs are). We make the important caveat that analyzing specific pairs of LLMs is noisy, and this analysis should be viewed as suggestive (e.g., in 2 dimensions, the cluster of GPT 5.5 sentences overlaps with the cluster of Gemini 3 and Llama sentences, despite the fact that detectors trained on GPT 5.5 perform poorly on other LLMs and vice versa).

Additionally, \Cref{fig:pipeline_and_feature_space} (right) plots an LDA with 5,000 in-sample sentence embeddings per source, from human writing, non-adversarial AI mirrors (from gpt-oss-120b), and adversarially generated AI mirrors (specifically to defeat Pangram 3.3.2).

\paragraph{Setup.} In the experiments below, we fit logistic regression classifiers on pairwise combinations of each generative source, using all 5,000 available sentences per-source, for both the LLMs used in the main text figures as well as for the five Gemini LLMs (i.e., can a logistic classifier on the text embeddings distinguish between the two LLMs). Then, for each pair of LLMs, we plot error rates in-sample (\cref{fig:linearboundaryheat}). In \cref{fig:linearboundaryscat}, we show scatter plots where each point represents a pair of LLMs, and we plot the AI recall rates of the supervised detection models trained in \cref{fig:llm-shift} against the logistic regression error rate---averaging the AI recall rates over  the two detection models trained on each LLM in the pair.

\paragraph{Findings.} At a high level, we corroborate our claims in the main text (\Cref{sec:results_llm}). Qualitatively, from the analysis with DistilBERT detection models (\Cref{fig:llm-shift}), the LDA (\Cref{fig:ldanongemini}), and error rates from logistic regression (\Cref{fig:linearboundaryheat-1}) on the main text experiment, the distributions of human writing and LLM writing (for all LLMs) are clearly separable. Additionally, it seems that our results are internally consistent for Gemini 3 Pro: that detection models tended to classify its outputs well out-of-distribution, its cluster was spread into the clusters of all other LLM writing in the LDA, and the distributions of Gemini 3 Pro writing and other LLMs (Qwen, Llama) have the smallest separability. Finally, \Cref{fig:linearboundaryscat} shows that the error rates between pairs of LLMs (a proxy for distance between distributions) correlate strongly with out-of-distribution performance of a \detectionmodel trained on one LLM in the pair and evaluated on the second LLM (e.g., for the LLMs used in \Cref{sec:results_llm}, $r=0.755$ in \Cref{fig:linearboundaryscat}).

\begin{figure}[ht]
    \centering
  \includegraphics[width=.64\linewidth]{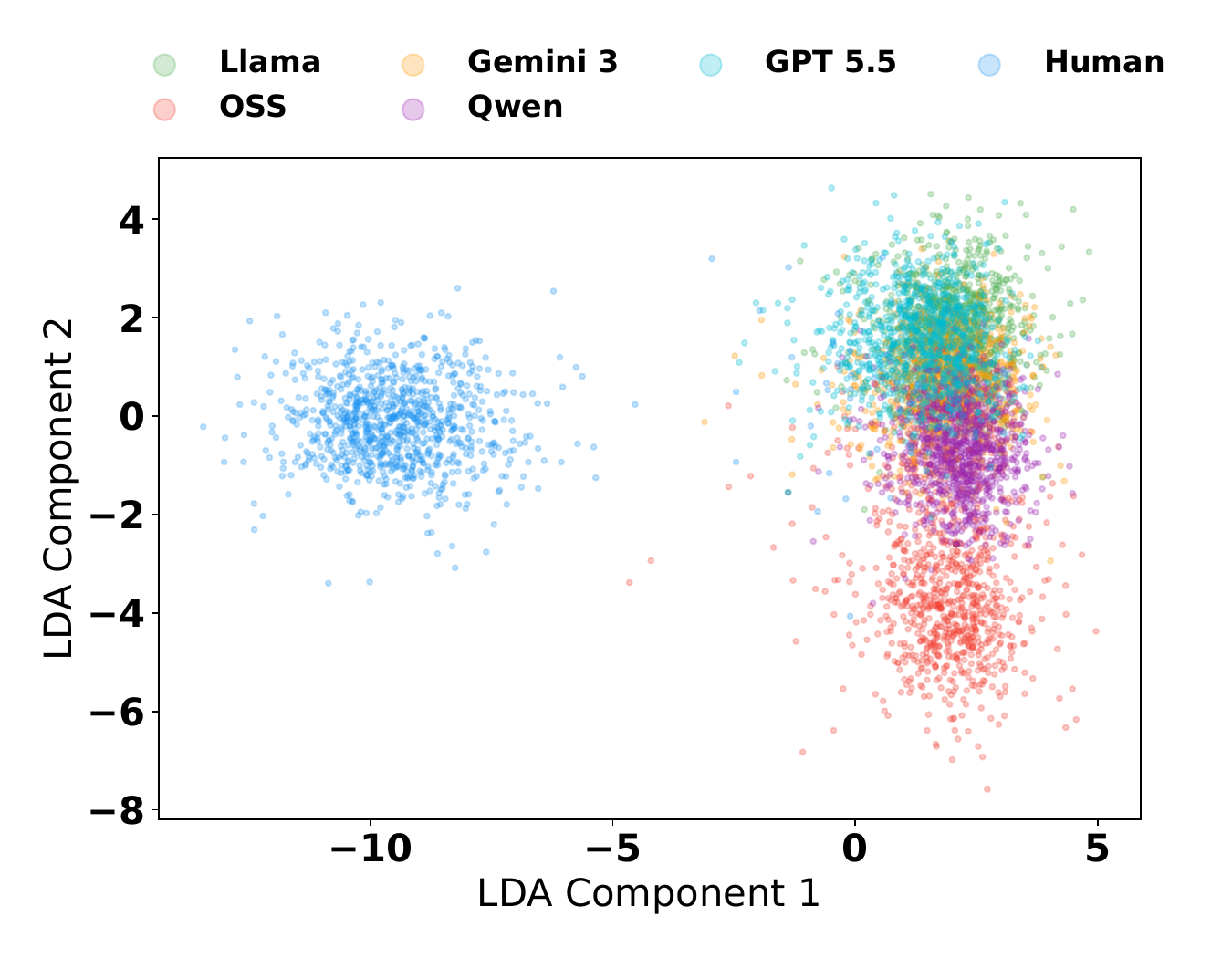}

\caption{Linear Discriminant Analysis with outputs of five different LLMs, each of which mirrored the same set of human writing. Each point represents an individual sentence; after generating synthetic mirrors, we split each abstract into sentences, make a sentence embedding using Gemini Embedding 2, and fit the LDA on 4,000 sentence embeddings per source and plot the remaining 1,000 sentence embeddings from each generative source (human, and each LLM). We find significant separation between human writing and all LLM writing, as well as separation between clusters of outputs between LLMs.}
\label{fig:ldanongemini}
\end{figure}

\begin{figure}[ht]
    \centering
  \includegraphics[width=.64\linewidth]{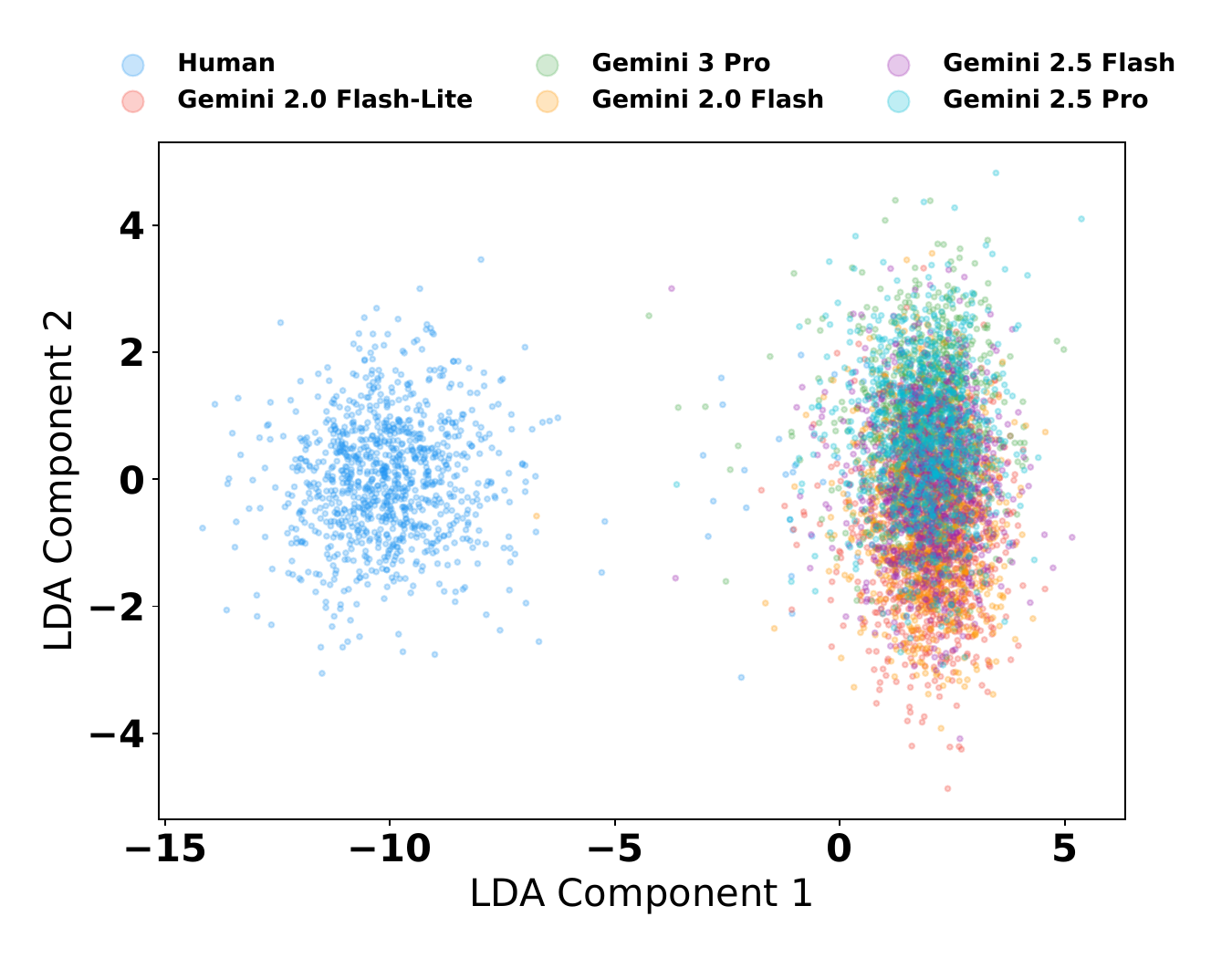}

\caption{Linear Discriminant Analysis with outputs of five different Gemini LLMs, each of which mirrored the same set of human writing. Each point represents an individual sentence; after generating synthetic mirrors, we split each abstract into sentences, make a sentence embedding using Gemini Embedding 2, and fit the LDA on 4,000 sentence embeddings per source and plot the remaining 1,000 sentence embeddings from each generative source (human, and each LLM). While each cluster is closer than when the LLMs are of different model classes, we still find significant separation between human writing and all LLM writing, as well as some separation between clusters of outputs between LLMs. Intuitively, the overlap in the Gemini-produced text distributions explains high out-of-distribution performance.}
\label{fig:ldagemini}
\end{figure}

\begin{figure}[tbh]
\begin{subfigure}{.49\linewidth}
  \centering
  \includegraphics[width=\linewidth]{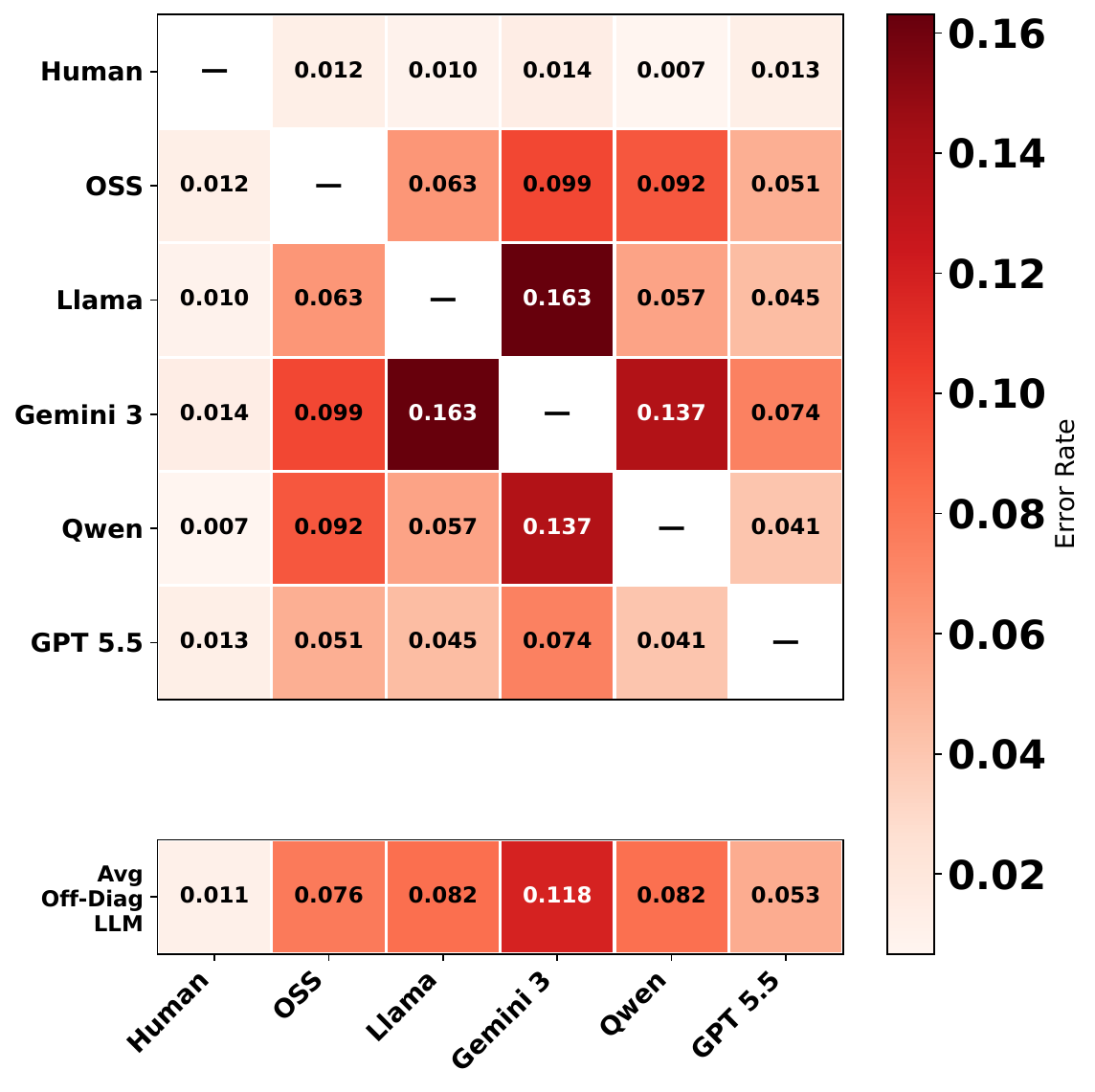}
  \caption{}
  \label{fig:linearboundaryheat-1}
\end{subfigure}\begin{subfigure}{.51\linewidth}
  \centering
  \includegraphics[width=\linewidth]{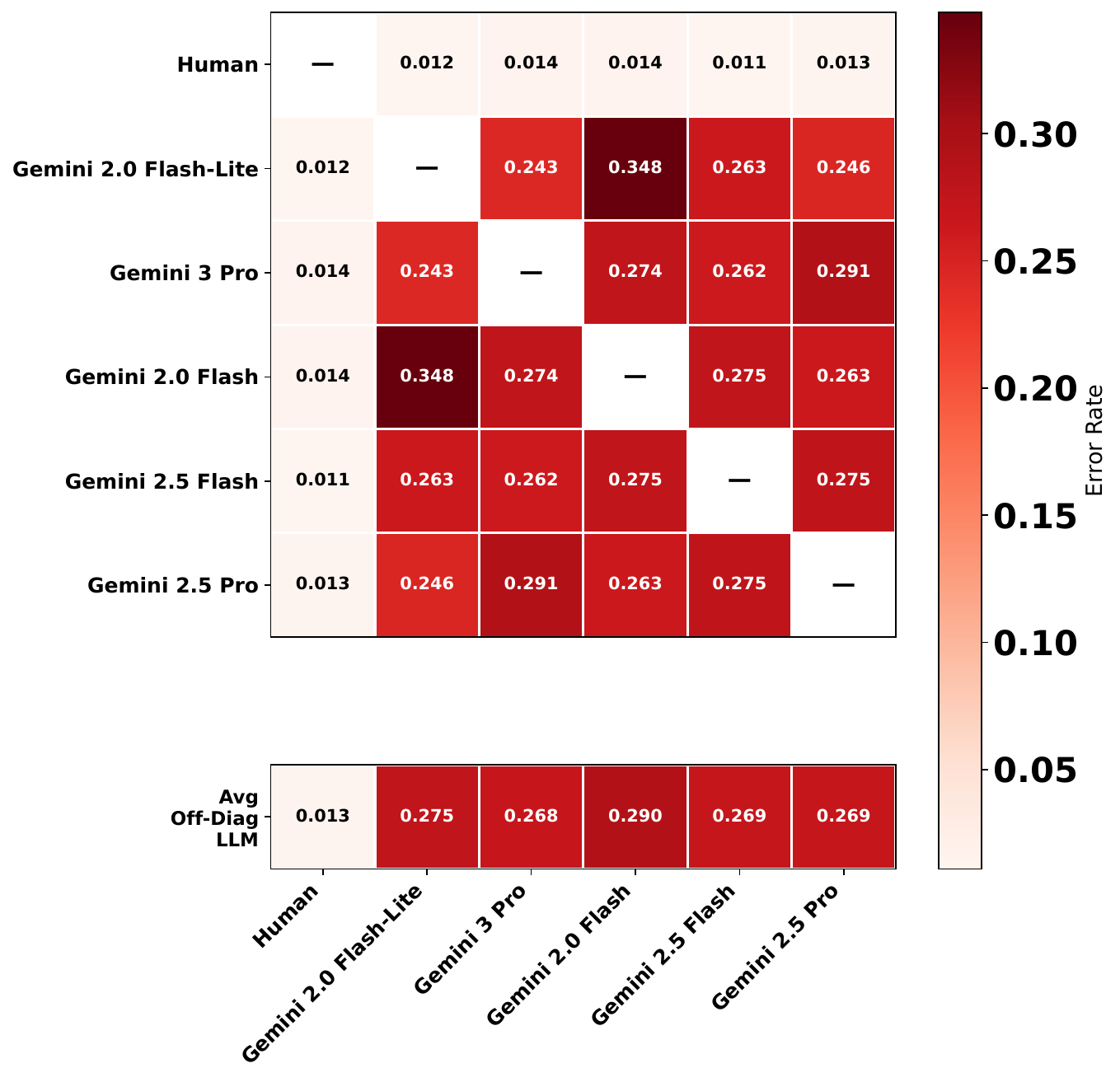}
  \caption{}
  \label{fig:linearboundaryheat-2}
\end{subfigure}

\caption{Error rates of logistic regressions trained and evaluated in-sample on abstracts from two sources. We corroborate the finding in \Cref{fig:llm-shift}, as Gemini 3 Pro Preview's outputs overlap with those of Qwen and Llama 3.3 70B Instruct, translating to better out-of-distribution performance for a \detectionmodel trained on one of these LLMs and evaluated on the other.}

\label{fig:linearboundaryheat}
\end{figure}

\begin{figure}[tbh]
\begin{subfigure}{.41\linewidth}
  \centering
  \includegraphics[width=\linewidth]{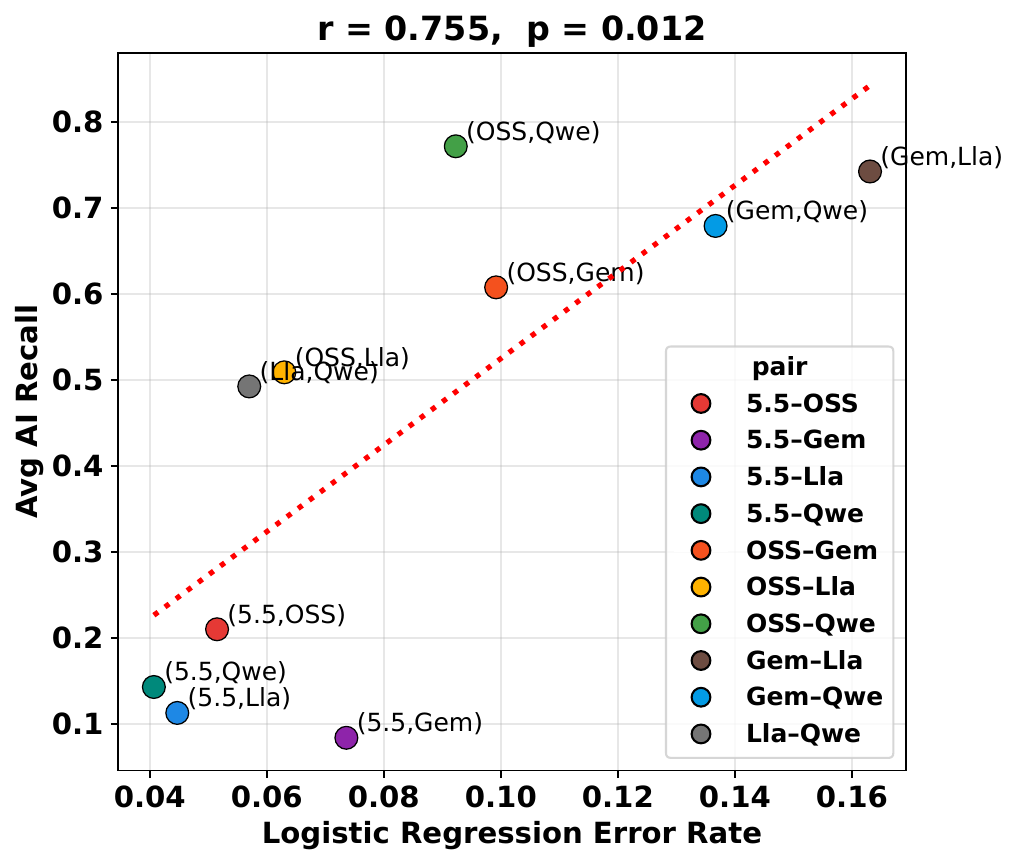}
  \caption{}
  \label{fig:linearboundaryscat-1}
\end{subfigure}\begin{subfigure}{.59\linewidth}
  \centering
  \includegraphics[width=\linewidth]{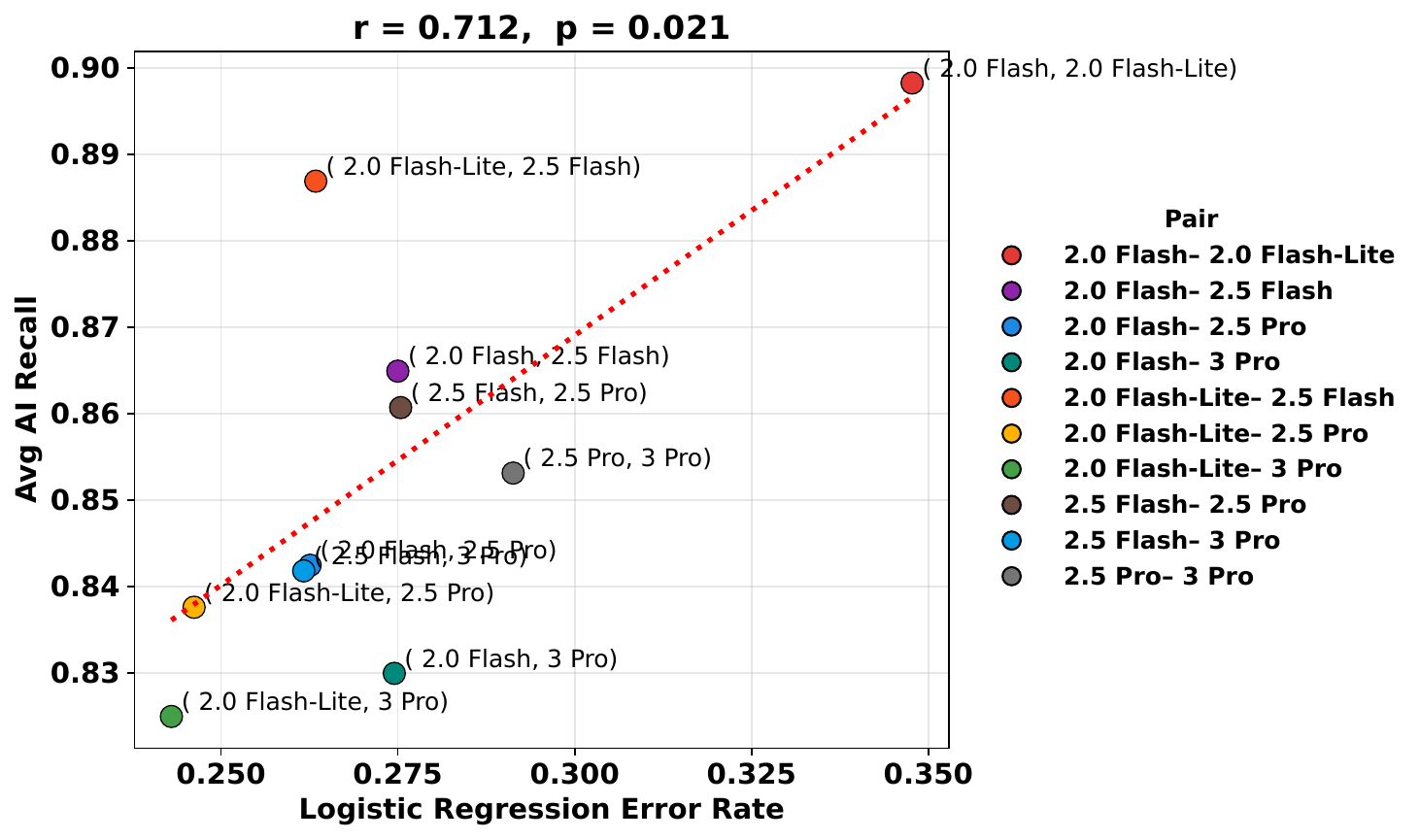}
  \caption{}
  \label{fig:linearboundaryscat-2}
\end{subfigure}

\caption{Scatter plots, for each pairwise combination of LLMs, for the error rates as in \Cref{fig:linearboundaryheat} against the averaged AI recall scores for the detection models trained on each LLM's outputs per-pair and evaluated on the second LLM's outputs. We find a strong positive correlation: as the error rate increases, the distributions of the two LLMs' outputs can be interpreted as becoming closer, which would allow a \detectionmodel trained on one of the LLMs' outputs to perform better on the second LLM's out-of-distribution outputs.}

\label{fig:linearboundaryscat}
\end{figure}

\clearpage
\newpage

\subsection{Pangram analysis post-GPT 5.4 release}

\begin{figure}[ht]
    \centering
  \includegraphics[width=.7\linewidth]{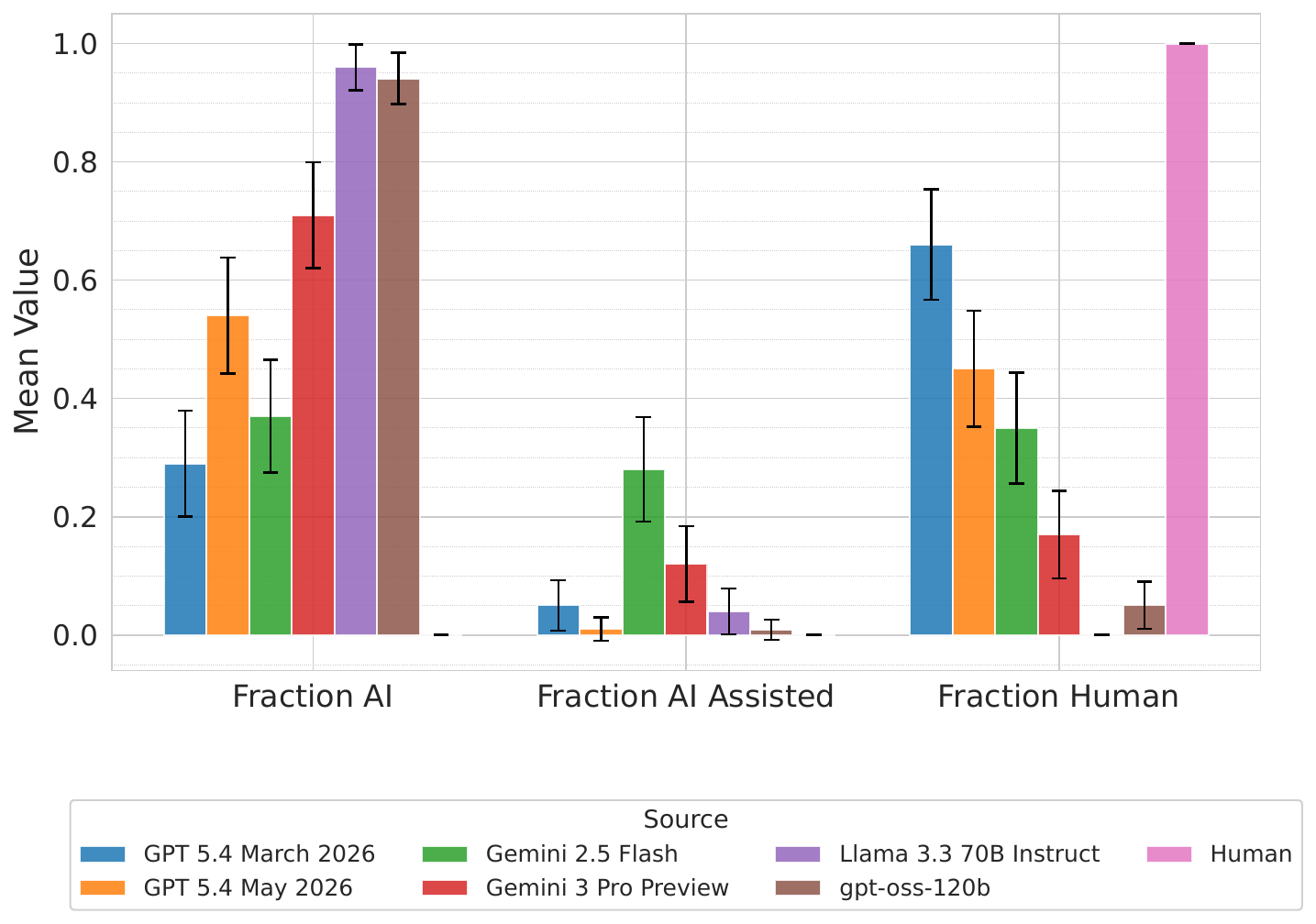}

\caption{We find that Pangram performs significantly better on some LLMs than others. Notably, shortly after the release of GPT 5.4 (likely before Pangram retrained its model on GPT 5.4 outputs; ``GPT 5.4 March 2026''), we observe that non-adversarially generated abstracts written by GPT 5.4 fool Pangram 3.2 (are predicted as fully human-written) 66\% of the time; notably, Pangram also perfectly classifies the human abstracts. After the release of Pangram 3.3.2 in May 2026, we note that Pangram's accuracy on GPT 5.4 abstracts significantly increases (``GPT 5.4 May 2026''). We mirror the same 100 human abstracts from arXiv using each model (written in 2010), pass the resulting 600 abstracts to Pangram's API, and report average statistics for each of the returned variables with 95\% error bars.}
\label{fig:pangram}
\end{figure}

\begin{figure}[ht]
    \centering
  \includegraphics[width=.7\linewidth]{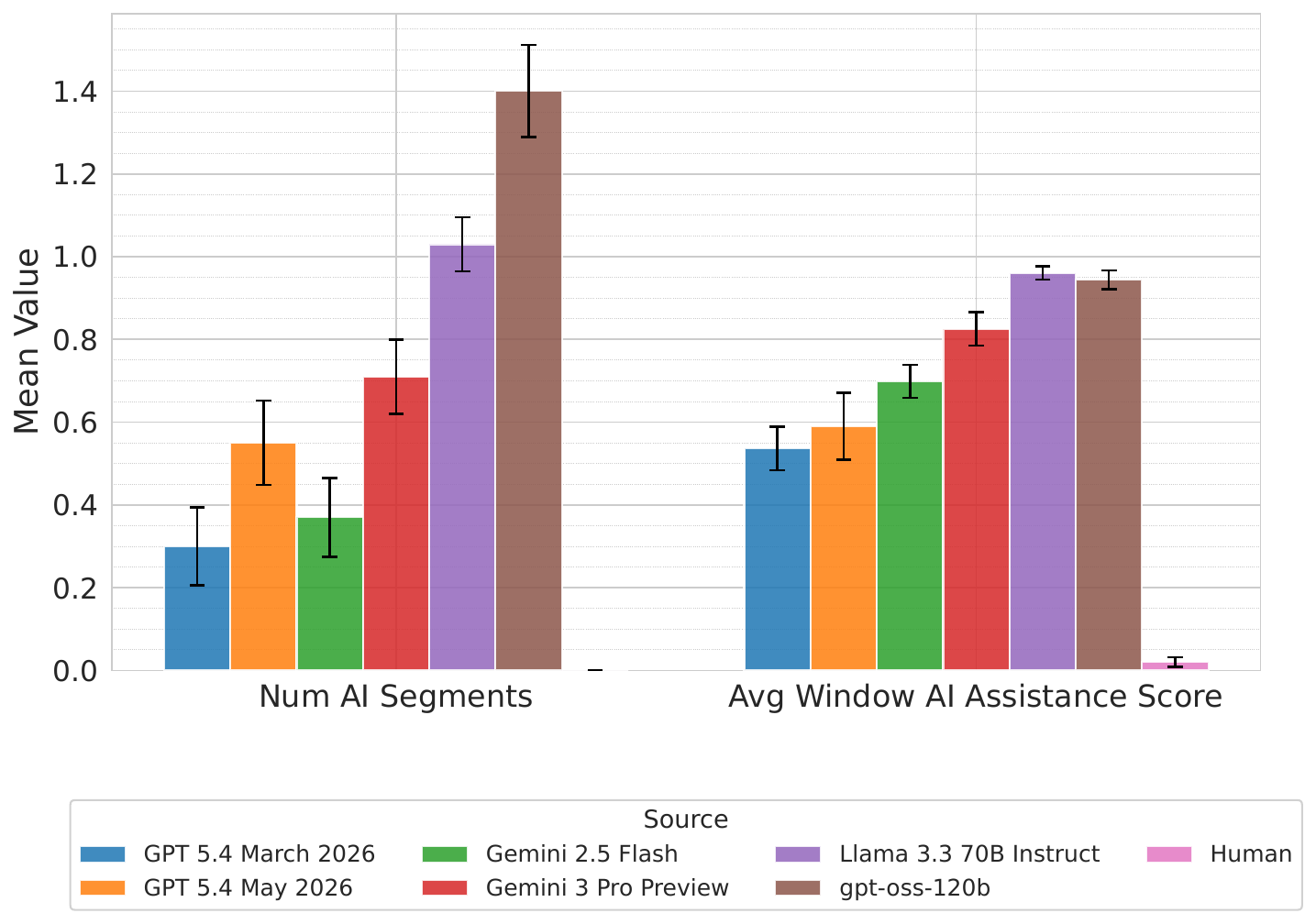}

\caption{Grouped bar plots as in \Cref{fig:pangram}, but plotting the number of segments classified as AI-written by Pangram, along with the continuous score reflecting P(LLM) from Pangram (averaged across windows). As in \Cref{fig:pangram}, pre-release of Pangram 3.3.2, GPT 5.4 has fewer predicted AI segments and lower continuous predictions of being \aigen than other LLMs (although Pangram 3.3.2 has significantly higher performance, presumably after adding labeled, out-of-distribution data to its training set).}
\label{fig:pangram2}
\end{figure}

\clearpage
\newpage

\subsection{Temporal shift in human writing}\label{app:cdf}

\begin{figure}[ht]

  \centering
  \includegraphics[width=\linewidth]{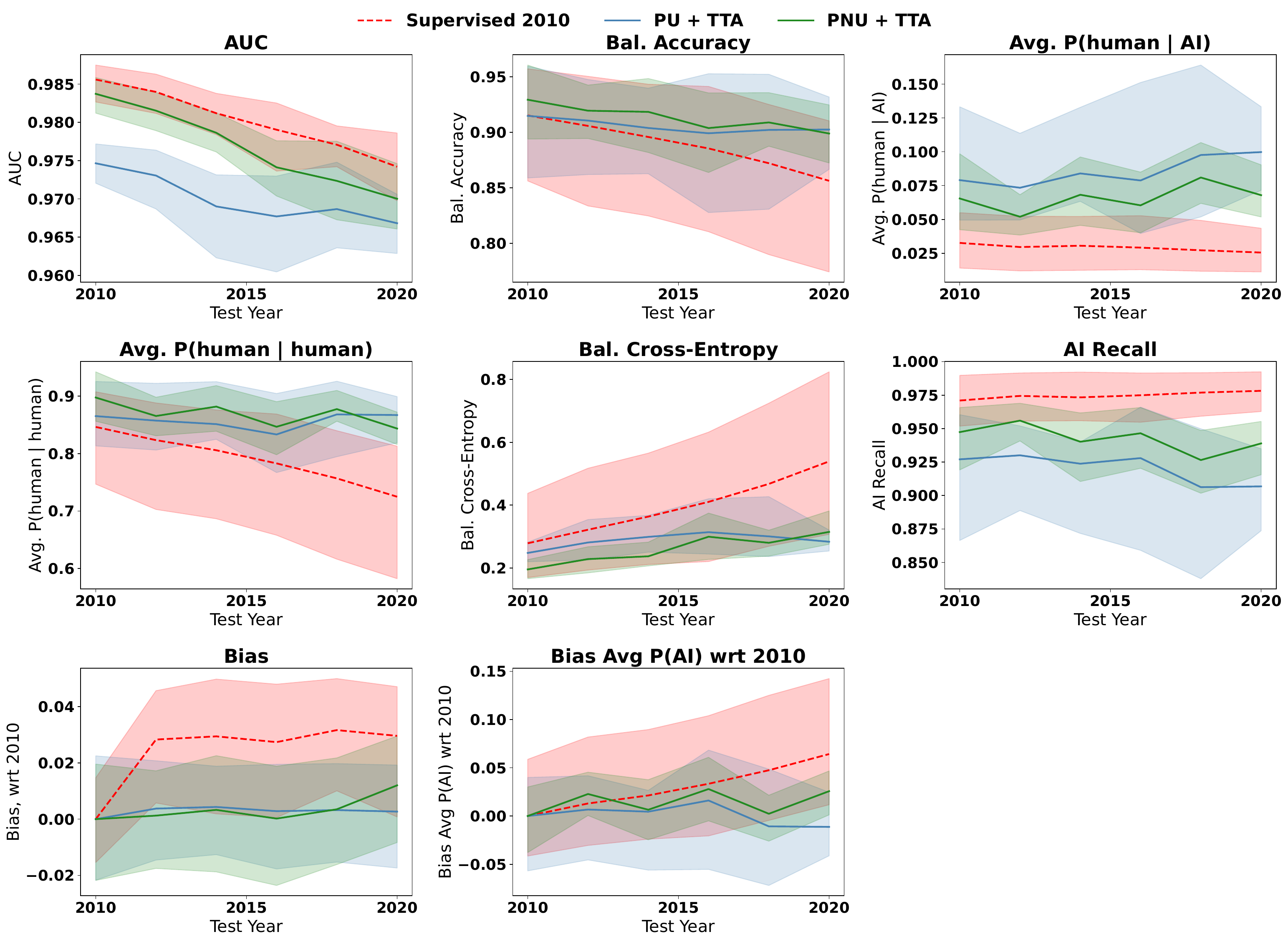}

\caption{Demonstrating that \Cref{fig:droptemporal-1} is robust on multiple metrics. While PU + TTA performs poorly on \aigen text compared to supervised methods, we see that, due to temporal shift in human writing from 2010 to 2020, supervised methods gradually predict worse on human writing (lower AUC, balanced accuracy, and continuous predictions on human writing; higher balanced cross-entropy and bias on prevalence estimation) over time, while PU + TTA has stable performance on human writing.}

\label{fig:extraline3}
\end{figure}

\begin{figure}[ht]

  \centering
  \includegraphics[width=.7\linewidth]{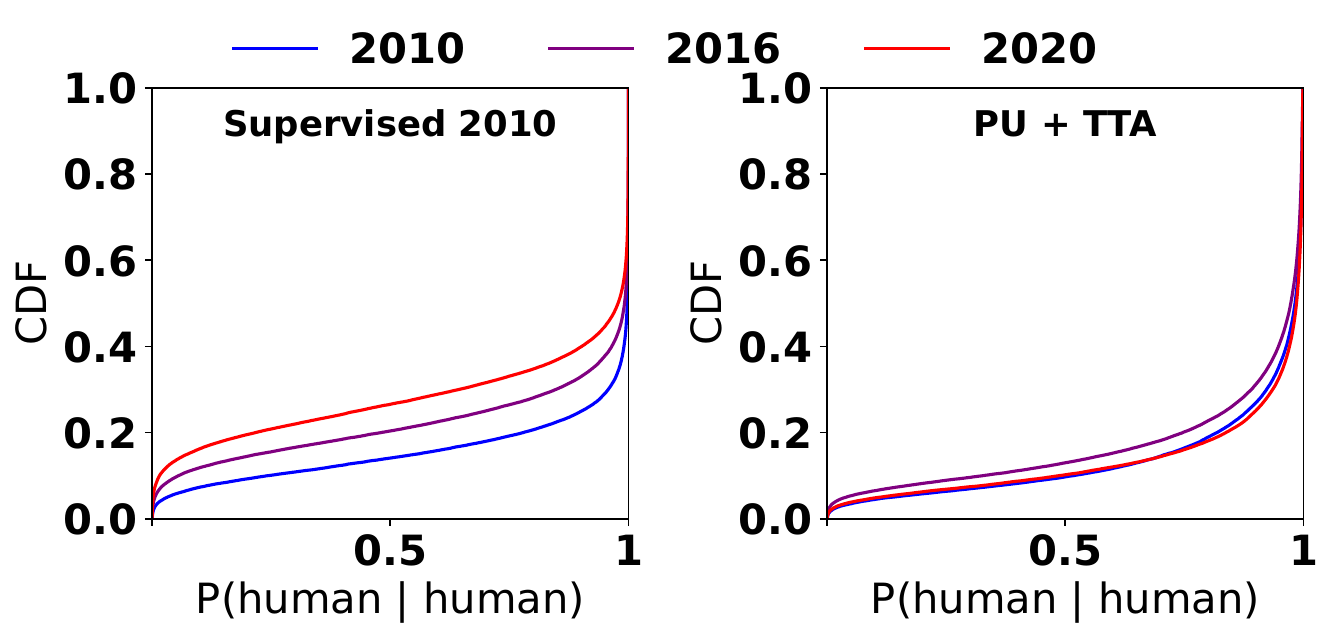}

\caption{We plot CDFs of the distribution of predictions on held-out human writing from 2010, 2016, and 2020, for a supervised \aigen text classifier trained on 2010 data (left) and PU + TTA (right). Note that the predictions by the supervised model shift over time, supporting our hypothesis of noticeable natural distribution shift in human writing between 2010 and 2020.}

\label{fig:cdf1}
\end{figure}

\begin{figure}[ht]

  \centering
  \includegraphics[width=.7\linewidth]{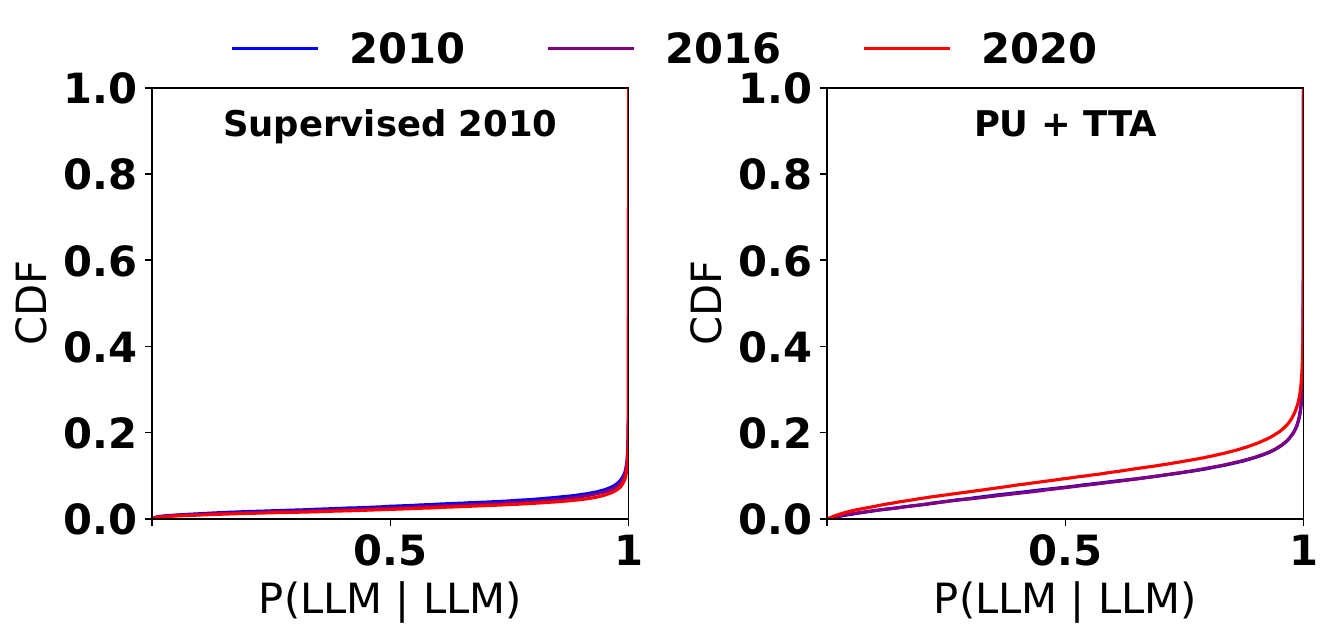}

\caption{We plot CDFs of the distribution of predictions on held-out \aigen writing from 2010, 2016, and 2020, for a supervised \aigen text classifier trained on 2010 data (left) and PU + TTA (right). Note that both distributions are fairly stable over time.}

\label{fig:cdf2}
\end{figure}

\begin{figure}[t]

  \centering
  \includegraphics[width=.7\linewidth]{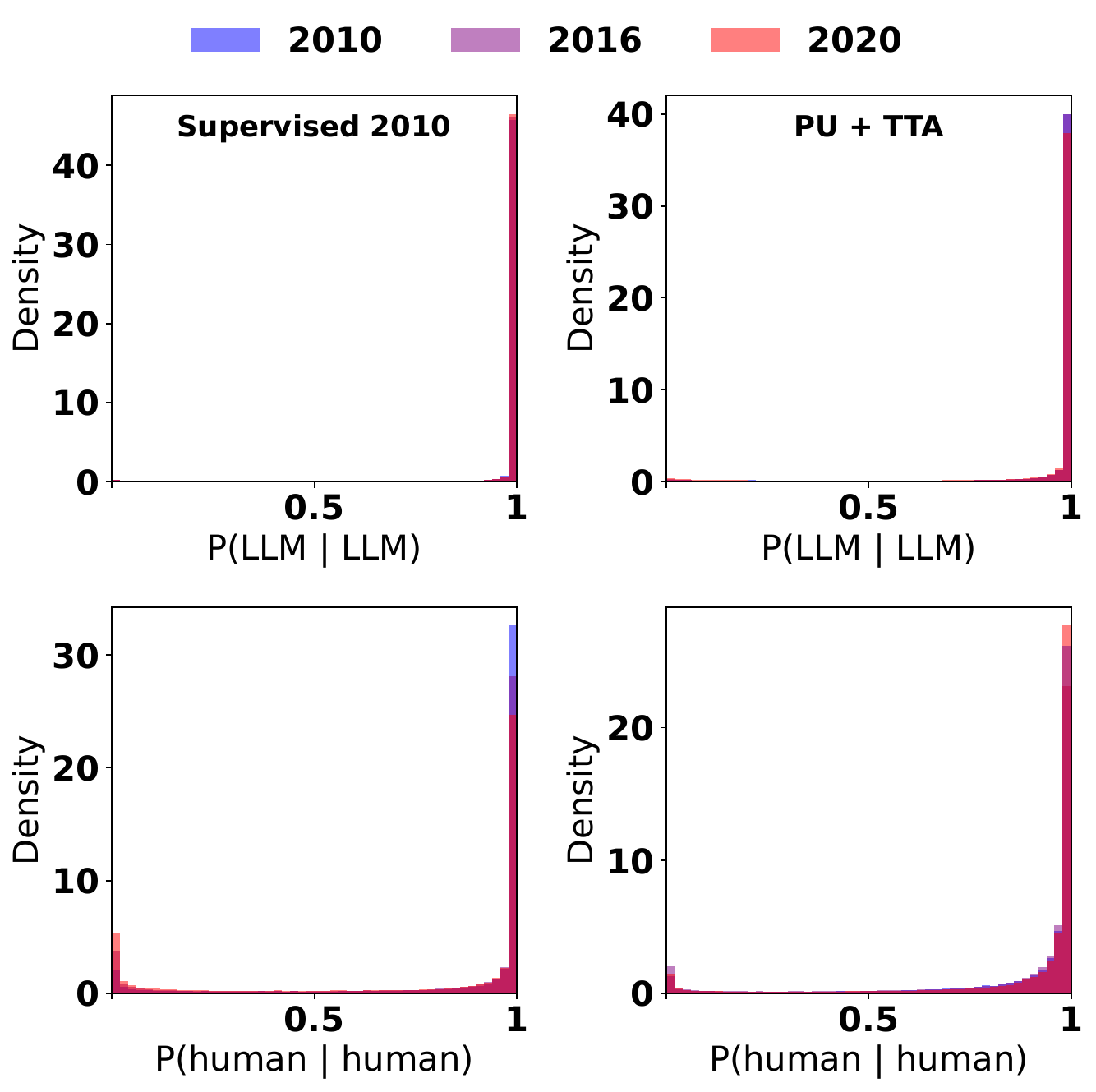}

\caption{We plot PDFs (histograms) of the distribution of predictions on held-out human writing (bottom) and \aigen writing (top) from 2010, 2016, and 2020, for a supervised \aigen text classifier trained on 2010 data (left) and PU + TTA (right).}

\label{fig:hist2}
\end{figure}

\clearpage

\section{RAID benchmarks}
\label{sec:raid}

We evaluate on the RAID dataset \citep{dugan-etal-2024-raid}, which contains over 6 million \aigen and human-written texts. For \aigen texts, RAID records features that can be used to simulate distribution shifts. Our experiments involve training supervised \aigen detection models on known human writing and train-time \aigen writing, and training PU + TTA models on known human writing and an unlabeled mixture of human writing and test-time \aigen writing. We evaluate models on the test-time distribution of \aigen text, reporting AI recall rate. At a high level, results replicate those of the main text. We describe results below and in Table \ref{tab:cross_tnr} and Table \ref{tab:tnr}. Under distribution shifts where supervised methods are unable to adapt at test-time, PU + TTA offers significant advantages. On RAID, PU + TTA is most beneficial under domain, generator, and repetition-penalty shifts; supervised detectors are already strong on many simple adversarial perturbations.

\paragraph{Experimental setup.}

Given the learning method (supervised; PU + TTA), the relative sizes of the positive (human) and negative (\aigen) texts from the source and target distributions are fixed. For supervised learning, we use a 4:1 train-to-validation split of labeled positives on a random subset of these texts (4,456 training positives; 1,114 validation positives), paired with 20,000 labeled source-distribution LLM negatives for training and 5,000 for validation. For PU + TTA, we additionally require unlabeled positive slices for both training and validation, so the human pool is divided in a 4:1:4:2 ratio across training labeled ($3{,}646$), validation labeled ($911$), training unlabeled ($3{,}646$), and validation unlabeled ($1{,}825$) positives; unlabeled AI negatives are drawn from the target distribution and mixed with the unlabeled positives equally (3,646 training and 1,825 validation unlabeled AI examples).
Supervised learning is never exposed to target-distribution AI text during training, while PU + TTA's unlabeled pool mixes source-distribution human text with target-distribution AI. For adversarial attack conditions, target-distribution AI refers to LLM output with the specified attack applied; for distribution shift conditions (model, domain, decoding, repetition penalty), it refers to unattacked LLM output filtered to the target value of the shift column. 
Results are averaged over 5 random seeds; given the same seed, PU's labeled training positives are a strict subset of supervised learning's, both drawn from the same human pool.

\paragraph{Adversarial attacks.} For \aigen text, RAID contains perturbed versions of that text, from alternative spellings of specific words (``alternative\_spelling'') to rewriting text with a paraphrasing language model (``paraphrase''); see \citep{dugan-etal-2024-raid} for an exhaustive list of descriptions. We additionally simulate an equal mixture of all attacks included in RAID (``all''). We treat the train-time distribution of \aigen writing as those texts without any perturbation, whereas test-time \aigen text consists of one of the attacks. Their adversarial attacks are static rather than continual (i.e., applying a fixed transformation such as removing whitespaces, rather than adapting the rewriting strategy to the release of a specific \detectionmodel).
While supervised models are largely robust to the attacks in RAID, their performance falls significantly on the ``homoglyph'' attack; our PU + TTA models outperform the supervised models (gain of 0.40 in AI recall). For all subsequent experiments, we preprocess the RAID dataset to remove all \aigen text containing adversarial attacks.

\paragraph{Decoding.} AI text is produced using one of two decoding strategies: greedy decoding, where at each step the token with highest probability is outputted by the LLM (we treat this as the train-time distribution), and sampling, where at each step the outputted token is sampled from the softmax distribution predicted by the LLM over all tokens (we treat this as the test-time distribution). 
Supervised models are fairly robust to shifts in decoding strategy (i.e., trained on greedily-decoded \aigen text but evaluated on \aigen text produced via sampling), and slightly outperform PU + TTA trained on the test-time distribution (loss of 0.02 in AI recall).

\paragraph{Domain.} RAID assembles human writing and \aigen writing from multiple existing domains (see \citep{dugan-etal-2024-raid} for an exhaustive list of domains and the original data source). We randomly sample 5 pairs of train-time (``Source'') and test-time (``Target'') domains. Supervised models trained on one domain perform poorly on held-out domains; furthermore, PU + TTA models trained on unlabeled samples from the test distribution enjoy significant gains relative to supervised models evaluated out-of-distribution. {While, due to label imbalance (more data existing per-domain for \aigen text than human text), both sets of models have high AI recall rates, PU + TTA performs significantly better on both human recall and AUC metrics, indicating that TTA allows the \detectionmodel to better separate the distributions of \aigen and human text than supervised learning.}

\paragraph{Generator LLM.} RAID contains \aigen text from 11 LLMs; we treat each LLM's outputs as one domain, and keep 4 LLMs for which over 30,000 texts are available. We train supervised and PU + TTA models with \aigen text from each of the LLMs. Each supervised model is evaluated on the outputs of each LLM, while the PU + TTA models are simply evaluated on the same test-time distribution on which they were trained. We find that PU + TTA models consistently outperform supervised models, when the train-time and test-time distributions of \aigen text differ.

\paragraph{Repetition penalty.} AI text can optionally be generated using a repetition penalty (i.e., downweighting tokens that have recently been outputted by the LLM in the sequence). We treat \aigen text without the repetition penalty as the train-time distribution, and \aigen text with the penalty as the test-time distribution. We find significant performance drops at test-time from supervised models, and improvement when using PU + TTA models (gain of 0.18 in AI recall).

\begin{table}[htbp]
\centering
\small
\caption{RAID benchmark results for the AUC metric, for different adversarial attacks (training and evaluation \aigen data balanced between domains, LLMs, decoding strategies, and repetition penalty). Results for LLM outputs with different adversarial attacks applied, with upper and lower bounds attached.}
\label{tab:cross_auc}
\begin{tabular}{p{3.3cm}p{3.3cm}p{3.3cm}c}
\toprule
\textbf{Eval Attack} & \textbf{Supervised, AUC} & \textbf{PU + TTA, AUC} & \textbf{PU + TTA Gain} \\
\midrule
all & 0.97 [0.95, 0.99] & 0.96 [0.95, 0.97] & -0.01 \\
alternative\_spelling & 0.99 [0.98, 0.99] & 0.97 [0.96, 0.98] & -0.02 \\
article\_deletion & 0.98 [0.97, 0.99] & 0.99 [0.98, 0.99] & +0.01 \\
homoglyph & 0.79 [0.65, 0.95] & 1.00 [1.00, 1.00] & +0.21 \\
insert\_paragraphs & 0.99 [0.98, 0.99] & 0.97 [0.96, 0.98] & -0.02 \\
number & 0.99 [0.98, 0.99] & 0.97 [0.96, 0.98] & -0.02 \\
paraphrase & 0.97 [0.96, 0.98] & 0.98 [0.97, 0.99] & +0.01 \\
perplexity\_misspelling & 0.98 [0.97, 0.99] & 0.97 [0.96, 0.98] & -0.01 \\
synonym & 0.99 [0.98, 0.99] & 0.97 [0.96, 0.98] & -0.02 \\
upper\_lower & 0.99 [0.99, 0.99] & 0.97 [0.96, 0.97] & -0.02 \\
whitespace & 0.99 [0.98, 0.99] & 0.97 [0.96, 0.97] & -0.02 \\
zero\_width\_space & 0.99 [0.98, 0.99] & 0.97 [0.96, 0.97] & -0.02 \\
none & 0.99 [0.99, 0.99] & \textemdash & \textemdash \\
\bottomrule
\end{tabular}
\end{table}

\begin{table}[htbp]
\centering
\small
\caption{RAID benchmark results for AI recall, for different adversarial attacks (training and evaluation \aigen data balanced between domains, LLMs, decoding strategies, and repetition penalty). Results for LLM outputs with different adversarial attacks applied, with upper and lower bounds attached.}
\label{tab:cross_tnr}
\begin{tabular}{p{3.3cm}p{3.3cm}p{3.3cm}c}
\toprule
\textbf{Eval Attack} & \textbf{Supervised, AI Recall} & \textbf{PU + TTA, AI Recall} & \textbf{PU + TTA Gain} \\
\midrule
all & 0.95 [0.90, 0.99] & 0.93 [0.85, 0.97] & -0.02 \\
alternative\_spelling & 0.99 [0.98, 1.00] & 0.96 [0.92, 0.99] & -0.03 \\
article\_deletion & 0.98 [0.97, 0.99] & 0.97 [0.93, 0.99] & -0.01 \\
homoglyph & 0.60 [0.10, 1.00] & 1.00 [1.00, 1.00] & +0.40 \\
insert\_paragraphs & 0.99 [0.98, 1.00] & 0.94 [0.91, 0.97] & -0.05 \\
number & 0.99 [0.98, 0.99] & 0.95 [0.92, 0.96] & -0.04 \\
paraphrase & 0.98 [0.95, 0.99] & 0.98 [0.97, 0.99] & +0.01 \\
perplexity\_misspelling & 0.98 [0.96, 0.99] & 0.95 [0.91, 0.98] & -0.03 \\
synonym & 0.99 [0.98, 0.99] & 0.95 [0.91, 0.98] & -0.04 \\
upper\_lower & 0.99 [0.98, 1.00] & 0.96 [0.93, 0.98] & -0.03 \\
whitespace & 0.99 [0.99, 1.00] & 0.96 [0.93, 0.98] & -0.04 \\
zero\_width\_space & 0.99 [0.98, 1.00] & 0.93 [0.83, 0.98] & -0.07 \\
none & 0.99 [0.99, 1.00] & \textemdash & \textemdash \\
\bottomrule
\end{tabular}
\end{table}

\begin{table}[htbp]
\centering
\small
\caption{RAID benchmark results for human recall, for different adversarial attacks (training and evaluation \aigen data balanced between domains, LLMs, decoding strategies, and repetition penalty).}
\label{tab:cross_tpr}
\begin{tabular}{p{3.2cm}p{3.5cm}p{3.4cm}c}
\toprule
\textbf{Eval Attack} & \textbf{Supervised, Human Recall} & \textbf{PU + TTA, Human Recall} & \textbf{PU + TTA Gain} \\
\midrule
all & 0.75 [0.67, 0.84] & 0.82 [0.72, 0.95] & +0.06 \\
alternative\_spelling & 0.75 [0.67, 0.84] & 0.74 [0.55, 0.92] & -0.01 \\
article\_deletion & 0.75 [0.67, 0.84] & 0.92 [0.85, 0.97] & +0.17 \\
homoglyph & 0.75 [0.67, 0.84] & 1.00 [0.99, 1.00] & +0.24 \\
insert\_paragraphs & 0.75 [0.67, 0.84] & 0.84 [0.74, 0.91] & +0.09 \\
number & 0.75 [0.67, 0.84] & 0.84 [0.78, 0.90] & +0.09 \\
paraphrase & 0.75 [0.67, 0.84] & 0.77 [0.68, 0.86] & +0.01 \\
perplexity\_misspelling & 0.75 [0.67, 0.84] & 0.85 [0.78, 0.89] & +0.09 \\
synonym & 0.75 [0.67, 0.84] & 0.80 [0.68, 0.90] & +0.04 \\
upper\_lower & 0.75 [0.67, 0.84] & 0.76 [0.70, 0.83] & +0.00 \\
whitespace & 0.75 [0.67, 0.84] & 0.79 [0.67, 0.88] & +0.03 \\
zero\_width\_space & 0.75 [0.67, 0.84] & 0.83 [0.63, 0.98] & +0.07 \\
none & 0.75 [0.67, 0.84] & \textemdash & \textemdash \\
\bottomrule
\end{tabular}
\end{table}

\begin{table}[htbp]
\centering
\small
\caption{RAID benchmark results for AUC, for different distribution shifts (training and evaluation \aigen data filtered to remove adversarial attacks). Results for out-of-distribution LLM outputs, with upper and lower bounds attached.}
\label{tab:auc}
\begin{tabular}{p{1.6cm}p{1.4cm}p{1.4cm}p{2.8cm}p{2.2cm}c}
\toprule
\textbf{Shift} & \textbf{Source} & \textbf{Target} & \textbf{Supervised} & \textbf{PU + TTA} & \textbf{PU + TTA Gain} \\
\midrule
decoding & greedy & sampling & 0.97 [0.95, 0.97] & 0.96 [0.95, 0.96] & -0.01 \\
\midrule
domain & abstracts & news & 0.71 [0.65, 0.77] & 0.86 [0.85, 0.88] & +0.15 \\
domain & books & abstracts & 0.80 [0.74, 0.87] & 0.93 [0.91, 0.94] & +0.13 \\
domain & news & recipes & 0.65 [0.43, 0.81] & 0.95 [0.93, 0.96] & +0.30 \\
domain & reddit & reviews & 0.65 [0.56, 0.72] & 0.97 [0.96, 0.98] & +0.33 \\
domain & reviews & reddit & 0.61 [0.56, 0.73] & 0.86 [0.82, 0.90] & +0.24 \\
\midrule
model & gpt2 & gpt2 & 1.00 [1.00, 1.00] & 0.99 [0.98, 0.99] & -0.01 \\
model & gpt2 & llama-chat & 0.90 [0.87, 0.92] & 1.00 [1.00, 1.00] & +0.10 \\
model & gpt2 & mpt & 0.99 [0.97, 1.00] & 0.99 [0.99, 1.00] & +0.01 \\
model & gpt2 & mpt-chat & 0.91 [0.89, 0.94] & 1.00 [0.99, 1.00] & +0.08 \\
model & llama-chat & gpt2 & 0.85 [0.81, 0.88] & 0.99 [0.98, 0.99] & +0.14 \\
model & llama-chat & llama-chat & 1.00 [1.00, 1.00] & 1.00 [1.00, 1.00] & -0.00 \\
model & llama-chat & mpt & 0.82 [0.76, 0.86] & 0.99 [0.99, 1.00] & +0.18 \\
model & llama-chat & mpt-chat & 0.98 [0.96, 0.99] & 1.00 [0.99, 1.00] & +0.02 \\
model & mpt & gpt2 & 0.97 [0.96, 0.98] & 0.99 [0.98, 0.99] & +0.01 \\
model & mpt & llama-chat & 0.84 [0.80, 0.87] & 1.00 [1.00, 1.00] & +0.16 \\
model & mpt & mpt & 1.00 [1.00, 1.00] & 0.99 [0.99, 1.00] & -0.00 \\
model & mpt & mpt-chat & 0.87 [0.82, 0.92] & 1.00 [0.99, 1.00] & +0.13 \\
model & mpt-chat & gpt2 & 0.79 [0.76, 0.82] & 0.99 [0.98, 0.99] & +0.20 \\
model & mpt-chat & llama-chat & 0.99 [0.99, 1.00] & 1.00 [1.00, 1.00] & +0.01 \\
model & mpt-chat & mpt & 0.86 [0.84, 0.88] & 0.99 [0.99, 1.00] & +0.14 \\
model & mpt-chat & mpt-chat & 1.00 [1.00, 1.00] & 1.00 [0.99, 1.00] & -0.00 \\
\midrule
repetition\_ penalty & no & yes & 0.85 [0.79, 0.89] & 0.99 [0.99, 0.99] & +0.14 \\
\bottomrule
\end{tabular}
\end{table}

\begin{table}[htbp]
\centering
\small
\caption{RAID benchmark results for AI recall, for different distribution shifts (training and evaluation \aigen data filtered to remove adversarial attacks). Results for out-of-distribution LLM outputs, with upper and lower bounds attached.}
\label{tab:tnr}
\begin{tabular}{p{1.6cm}p{1.4cm}p{1.4cm}p{2.3cm}p{2.4cm}c}
\toprule
\textbf{Shift} & \textbf{Source} & \textbf{Target} & \textbf{Supervised, AI Recall} & \textbf{PU + TTA, AI Recall} & \textbf{PU + TTA Gain} \\
\midrule
decoding & greedy & sampling & 0.95 [0.92, 0.98] & 0.93 [0.87, 0.96] & -0.02 \\
\midrule
domain & abstracts & news & 0.99 [0.99, 1.00] & 1.00 [1.00, 1.00] & +0.01 \\
domain & books & abstracts & 0.99 [0.96, 1.00] & 1.00 [1.00, 1.00] & +0.01 \\
domain & news & recipes & 0.99 [0.98, 1.00] & 1.00 [1.00, 1.00] & +0.01 \\
domain & reddit & reviews & 1.00 [1.00, 1.00] & 0.99 [0.99, 1.00] & -0.01 \\
domain & reviews & reddit & 0.98 [0.95, 1.00] & 0.99 [0.99, 1.00] & +0.01 \\
\midrule
model & gpt2 & gpt2 & 1.00 [0.99, 1.00] & 0.97 [0.94, 0.99] & -0.03 \\
model & gpt2 & llama-chat & 0.74 [0.62, 0.90] & 0.99 [0.99, 1.00] & +0.25 \\
model & gpt2 & mpt & 0.97 [0.94, 0.99] & 0.97 [0.91, 0.99] & -0.01 \\
model & gpt2 & mpt-chat & 0.78 [0.68, 0.90] & 0.99 [0.97, 1.00] & +0.21 \\
model & llama-chat & gpt2 & 0.55 [0.40, 0.66] & 0.97 [0.94, 0.99] & +0.42 \\
model & llama-chat & llama-chat & 1.00 [0.99, 1.00] & 0.99 [0.99, 1.00] & -0.00 \\
model & llama-chat & mpt & 0.41 [0.28, 0.51] & 0.97 [0.91, 0.99] & +0.55 \\
model & llama-chat & mpt-chat & 0.90 [0.85, 0.96] & 0.99 [0.97, 1.00] & +0.08 \\
model & mpt & gpt2 & 0.85 [0.78, 0.90] & 0.97 [0.94, 0.99] & +0.12 \\
model & mpt & llama-chat & 0.47 [0.38, 0.56] & 0.99 [0.99, 1.00] & +0.52 \\
model & mpt & mpt & 0.99 [0.99, 1.00] & 0.97 [0.91, 0.99] & -0.03 \\
model & mpt & mpt-chat & 0.42 [0.24, 0.58] & 0.99 [0.97, 1.00] & +0.57 \\
model & mpt-chat & gpt2 & 0.46 [0.34, 0.67] & 0.97 [0.94, 0.99] & +0.50 \\
model & mpt-chat & llama-chat & 0.97 [0.94, 1.00] & 0.99 [0.99, 1.00] & +0.02 \\
model & mpt-chat & mpt & 0.66 [0.58, 0.76] & 0.97 [0.91, 0.99] & +0.31 \\
model & mpt-chat & mpt-chat & 1.00 [1.00, 1.00] & 0.99 [0.97, 1.00] & -0.01 \\
\midrule
repetition\_ penalty & no & yes & 0.80 [0.75, 0.85] & 0.98 [0.96, 0.99] & +0.18 \\
\bottomrule
\end{tabular}
\end{table}

\begin{table}[htbp]
\centering
\small
\caption{RAID benchmark results for human recall, for different distribution shifts (training and evaluation \aigen data filtered to remove adversarial attacks). Results for out-of-distribution LLM outputs, with upper and lower bounds attached.}
\label{tab:tpr}
\begin{tabular}{p{1.6cm}p{1.4cm}p{1.4cm}p{2.4cm}p{2.3cm}c}
\toprule
\textbf{Shift} & \textbf{Source} & \textbf{Target} & \textbf{Supervised, Human Recall} & \textbf{PU + TTA, Human Recall} & \textbf{PU + TTA Gain} \\
\midrule
decoding & greedy & sampling & 0.79 [0.72, 0.87] & 0.79 [0.67, 0.91] & -0.00 \\
\midrule
domain & abstracts & news & 0.13 [0.09, 0.18] & 0.14 [0.12, 0.17] & +0.01 \\
domain & books & abstracts & 0.24 [0.16, 0.32] & 0.67 [0.55, 0.84] & +0.43 \\
domain & news & recipes & 0.10 [0.06, 0.13] & 0.65 [0.29, 0.86] & +0.55 \\
domain & reddit & reviews & 0.12 [0.08, 0.16] & 0.66 [0.52, 0.81] & +0.54 \\
domain & reviews & reddit & 0.11 [0.05, 0.25] & 0.13 [0.08, 0.17] & +0.02 \\
\midrule
model & gpt2 & gpt2 & 0.88 [0.74, 0.95] & 0.90 [0.84, 0.96] & +0.02 \\
model & gpt2 & llama-chat & 0.88 [0.74, 0.95] & 0.95 [0.93, 0.97] & +0.07 \\
model & gpt2 & mpt & 0.88 [0.74, 0.95] & 0.97 [0.95, 0.98] & +0.08 \\
model & gpt2 & mpt-chat & 0.88 [0.74, 0.95] & 0.87 [0.53, 0.97] & -0.02 \\
model & llama-chat & gpt2 & 0.93 [0.88, 0.98] & 0.90 [0.84, 0.96] & -0.03 \\
model & llama-chat & llama-chat & 0.93 [0.88, 0.98] & 0.95 [0.93, 0.97] & +0.02 \\
model & llama-chat & mpt & 0.93 [0.88, 0.98] & 0.97 [0.95, 0.98] & +0.03 \\
model & llama-chat & mpt-chat & 0.93 [0.88, 0.98] & 0.87 [0.53, 0.97] & -0.07 \\
model & mpt & gpt2 & 0.97 [0.95, 0.99] & 0.90 [0.84, 0.96] & -0.06 \\
model & mpt & llama-chat & 0.97 [0.95, 0.99] & 0.95 [0.93, 0.97] & -0.01 \\
model & mpt & mpt & 0.97 [0.95, 0.99] & 0.97 [0.95, 0.98] & +0.00 \\
model & mpt & mpt-chat & 0.97 [0.95, 0.99] & 0.87 [0.53, 0.97] & -0.10 \\
model & mpt-chat & gpt2 & 0.91 [0.79, 0.95] & 0.90 [0.84, 0.96] & -0.00 \\
model & mpt-chat & llama-chat & 0.91 [0.79, 0.95] & 0.95 [0.93, 0.97] & +0.05 \\
model & mpt-chat & mpt & 0.91 [0.79, 0.95] & 0.97 [0.95, 0.98] & +0.06 \\
model & mpt-chat & mpt-chat & 0.91 [0.79, 0.95] & 0.87 [0.53, 0.97] & -0.04 \\
\midrule
repetition\_ penalty & no & yes & 0.73 [0.53, 0.88] & 0.89 [0.80, 0.96] & +0.16 \\
\bottomrule
\end{tabular}
\end{table}

\clearpage
\newpage

\section{ConDA, Fast-DetectGPT, and Test-time training (TTT) benchmarks}\label{sec:conda}

\paragraph{ConDA details.} We implement ConDA as described in \citet{bhattacharjee2023conda}. We use DistilBERT (instead of RoBERTa, as in the original paper); we treat the source domain as being a mix of labeled LLM and human sentences, and the test domain as being an unlabeled mix of human sentences and LLM sentences from a different LLM from training (the exact same setup as for PNU + TTA; i.e., we reuse the training and evaluation sets from PNU for ConDA). We train models for 3 epochs and a learning rate of .00001, using the same contrastive loss as in the original paper. ConDA also requires a text transformation for the contrastive loss: while synonym replacement is used in the original paper, we instead use token dropout with a probability of 0.2. We evaluate our ConDA models on the test sets of all LLMs, and show results on the test domain that was used in training along with averages on all other domains.

\paragraph{Fast-DetectGPT details.} We implement Fast-DetectGPT for zero-shot AI detection as described in \citet{bao2023fast}. This method requires open-weight LLMs to serve as a \textit{scoring model} (gives log-probabilities for the token at each index of a test-time text), and a \textit{sampling model} (defines a distribution over possible next tokens at each index). Using these two distributions, Fast-DetectGPT then computes the \textit{conditional probability curvature}, a continuous score that positively correlates to the likelihood that a text is AI-generated. In our experiments, we use 3 distinct combinations of scoring and sampling model: (1) a two-model pair with gpt-j-6B as the sampling model and gpt-neo-2.7B as the scoring model; (2) gpt-neo-2.7B as both the sampling and scoring model; and (3) Qwen2.5-7B as both the sampling and scoring model. (1) and (2) are both used in the original paper; we add (3) as an extension to a newer open-source LLM. We evaluate on the same evaluation set per-LLM as in \Cref{fig:llm-shift} (468 human abstracts and 468 AI mirrors per-LLM), and use AUC as the primary evaluation metric (the same as what was used in the original paper).

\paragraph{Test-Time Training details.} We implement Test-Time Training (TTT) as described in \citet{sun2020test}; we train and adapt models at both the sentence and abstract level to demonstrate robustness to texts of differing lengths. In short, TTT requires a main prediction task for which a model is pretrained (in this case, binary classification with text covariates) and an auxiliary self-supervised task, for which we choose masked language modeling (MLM). Architecturally, TTT requires two classification heads for each of the MLM and classification losses; the DistilBERT encoder is involved in the forward and backward passes for both classification heads and is therefore jointly optimized for MLM and classification.

We train 3 DistilBERT models on different training seeds, with each model being pretrained on data from 8,000 abstracts (4,000 naive AI abstracts and 4,000 human abstracts). We train each model for 3 epochs and a learning rate of .00001. During pre-training TTT jointly optimizes classification loss plus MLM loss. Then, given a stream of test-time data, or batches of test-time data, prior to scoring each batch, TTT performs self-supervised updates to the model with the MLM loss using only the unlabeled test text, and then performs a forward pass with the adapted model. We use the same test-time data as in \Cref{fig:adversarial} (an even mix per-batch of human writing and adversarial AI writing from 4,000 held-out abstracts). For sentences, we use a test-time learning rate of .01 and batch size of 64; for abstracts, we use a learning rate of .001 and batch size of 2.

\begin{figure}[tbh]

  \centering
  \includegraphics[width=.7\linewidth]{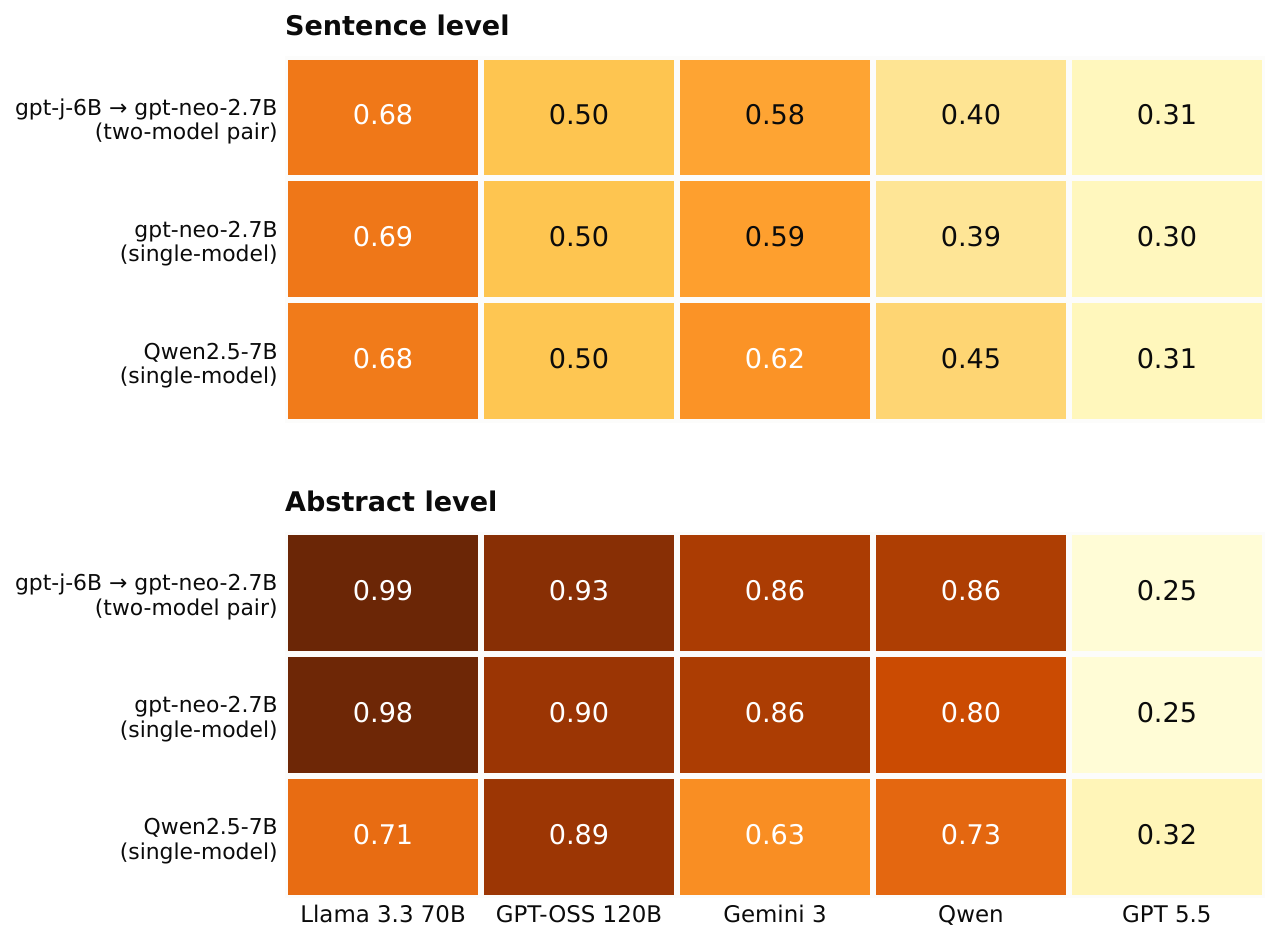}

\caption{AUC of Fast-DetectGPT on AI mirrors, for three combinations of scoring/sampling model. While the method works poorly at the sentence level, it is capable of reliably detecting abstracts written by some LLMs (e.g., Llama 3.3 70B Instruct). However, even at the abstract level, there are significant differences in AUC on AI mirrors by different LLMs---notably, AUC is under 0.5 for GPT 5.5.}

\label{fig:fastdetect}
\end{figure}

\begin{figure}[tbh]

  \centering
  \includegraphics[width=\linewidth]{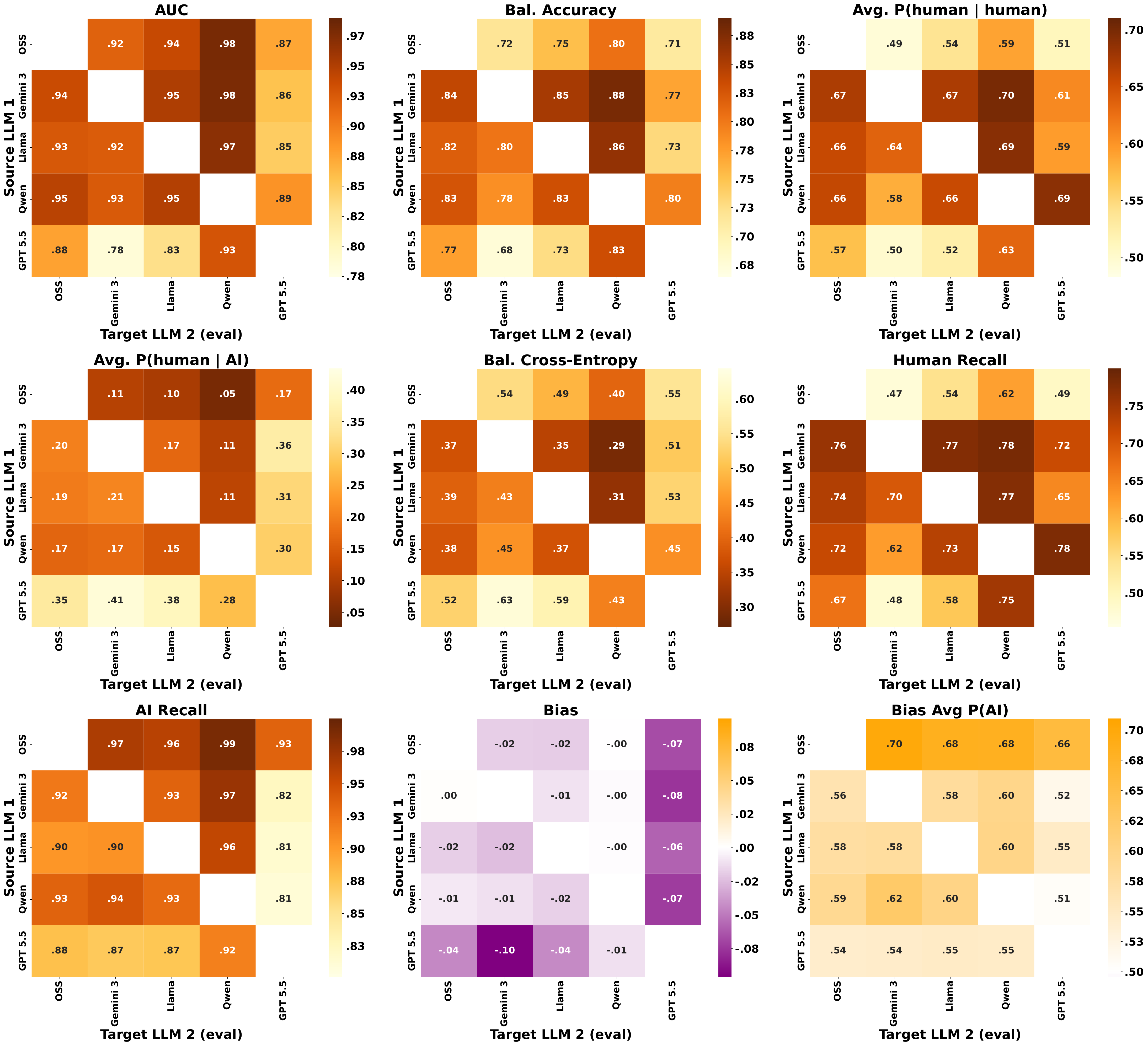}

\caption{We train ConDA as a TTA method, on labeled human writing and AI mirrors from one LLM (source distribution), along with an unlabeled mix of human writing and AI mirrors from a second LLM (target distribution); we evaluate on held-out text from the target distribution. ConDA adapts well to the test-time distribution.}

\label{fig:conda1}
\end{figure}

\begin{figure}[tbh]

  \centering
  \includegraphics[width=\linewidth]{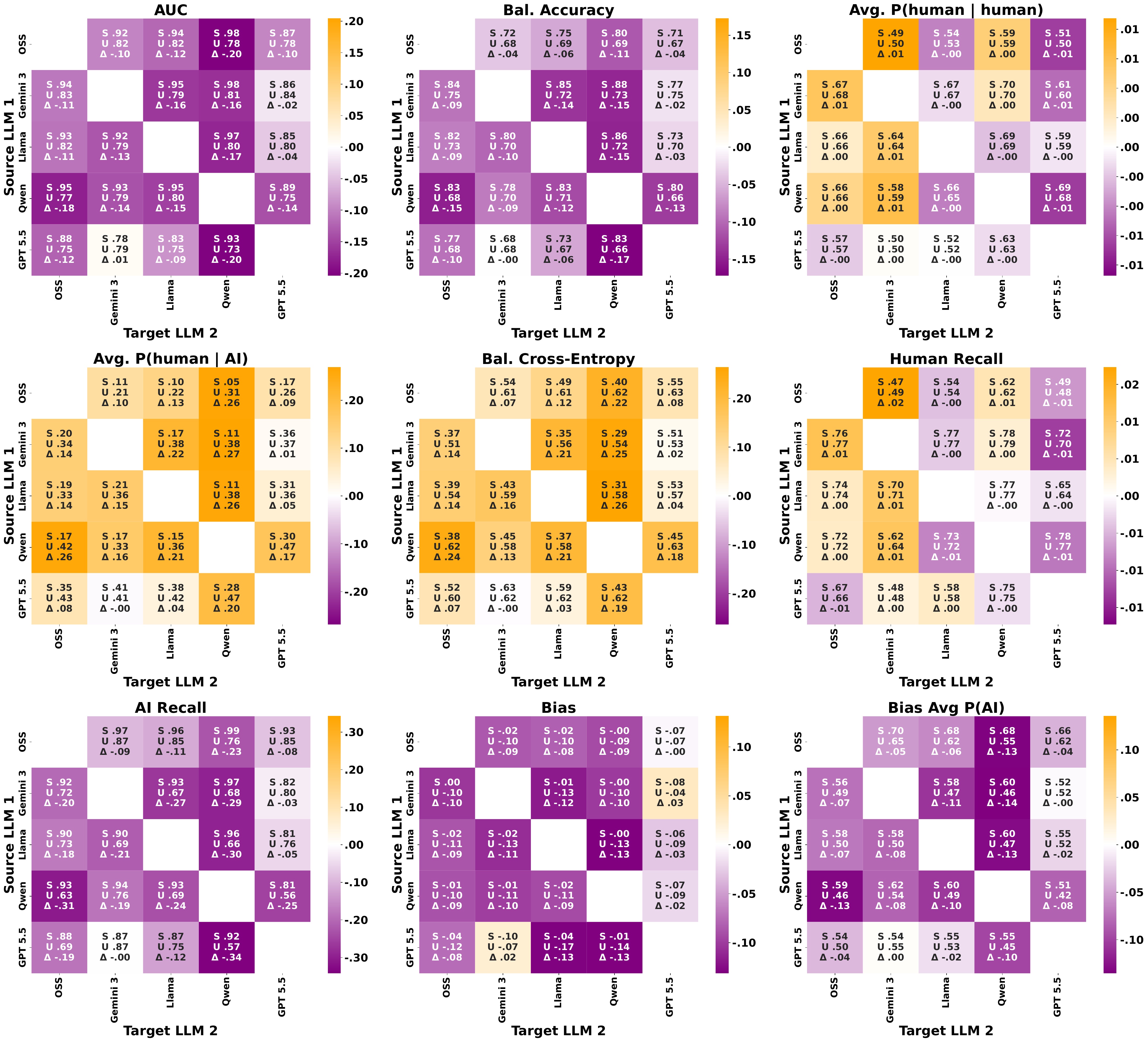}

\caption{We train ConDA as a TTA method, on labeled human writing and AI mirrors from one LLM (source distribution), along with an unlabeled mix of human writing and AI mirrors from a second LLM (target distribution). Here, the test-time distribution consists of an average over all LLMs other than the ones in the source and target distribution; we plot the raw metrics and difference from \Cref{fig:conda1}. Perhaps unsurprisingly, ConDA does poorly on LLMs not contained in either of its source or target distributions at train-time (lower AUC and accuracy; higher cross-entropy and $\Pr(\text{human} \mid \text{AI})$, e.g.). Each cell shows S (target LLM, as in \Cref{fig:conda1}), U (average over unseen LLMs) and $\Delta$ = U - S.}

\label{fig:conda2}
\end{figure}

\begin{figure}[tbh]

\begin{subfigure}{.5\linewidth}
  \centering
  \includegraphics[width=\linewidth]{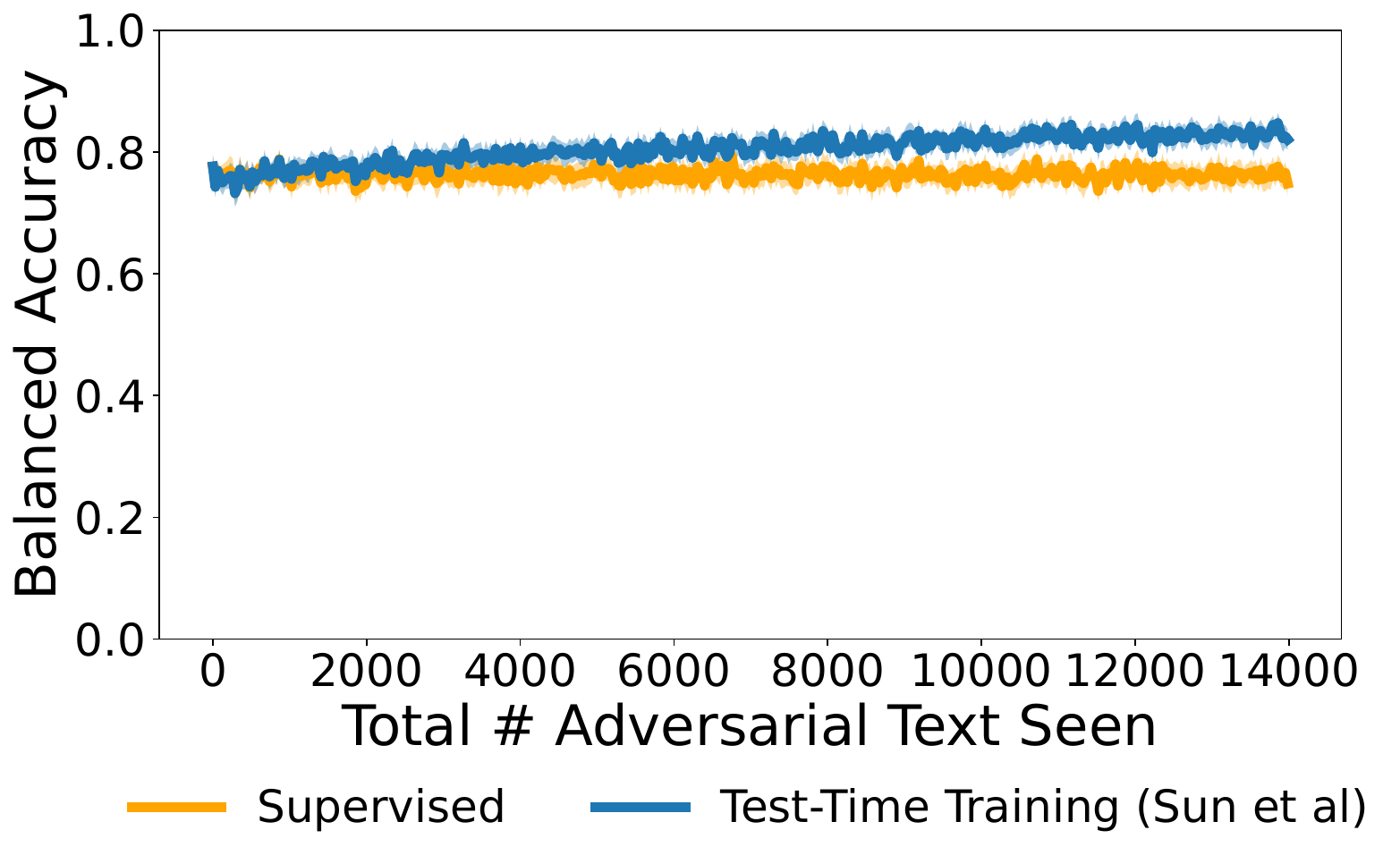}
  \caption{}
  \label{fig:ttt100-1}
\end{subfigure}\begin{subfigure}{.5\linewidth}
  \centering
  \includegraphics[width=\linewidth]{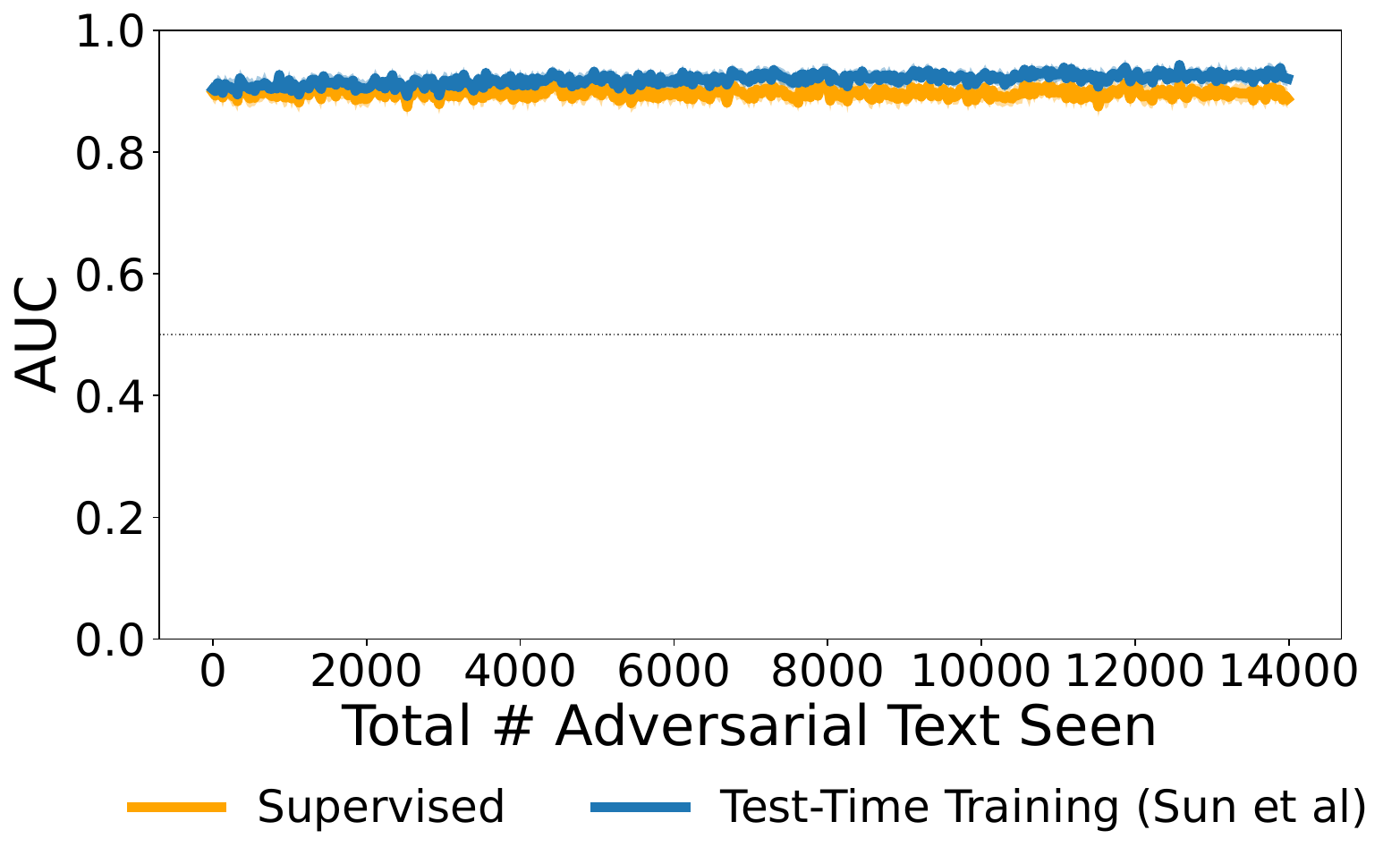}
  \caption{}
  \label{fig:ttt100-2}
\end{subfigure}\\
\begin{subfigure}{.5\linewidth}
  \centering
  \includegraphics[width=\linewidth]{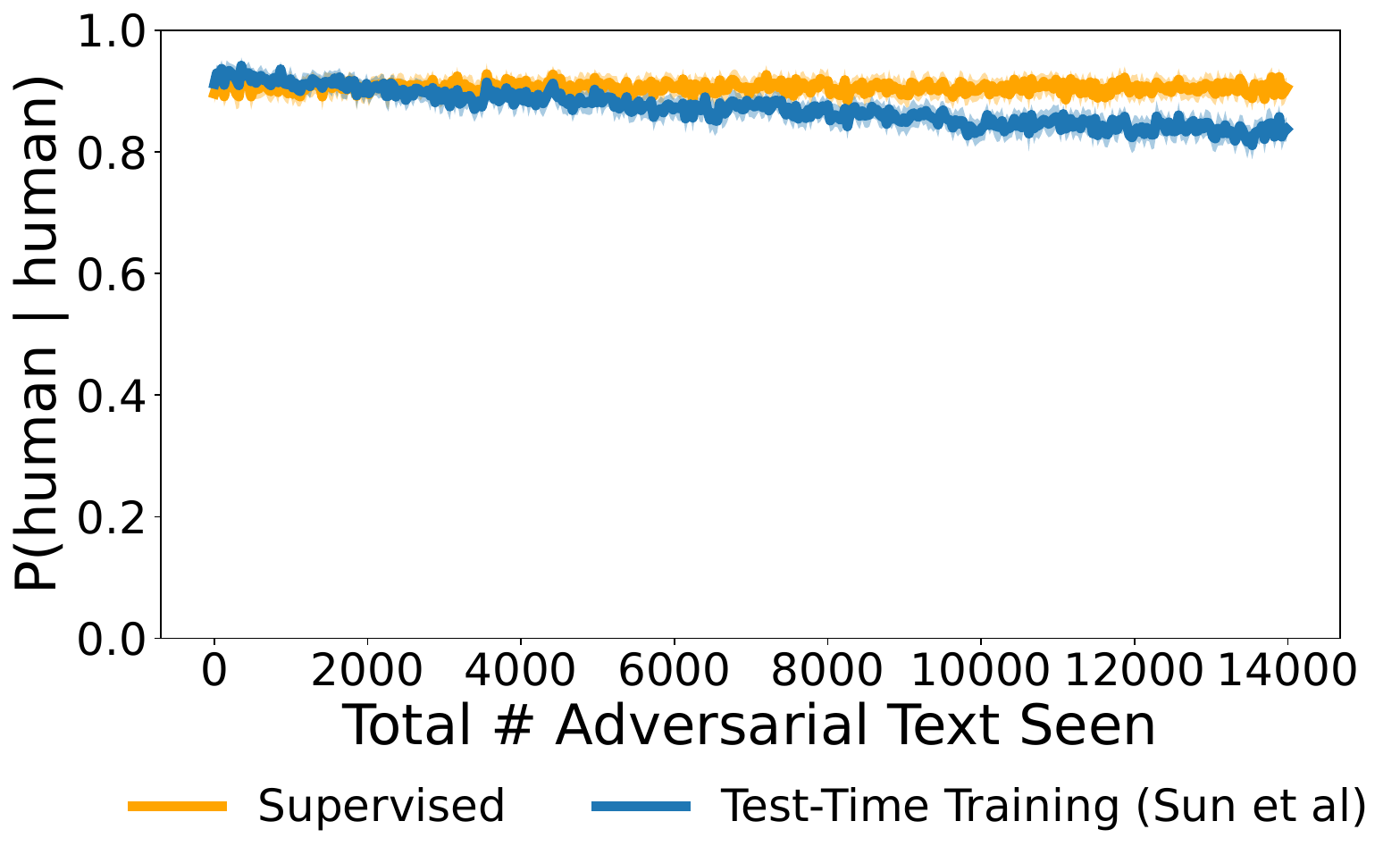}
  \caption{}
  \label{fig:ttt100-3}
\end{subfigure}\begin{subfigure}{.5\linewidth}
  \centering
  \includegraphics[width=\linewidth]{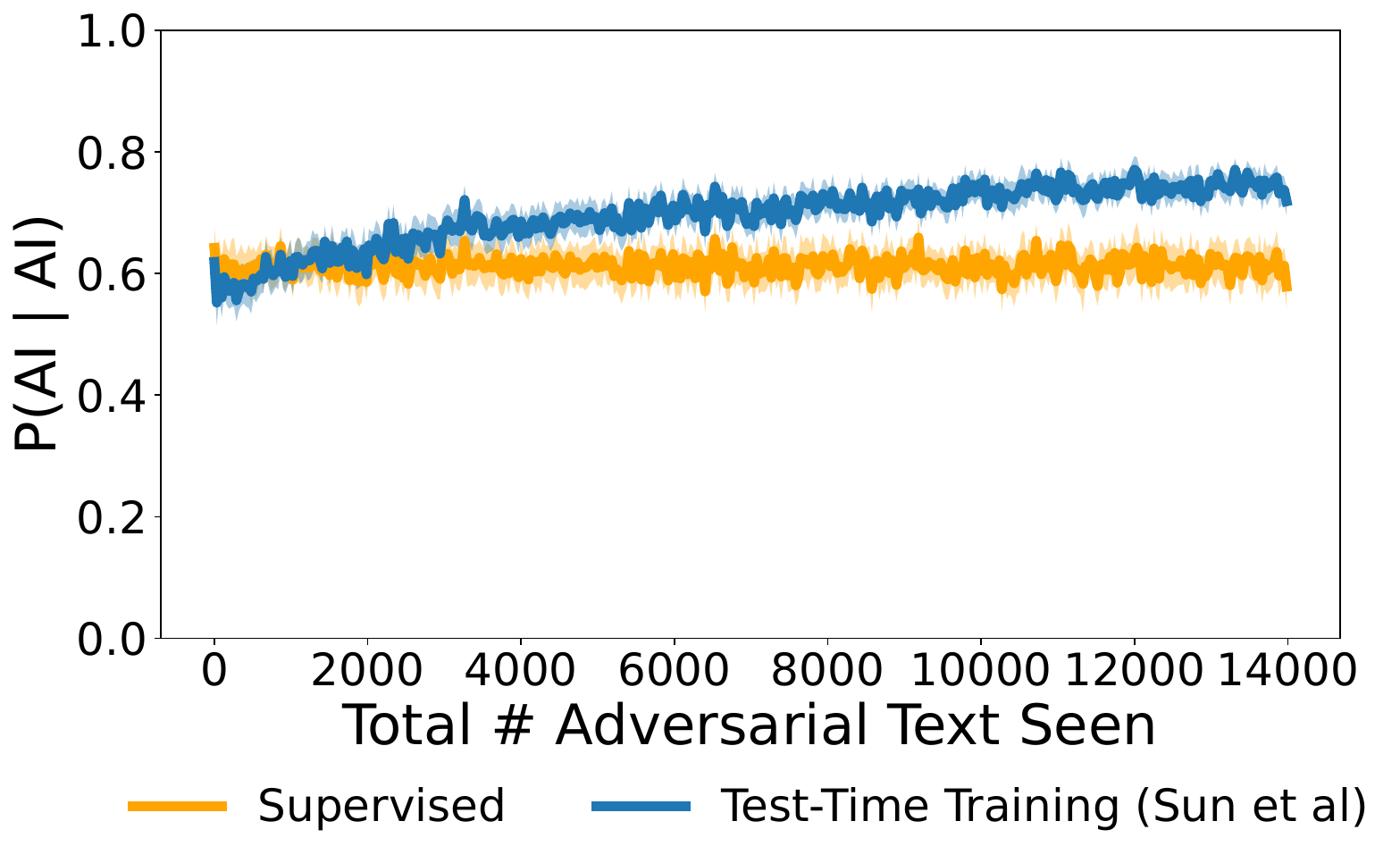}
  \caption{}
  \label{fig:ttt100-4}
\end{subfigure}\\
\begin{subfigure}{.5\linewidth}
  \centering
  \includegraphics[width=\linewidth]{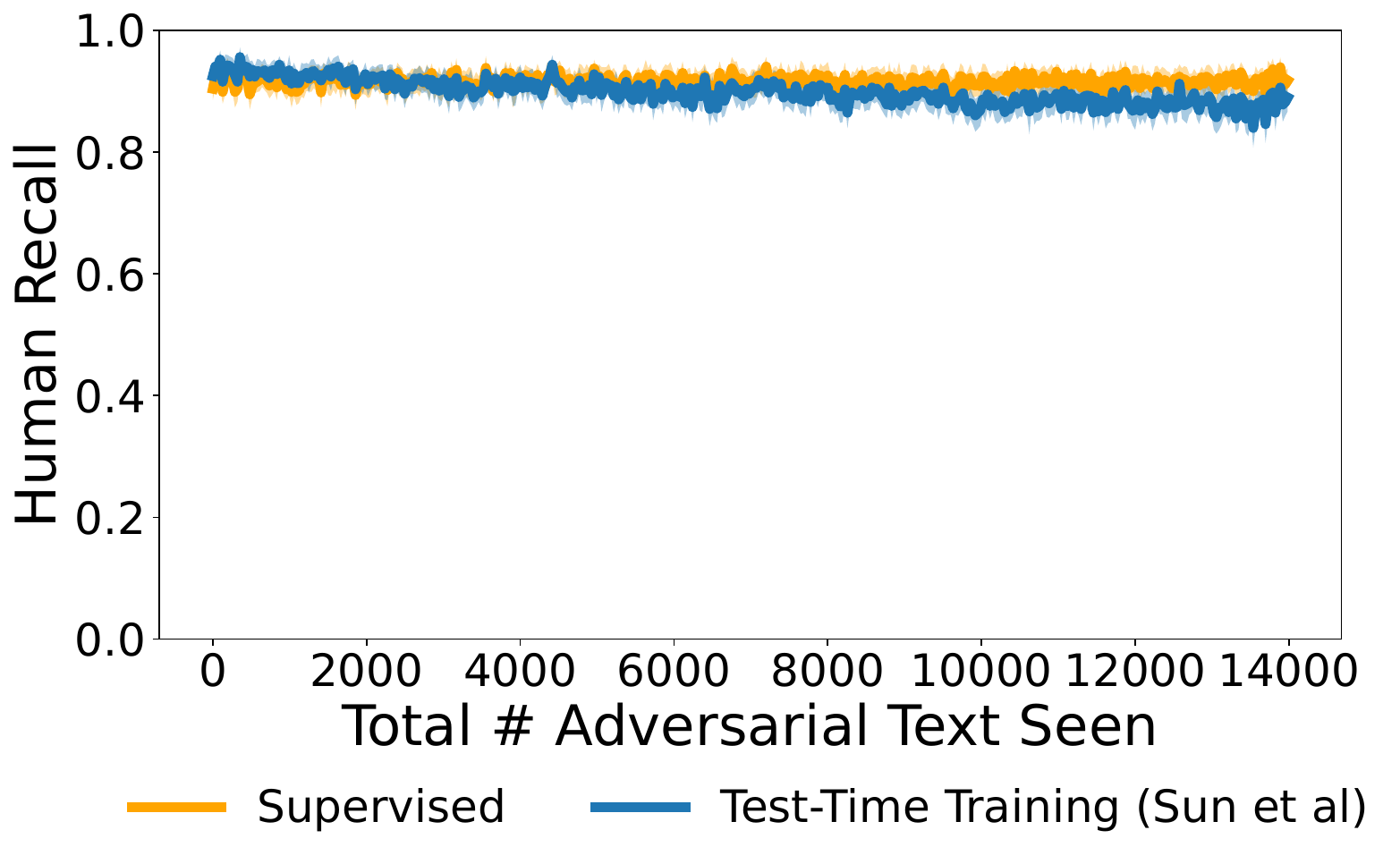}
  \caption{}
  \label{fig:ttt100-5}
\end{subfigure}
\label{fig:ttt100}
\caption{Test-Time Training on a supervised model trained on naive AI and human writing from 2020. The test-time stream consists of an even mix of human and AI sentences; all test-time AI text is adversarial. Evaluation is on the entire test-time batch.}
\end{figure}

\begin{figure}[tbh]

\begin{subfigure}{.5\linewidth}
  \centering
  \includegraphics[width=\linewidth]{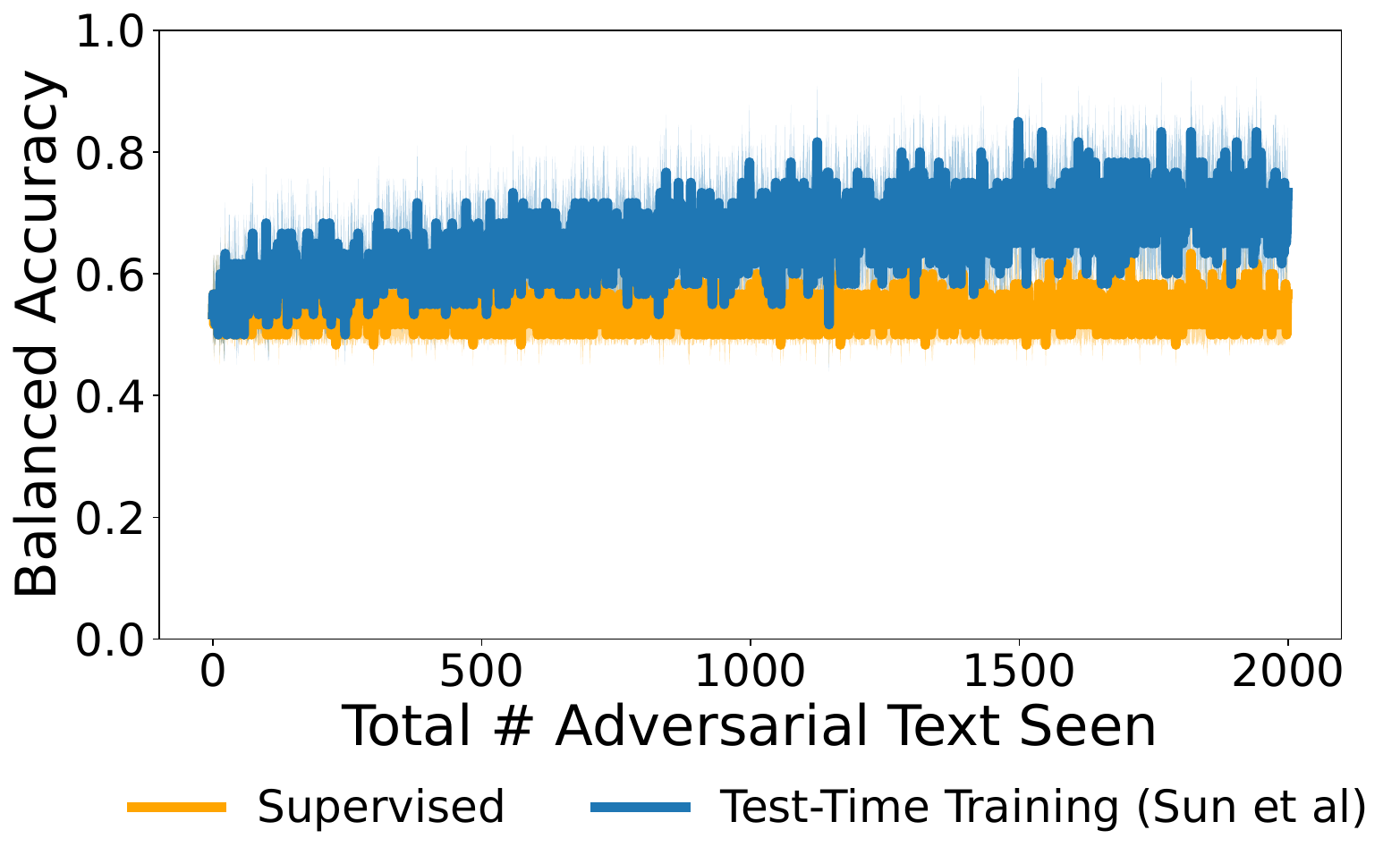}
  \caption{}
  \label{fig:ttt100abs-1}
\end{subfigure}\begin{subfigure}{.5\linewidth}
  \centering
  \includegraphics[width=\linewidth]{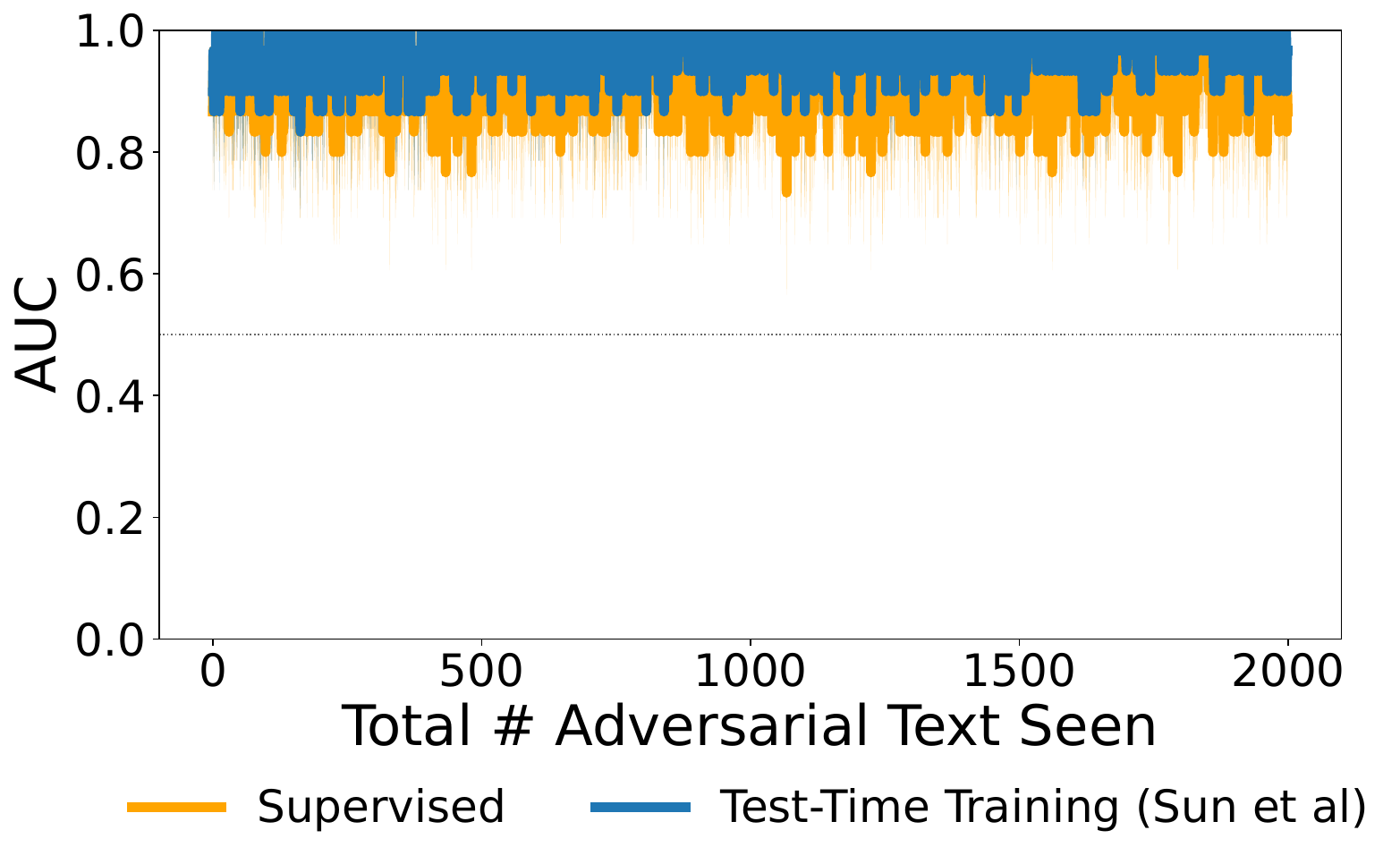}
  \caption{}
  \label{fig:ttt100abs-2}
\end{subfigure}\\
\begin{subfigure}{.5\linewidth}
  \centering
  \includegraphics[width=\linewidth]{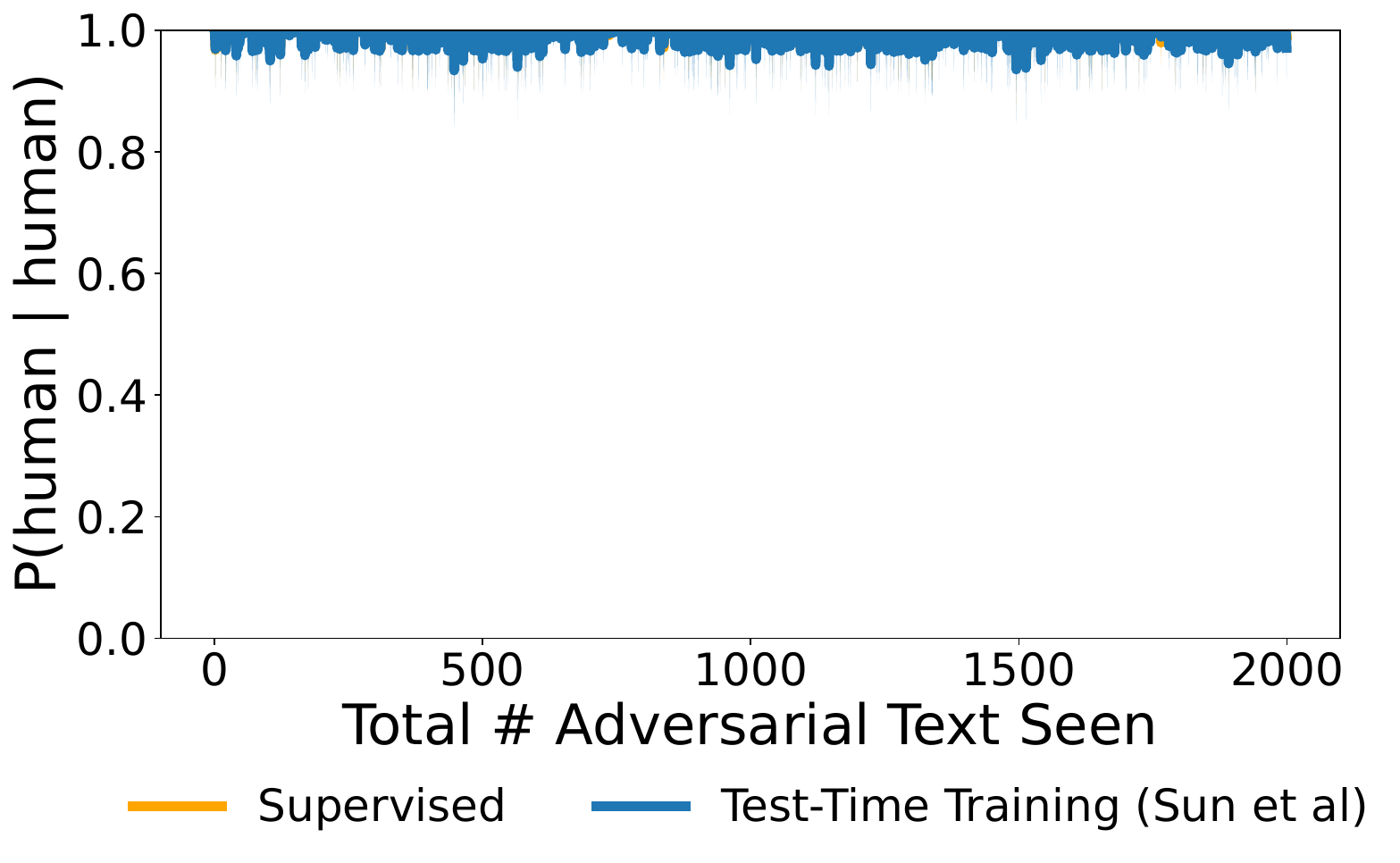}
  \caption{}
  \label{fig:ttt100abs-3}
\end{subfigure}\begin{subfigure}{.5\linewidth}
  \centering
  \includegraphics[width=\linewidth]{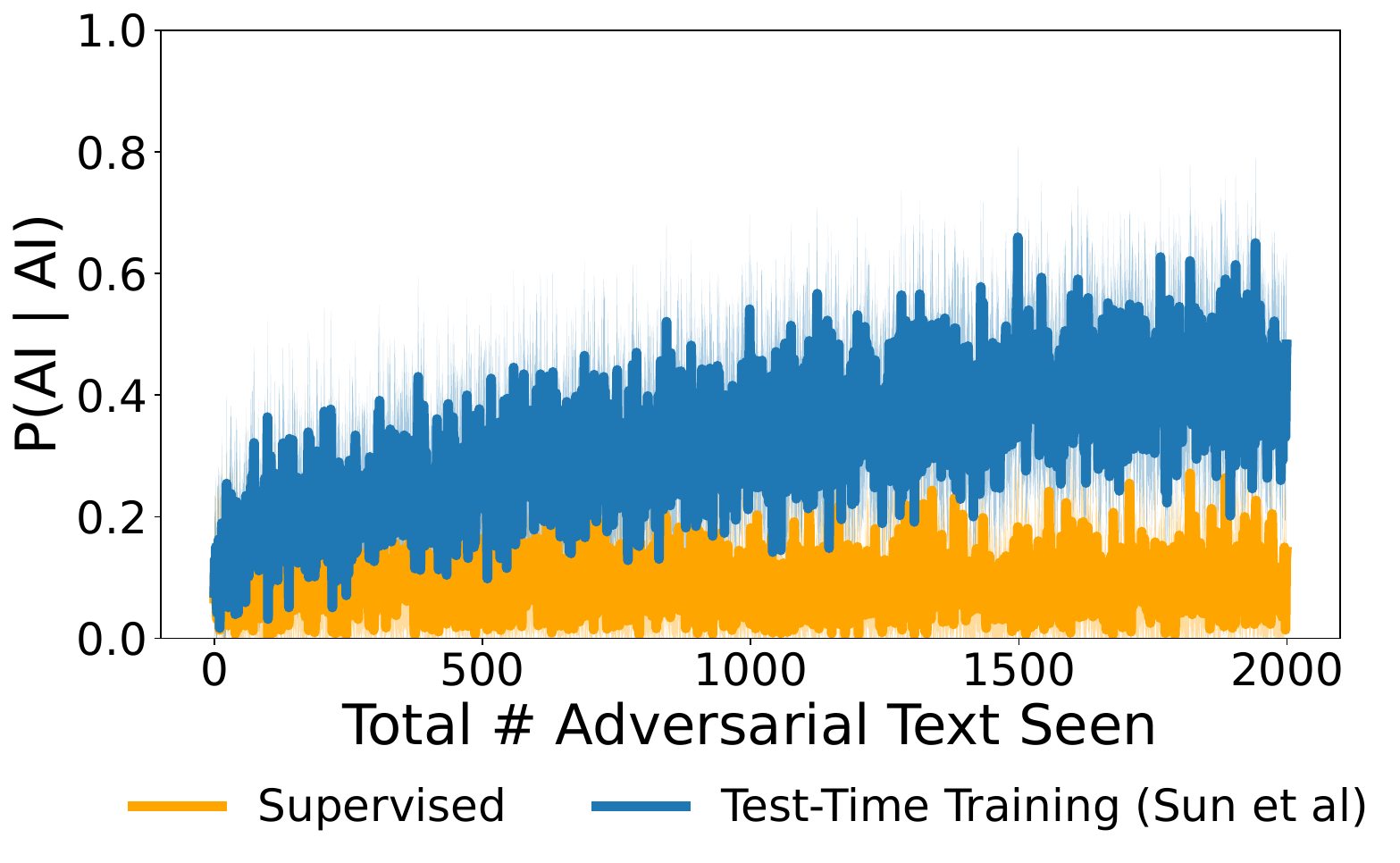}
  \caption{}
  \label{fig:ttt100abs-4}
\end{subfigure}\\
\begin{subfigure}{.5\linewidth}
  \centering
  \includegraphics[width=\linewidth]{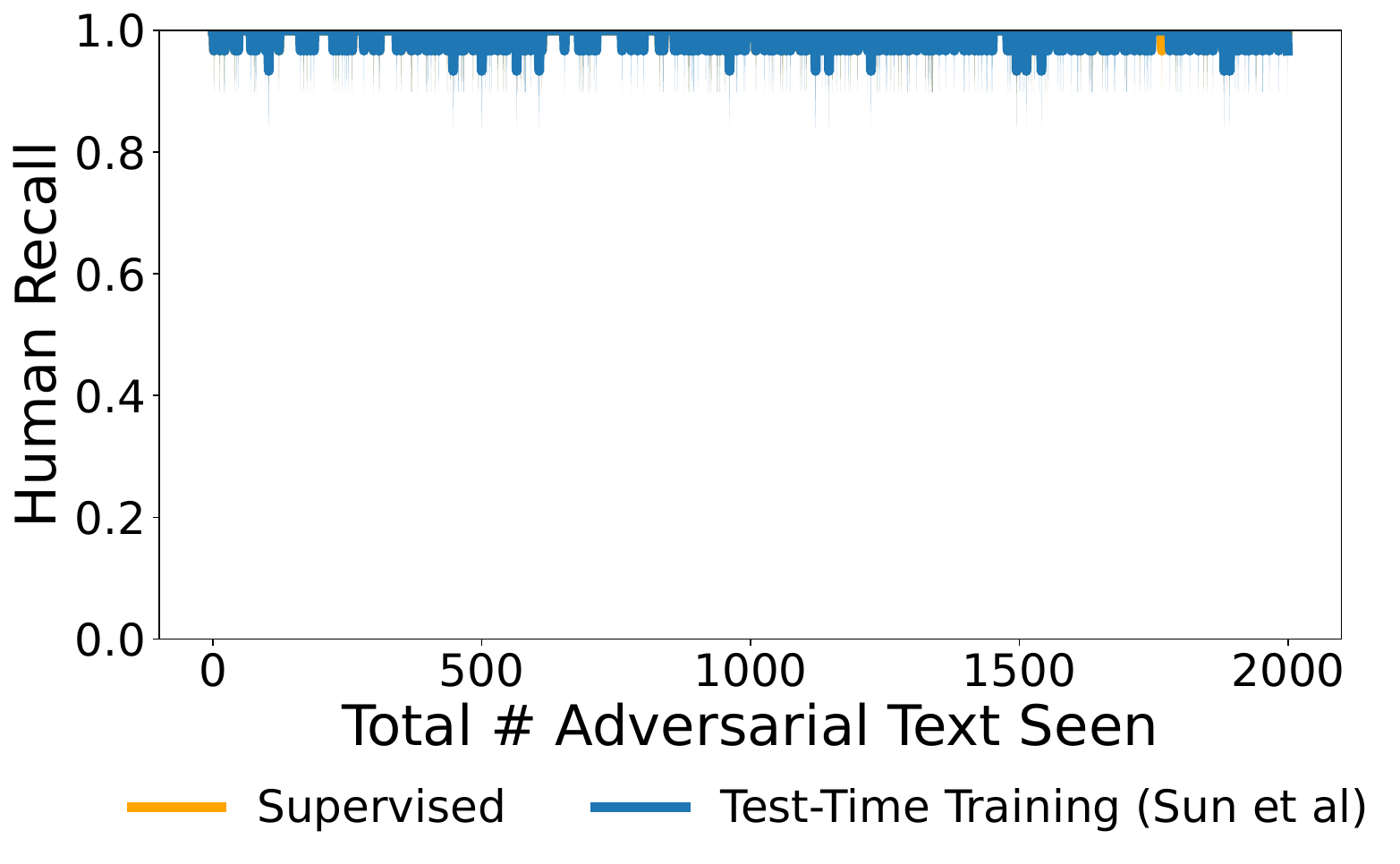}
  \caption{}
  \label{fig:ttt100abs-5}
\end{subfigure}\begin{subfigure}{.5\linewidth}
  \centering
  \includegraphics[width=\linewidth]{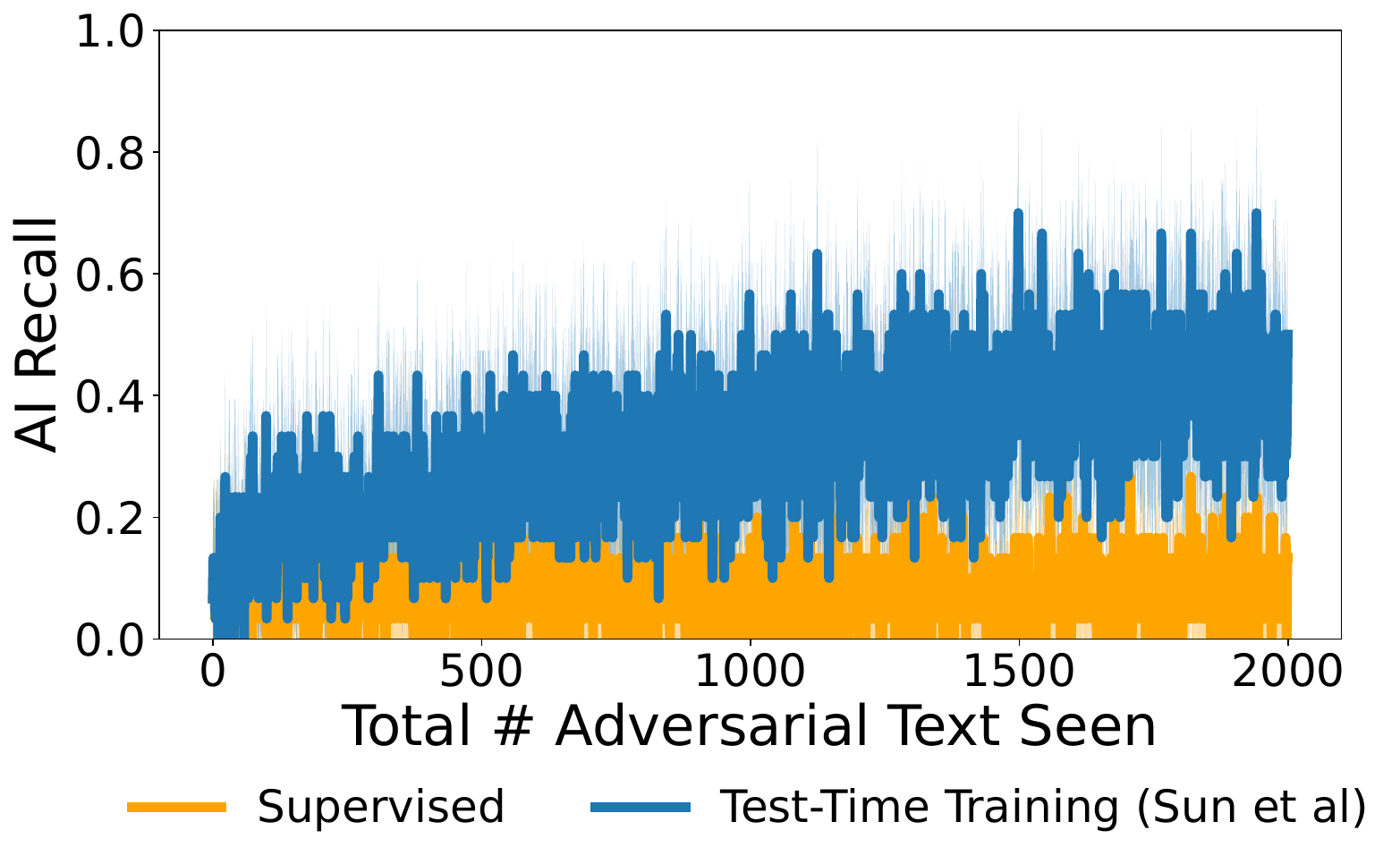}
  \caption{}
  \label{fig:ttt100abs-6}
\end{subfigure}

\label{fig:ttt100abs}
\caption{Test-Time Training on a supervised model trained on naive AI and human writing from 2020. The test-time stream consists of an even mix of human and AI abstracts; all test-time AI text is adversarial. Evaluation is on the entire test-time batch.}
\end{figure}

\clearpage
\newpage

\section{Copying analysis}\label{sec:copy}

\begin{table}[tbh]
\centering
\small
\setlength{\tabcolsep}{5pt}
\caption{Precision: fraction of each rewrite's $n$-grams occurring verbatim in
the corresponding human abstract, reported as median / p95 over 10{,}000 rows. Each AI rewrite is compared to the original human abstract its information is derived from.}
\label{tab:precision-v2}
\begin{tabular}{lccccc}
\toprule
Rewrite & unigram & 2-gram & 3-gram & 4-gram & 5-gram \\
\midrule
\texttt{Naive}                & .46 / .57 & .09 / .17 & .03 / .07 & .01 / .04 & .003 / .02 \\
\texttt{Beats Pangram 3.2}    & .44 / .56 & .08 / .15 & .02 / .06 & .01 / .03 & .000 / .02 \\
\texttt{Beats Pangram 3.3.2}  & .45 / .58 & .08 / .17 & .02 / .07 & .01 / .04 & .000 / .02 \\
\texttt{Iteration 2 PU}       & .45 / .58 & .06 / .13 & .01 / .05 & .00 / .02 & .000 / .01 \\
\texttt{Iteration 2 PN}       & .45 / .59 & .07 / .16 & .02 / .07 & .00 / .03 & .000 / .01 \\
\texttt{Iteration 3 PU}       & .43 / .55 & .06 / .14 & .01 / .06 & .00 / .03 & .000 / .01 \\
\texttt{Iteration 3 PN}       & .46 / .56 & .09 / .16 & .01 / .05 & .00 / .02 & .000 / .01 \\
\bottomrule
\end{tabular}
\end{table}

\begin{table}[tbh]
\centering
\small
\setlength{\tabcolsep}{5pt}
\caption{Recall: fraction of each human abstract's $n$-grams occurring
verbatim in the corresponding rewrite, reported as median / p95 over 10{,}000 rows. Each AI rewrite is compared to the original human abstract its information is derived from.}
\label{tab:recall-v2}
\begin{tabular}{lccccc}
\toprule
Rewrite & unigram & 2-gram & 3-gram & 4-gram & 5-gram \\
\midrule
\texttt{Naive}                & .70 / .81 & .20 / .36 & .07 / .16 & .02 / .09 & .01 / .05 \\
\texttt{Beats Pangram 3.2}    & .57 / .69 & .10 / .21 & .02 / .08 & .01 / .04 & .00 / .02 \\
\texttt{Beats Pangram 3.3.2}  & .52 / .66 & .09 / .19 & .02 / .08 & .01 / .04 & .00 / .02 \\
\texttt{Iteration 2 PU}       & .43 / .58 & .05 / .12 & .01 / .04 & .00 / .02 & .00 / .01 \\
\texttt{Iteration 2 PN}       & .44 / .60 & .07 / .15 & .01 / .06 & .00 / .03 & .00 / .01 \\
\texttt{Iteration 3 PU}       & .38 / .49 & .04 / .09 & .01 / .03 & .00 / .01 & .00 / .01 \\
\texttt{Iteration 3 PN}       & .55 / .70 & .10 / .22 & .03 / .09 & .01 / .04 & .00 / .02 \\
\bottomrule
\end{tabular}
\end{table}

\begin{table}[tbh]
\centering
\small
\setlength{\tabcolsep}{6pt}
\caption{Longest contiguous run of tokens shared between the rewrite and the
corresponding human abstract (in tokens), plus the share of rows with any shared run
$\geq$15 / $\geq$20 tokens. Symmetric metric (same in both directions). Each AI rewrite is compared to the original human abstract its information is derived from.}
\label{tab:lcs-v2}
\begin{tabular}{lccccc}
\toprule
Rewrite & median & p99 & max & rows $\geq$15 & rows $\geq$20 \\
\midrule
\texttt{Naive}                & 5 & 12 & 25 & 0.35\% & 0.06\% \\
\texttt{Beats Pangram 3.2}    & 4 &  9 & 20 & 0.08\% & 0.01\% \\
\texttt{Beats Pangram 3.3.2}  & 4 &  9 & 16 & 0.04\% & 0.00\% \\
\texttt{Iteration 2 PU}       & 3 &  6 & 13 & 0.00\% & 0.00\% \\
\texttt{Iteration 2 PN}       & 3 &  7 & 14 & 0.00\% & 0.00\% \\
\texttt{Iteration 3 PU}       & 3 &  6 & 13 & 0.00\% & 0.00\% \\
\texttt{Iteration 3 PN}       & 4 &  9 & 16 & 0.05\% & 0.00\% \\
\bottomrule
\end{tabular}
\end{table}

\clearpage

\section{Text analysis per-year, naive AI writing and human abstracts}\label{sec:eda}

\begin{table}[tbh]
\centering
\caption{Mean sentences per abstract and words per sentence, by source and year.}
\label{tab:sent_word_stats}
\begin{tabular}{llrr}
\toprule
Year & Source & Sentences/Abstract & Words/Sentence \\
\midrule
2010 & human & 5.80 & 22.87 \\
2010 & Llama 3.3 70B Instruct & 11.02 & 26.80 \\
2010 & Gemini 3 Pro Preview & 7.64 & 24.08 \\
2010 & GPT OSS 120b & 13.88 & 26.56 \\
2010 & Qwen & 8.99 & 31.22 \\
\midrule
2012 & human & 5.99 & 23.03 \\
2012 & Llama 3.3 70B Instruct & 10.99 & 26.92 \\
2012 & Gemini 3 Pro Preview & 7.80 & 24.06 \\
2012 & GPT OSS 120b & 13.95 & 26.58 \\
2012 & Qwen & 9.11 & 31.15 \\
\midrule
2014 & human & 6.22 & 23.37 \\
2014 & Llama 3.3 70B Instruct & 11.40 & 26.73 \\
2014 & Gemini 3 Pro Preview & 8.01 & 24.15 \\
2014 & GPT OSS 120b & 14.00 & 26.48 \\
2014 & Qwen & 9.18 & 31.16 \\
\midrule
2016 & human & 6.38 & 23.62 \\
2016 & Llama 3.3 70B Instruct & 11.33 & 26.74 \\
2016 & Gemini 3 Pro Preview & 8.19 & 24.04 \\
2016 & GPT OSS 120b & 14.07 & 26.33 \\
2016 & Qwen & 9.27 & 31.17 \\
\midrule
2018 & human & 6.74 & 23.44 \\
2018 & Llama 3.3 70B Instruct & 11.57 & 26.53 \\
2018 & Gemini 3 Pro Preview & 8.64 & 23.71 \\
2018 & GPT OSS 120b & 14.03 & 26.36 \\
2018 & Qwen & 9.41 & 30.70 \\
\midrule
2020 & human & 7.02 & 23.49 \\
2020 & Llama 3.3 70B Instruct & 11.74 & 26.47 \\
2020 & Gemini 3 Pro Preview & 8.73 & 23.60 \\
2020 & GPT OSS 20b & 8.31 & 25.34 \\
2020 & Gemini 2.5 Flash & 10.24 & 23.91 \\
2020 & GPT OSS 120b & 14.04 & 26.22 \\
2020 & Qwen & 9.57 & 30.47 \\
2020 & GPT 5.5 & 4.99 & 20.22 \\
\midrule
2023 & human & 7.41 & 23.26 \\
2023 & Llama 3.3 70B Instruct & 11.82 & 26.65 \\
2023 & Gemini 3 Pro Preview & 9.08 & 23.41 \\
2023 & GPT OSS 20b & 8.49 & 25.04 \\
2023 & Gemini 2.5 Flash & 10.57 & 23.90 \\
\midrule
2025 & human & 7.51 & 23.13 \\
2025 & Llama 3.3 70B Instruct & 11.86 & 26.12 \\
2025 & Gemini 3 Pro Preview & 9.52 & 23.28 \\
2025 & GPT OSS 20b & 8.56 & 24.70 \\
2025 & Gemini 2.5 Flash & 11.01 & 23.59 \\
\bottomrule
\end{tabular}
\end{table}

\section{Analysis of post-2020 text with LLM-derived features}\label{app:20205}

By training new AI detection models with $(TED)^{n}$ to label post-2022 text that is predicted to be human with high confidence, we compare pre-2022 verified-human writing with post-2022 high-confidence human writing with respect to lexical, syntactic, stylistic, discourse, and substantive traits. At a high level, we (1) train detection models using AI mirrors of 2025 abstracts and unlabeled abstracts written in 2025 (i.e., a mixture of human and likely AI-generated writing), (2) generate features using GPT 5.5, and (3) analyze which features distinguish human writing in 2020 and high-confidence human writing in 2025, to get a sense of changes in human writing styles shortly after the introduction of LLMs.

For getting data for training detection models, feature generation, and downstream analysis, we sample arXiv cs.AI abstracts from 2020 (pre-ChatGPT) and 2025 (authorship unknown). For feature generation, we sample 200 2020 and 1{,}000 2025 abstracts, split 50/50 train/val, and split into sentences (train: 4{,}610 sentences; val: 4{,}528). We then score every 2025 sentence with five PU detectors trained with $(TED)^{n}$ (five seeds) on unlabeled real 2025 abstracts plus positives that are LLM mirrors of 2025 abstracts (the training data for the detection models consists of 2{,}500 abstracts disjoint from the corpus for feature generation), and keep 2025 sentences with mean $P(\text{human}) > 0.9$ across the five models. For feature generation, the comparison sets are therefore \{all 2020 human sentences\} vs.\ \{high-confidence-human 2025 sentences\}: train $= 2{,}189$ ($669 + 1{,}520$), val $= 2{,}131$ ($673 + 1{,}458$).

To generate features, we use GPT 5.5 via Codex at extra-high effort to produce two disjoint sets of 256 features each from the training sentences: a set of \emph{stylistic} features (lexical/syntactic/stylistic/discourse, e.g., first-person-plural authorial framing, concessive connectives) and a set of \emph{substantive} features (topic/content, e.g., LLM capability framing, chain-of-thought). After using GPT 5.5 to label the training and validation sets with each of the two feature sets, we fit one lasso model per set (regressing whether the text was written in 2020 or 2025 from the LLM annotations), choosing $C$ by cross-validation on the training data and evaluating on the validation data. The two regressions give highly different results. The substantive lasso (181/256 features selected) reaches a validation AUC of 0.752. Signs of 2025 human writing center on LLM capability framing, chain-of-thought/inference-time compute, agents and tool use, RAG/long-context, multimodal foundation models, and model compression; signs of 2020 human writing include classical NLP tasks, reinforcement learning, autonomous driving, optimization, and graph neural networks. Meanwhile, the stylistic lasso does not generalize: it achieves a validation AUC of only 0.549, and many of the features that appeared to be strong signals for 2020/2025 human writing during training do not generalize to the validation set.

Our comparative analysis suggests that, at the moment, substantive changes in research fields over time outweigh stylistic changes in human writing (e.g., stylistic convergence toward LLM prose)---while there may be human distribution shift in writing style over time, it is not large enough to be as measurable by our approach (compared to topic changes in research fields during the same time period).

Our findings are also consistent with the idea that our detectors have reasonably learned to distinguish between human and AI writing post-2022 (in the sense that, beyond substantive differences between pre- and post-2022 writing, our identified-human sentences in 2025 are reasonably similar stylistically to pre-2022 known human writing); however, they also demonstrate the limits of pure supervised approaches using pre-GPT era writing: these detectors may pick up distribution shifts in what is talked about---for example, these shifts may explain the performance of our supervised model in comparatively overstating the fraction of texts that are AI written over time, compared to our TTA model. 

 One limitation we note is that our comparison is likely conservative: since our filtering with PU detection models selects 2025 sentences judged to be human-like, it may understate any true stylistic drift (e.g., miss some human-written text closer to the decision boundary); the low stylistic AUC should therefore be read as a lower bound on stylistic separability, not as proof of none. We include the top 10 features by magnitude of the lasso coefficient for each of the stylistic and substantive regressions in \Cref{tab:stylistic-features} and \Cref{tab:substantive-features}, respectively.

\begin{table*}[t]
\centering
\small
\caption{Top 10 substantive features per direction selected by the substantive lasso
(181/256 features retained; validation AUC $= 0.752$). $\beta$ is the standardized
lasso coefficient (positive $=$ toward 2025). $\Delta$Train/$\Delta$Val are mean
feature-activation differences (i.e., difference in average prevalence; 2025 $-$ 2020) on the respective split. ``Sign match''
indicates whether the validation-split difference has the same sign as $\beta$.}
\label{tab:substantive-features}
\begin{tabular}{p{0.52\linewidth} r r r c}
\toprule
\textbf{Feature} & $\boldsymbol{\beta}$ & \textbf{$\Delta$Train} & \textbf{$\Delta$Val} & \textbf{Sign match} \\
\midrule
\multicolumn{5}{l}{\textit{Toward 2025}} \\
\midrule
LLM capability framing (zero-shot / reasoning / code ability)        &  1.250 &  0.181 &  0.204 & \checkmark \\
Chain-of-thought / inference-time compute                            &  0.469 &  0.063 &  0.048 & \checkmark \\
Spatial / temporal / physical grounding                               &  0.460 &  0.044 &  0.026 & \checkmark \\
User-intent matching / search ranking                                &  0.378 &  0.025 & $-0.011$ & $\times$ \\
Internal model-component analysis (attention, KV cache)              &  0.375 &  0.036 &  0.020 & \checkmark \\
Named contribution with acronym expansion                            &  0.349 &  0.064 &  0.057 & \checkmark \\
Intrusion detection / cyber-defense agents                           &  0.337 &  0.013 & $-0.001$ & $\times$ \\
Retrieval-augmented generation / long-context retrieval              &  0.292 &  0.030 &  0.025 & \checkmark \\
Agents, tool use, multi-agent workflows                              &  0.291 &  0.032 &  0.050 & \checkmark \\
Image generation / 3D synthesis (NeRF)                                &  0.282 &  0.019 &  0.026 & \checkmark \\
\midrule
\multicolumn{5}{l}{\textit{Toward 2020}} \\
\midrule
Classical NLP tasks (QA, translation, IE)                            & $-0.356$ & $-0.032$ &  0.079 & $\times$ \\
Reinforcement learning / MDPs / multi-agent RL                        & $-0.315$ & $-0.116$ & $-0.049$ & \checkmark \\
Autonomous driving / robot manipulation                               & $-0.310$ & $-0.075$ & $-0.014$ & \checkmark \\
First-person self-reference (foregrounds authors' action)            & $-0.309$ & $-0.112$ & $-0.031$ & \checkmark \\
Optimization / combinatorial search / scheduling                     & $-0.192$ & $-0.082$ & $-0.027$ & \checkmark \\
Anomaly / threat / adversarial-attack detection                      & $-0.152$ & $-0.015$ &  0.001 & $\times$ \\
Model opacity / interpretability discussion                          & $-0.138$ & $-0.012$ &  0.001 & $\times$ \\
Passive impersonal construction                                      & $-0.118$ & $-0.037$ & $-0.040$ & \checkmark \\
Threshold timing / control-policy cost                                & $-0.113$ & $-0.011$ &  0.000 & \checkmark \\
Chatbot vs. conventional survey comparison                           & $-0.111$ & $-0.007$ &  0.000 & $\times$ \\
\bottomrule
\end{tabular}
\end{table*}

\begin{table*}[t]
\centering
\small
\caption{Top 10 stylistic features per direction selected by the stylistic lasso
(166/256 features retained; validation AUC $= 0.549$, near chance). Columns as in
Table~\ref{tab:substantive-features}. Many coefficients in each direction reverse their
training-set sign on validation, illustrating the lasso's lack of held-out generalization.}
\label{tab:stylistic-features}
\begin{tabular}{p{0.52\linewidth} r r r c}
\toprule
\textbf{Feature} & $\boldsymbol{\beta}$ & \textbf{$\Delta$Train} & \textbf{$\Delta$Val} & \textbf{Sign match} \\
\midrule
\multicolumn{5}{l}{\textit{Toward 2025}} \\
\midrule
Acronym-definition pattern                                            &  0.584 &  0.323 &  0.187 & \checkmark \\
Multiple parentheticals                                                &  0.557 &  0.161 & $-0.023$ & $\times$ \\
Evidence verbs (demonstrate / validate / confirm)                     &  0.459 &  0.097 & $-0.073$ & $\times$ \\
Concessive marker (while / although / despite)                        &  0.386 &  0.053 &  0.044 & \checkmark \\
Coordination-heavy listing                                            &  0.337 &  0.203 & $-0.152$ & $\times$ \\
``This enables/allows'' consequence construction                     &  0.337 &  0.051 &  0.004 & \checkmark \\
``Available at'' formula                                              &  0.316 &  0.042 &  0.026 & \checkmark \\
Contribution labeled with artifact noun (framework / method / system) &  0.307 &  0.166 & $-0.040$ & $\times$ \\
Zero-/few-shot register modifier                                      &  0.306 &  0.026 &  0.010 & \checkmark \\
Initial prepositional framing phrase                                  &  0.292 &  0.136 &  0.008 & \checkmark \\
\midrule
\multicolumn{5}{l}{\textit{Toward 2020}} \\
\midrule
Plain impersonal reporting (no authorial ``we'')                      & $-0.611$ & $-0.221$ & $-0.007$ & \checkmark \\
Balanced promise/limitation register (while / yet / despite / but)    & $-0.492$ & $-0.046$ & $-0.001$ & \checkmark \\
Training-procedure statement                                          & $-0.443$ & $-0.088$ & $-0.006$ & \checkmark \\
Contribution announcement                                             & $-0.431$ &  0.019 &  0.003 & $\times$ \\
Triadic parallel list                                                  & $-0.411$ & $-0.030$ &  0.021 & $\times$ \\
First-person-plural authorial framing                                  & $-0.385$ & $-0.031$ &  0.022 & $\times$ \\
Competitive-baseline register                                          & $-0.336$ & $-0.022$ & $-0.026$ & \checkmark \\
Rapid-growth marker (rapidly / increasingly / emerged)                & $-0.304$ & $-0.022$ &  0.000 & $\times$ \\
Theoretical-guarantee register (bounds / proof / convergence)          & $-0.302$ & $-0.021$ & $-0.002$ & \checkmark \\
Relative-clause extension (which / that / who / where)                & $-0.292$ & $-0.144$ &  0.025 & $\times$ \\
\bottomrule
\end{tabular}
\end{table*}

\clearpage
\newpage

\section{Analysis of naive and adversarial AI text with LLM-derived features}\label{app:naiveadv}

We also conduct an experiment using LLM annotations of naive and adversarial AI text, to determine what traits of the text are concretely changing between these two kinds of \aigen text. Using the same LLM-generated features as in \Cref{app:20205}, we similarly label 5,000 naive AI sentences (generated with prompt in \citet{liang2024mapping}) and 5,000 adversarial AI sentences from the prompt designed to bypass Pangram 3.3.2, with all 512 features. Then, we produce two lasso regressions (also in exactly the same fashion as in \Cref{app:20205}).

Unsurprisingly, our analysis suggests that stylistic changes between the naive and adversarial prompts dominate the difference between their resulting \aigen sentences (for substantive features, AUC on held-out text of 0.69; AUC of 0.85 for stylistic features). We include the top 10 features for each of the two regressions by magnitude of the lasso coefficient in \Cref{tab:substantive-features-naive-vs-adversarial} and \Cref{tab:stylistic-features-naive-vs-adversarial}.

\begin{table*}[tbh]
\centering
\small
\caption{Top 10 substantive features per direction selected by the substantive lasso
(140/256 features retained; validation AUC $= 0.692$). $\beta$ is the standardized
lasso coefficient (positive $=$ toward adversarial AI, negative $=$ toward naive AI). $\Delta$Train/$\Delta$Val are mean
feature-activation differences (adversarial $-$ naive) on the respective split. ``Sign match''
indicates whether the validation-split difference has the same sign as $\beta$; sign matches here are 3/10 toward adversarial and 10/10 toward naive.}
\label{tab:substantive-features-naive-vs-adversarial}
\begin{tabular}{p{0.52\linewidth} r r r c}
\toprule
\textbf{Feature} & $\boldsymbol{\beta}$ & \textbf{$\Delta$Train} & \textbf{$\Delta$Val} & \textbf{Sign match} \\
\midrule
\multicolumn{5}{l}{\textit{Toward adversarial AI}} \\
\midrule
Terse / telegraphic scientific claim                                     &  0.265 &  0.038 &  0.041 & \checkmark \\
First-person or paper self-reference                                     &  0.088 &  0.028 &  0.017 & \checkmark \\
Alignment / steering / jailbreak defense framing                         &  0.021 &  0.001 &  0.001 & \checkmark \\
Temporal grounding / long-video reasoning                                &  0.019 &  0.001 & $-0.000$ & $\times$ \\
AI-assisted research ideation                                            &  0.018 &  0.001 & $-0.000$ & $\times$ \\
Efficient algorithm for formal constrained problem                       &  0.016 &  0.004 & $-0.000$ & $\times$ \\
Human-understandable interpretability                                    &  0.015 &  0.004 & $-0.005$ & $\times$ \\
Data contamination / noisy labels / benchmark contamination              &  0.015 &  0.001 & $-0.000$ & $\times$ \\
Federated or distributed learning                                        &  0.013 &  0.004 & $-0.002$ & $\times$ \\
Automata / state-machine complexity and reasoning length                 &  0.012 &  0.000 &  0.000 & $\times$ \\
\midrule
\multicolumn{5}{l}{\textit{Toward naive AI}} \\
\midrule
Practical usefulness / deployment readiness                              & $-0.257$ & $-0.059$ & $-0.044$ & \checkmark \\
First / novel / innovative contribution claim                            & $-0.212$ & $-0.025$ & $-0.021$ & \checkmark \\
Multi-stage pipeline or sequential procedure                             & $-0.193$ & $-0.031$ & $-0.024$ & \checkmark \\
Hope that resource/method catalyzes future work                          & $-0.190$ & $-0.019$ & $-0.017$ & \checkmark \\
Importance / centrality of task or area                                  & $-0.169$ & $-0.027$ & $-0.033$ & \checkmark \\
Explicit first/second/third sequencing                                   & $-0.160$ & $-0.013$ & $-0.010$ & \checkmark \\
Experiments validate or confirm effectiveness                            & $-0.142$ & $-0.034$ & $-0.051$ & \checkmark \\
Theoretical or principled analysis                                       & $-0.136$ & $-0.045$ & $-0.063$ & \checkmark \\
Explicit limitation or drawback of prior methods                         & $-0.131$ & $-0.035$ & $-0.023$ & \checkmark \\
Concessive framing of remaining difficulty                               & $-0.116$ & $-0.008$ & $-0.004$ & \checkmark \\
\bottomrule
\end{tabular}
\end{table*}

\begin{table*}[t]
\centering
\small
\caption{Top 10 stylistic features per direction selected by the stylistic lasso
(91/256 features retained; validation AUC $= 0.845$). $\beta$ is the standardized
lasso coefficient (positive $=$ toward adversarial AI, negative $=$ toward naive AI). $\Delta$Train/$\Delta$Val are mean
feature-activation differences (adversarial $-$ naive) on the respective split. ``Sign match''
indicates whether the validation-split difference has the same sign as $\beta$; sign matches here are 8/10 toward adversarial and 10/10 toward naive.}
\label{tab:stylistic-features-naive-vs-adversarial}
\begin{tabular}{p{0.52\linewidth} r r r c}
\toprule
\textbf{Feature} & $\boldsymbol{\beta}$ & \textbf{$\Delta$Train} & \textbf{$\Delta$Val} & \textbf{Sign match} \\
\midrule
\multicolumn{5}{l}{\textit{Toward adversarial AI}} \\
\midrule
Very short declarative / fragment                                        &  0.586 &  0.334 &  0.343 & \checkmark \\
Terse telegraphic register                                               &  0.407 &  0.307 &  0.313 & \checkmark \\
First-person plural authorial framing                                    &  0.278 &  0.064 &  0.042 & \checkmark \\
Integration verbs                                                        &  0.034 &  0.007 & $-0.004$ & $\times$ \\
Acronym-dense surface                                                    &  0.026 &  0.037 &  0.038 & \checkmark \\
Algorithm-property statement                                             &  0.022 &  0.003 &  0.004 & \checkmark \\
Semicolon clause linking                                                 &  0.020 &  0.002 &  0.004 & \checkmark \\
Through/using mechanism phrase                                           &  0.013 &  0.007 &  0.008 & \checkmark \\
Engineering construction register                                        &  0.011 &  0.008 &  0.006 & \checkmark \\
Simulation evidence register                                             &  0.009 &  0.004 & $-0.002$ & $\times$ \\
\midrule
\multicolumn{5}{l}{\textit{Toward naive AI}} \\
\midrule
This paper/work self-reference                                           & $-0.294$ & $-0.053$ & $-0.058$ & \checkmark \\
Ordinal marker                                                           & $-0.252$ & $-0.051$ & $-0.041$ & \checkmark \\
Initial ``To'' infinitive purpose clause                                 & $-0.206$ & $-0.031$ & $-0.032$ & \checkmark \\
Dash appositive or interruption                                          & $-0.187$ & $-0.040$ & $-0.044$ & \checkmark \\
Parenthesized enumeration                                                & $-0.167$ & $-0.036$ & $-0.030$ & \checkmark \\
Colon introduces list or explanation                                     & $-0.160$ & $-0.078$ & $-0.070$ & \checkmark \\
Promising future register                                                & $-0.155$ & $-0.030$ & $-0.022$ & \checkmark \\
Promise / potential marker                                               & $-0.147$ & $-0.042$ & $-0.036$ & \checkmark \\
Long multi-clause sentence                                               & $-0.110$ & $-0.049$ & $-0.053$ & \checkmark \\
Acronym-definition pattern                                               & $-0.106$ & $-0.052$ & $-0.062$ & \checkmark \\
\bottomrule
\end{tabular}
\end{table*}

\clearpage
\newpage

\newpage

\end{document}